\documentclass[a4paper,11pt,twoside,openright,final]{memoir}
\aliaspagestyle{chapter}{empty}

\setpnumwidth{2.5em}
\setrmarg{3.5em}

\usepackage[
  plainpages=false,
  pdfpagelabels,
  pdftex,
  colorlinks,
  breaklinks=true,
  linkcolor=black,
  citecolor=black,
  urlcolor=black,
  pdfauthor={David Speck},
  pdftitle={Symbolic Search for Optimal Planning with Expressive Extensions},
  pdfsubject={Dissertation},
  pdfkeywords={thesis, planning, classical planning, decision diagrams, search, symbolic search}
]{hyperref}

\usepackage[final]{microtype}
\usepackage[latin1]{inputenc}
\usepackage[T1]{fontenc}
\usepackage[ngerman,english]{babel}
\usepackage[backend=biber,style=authoryear-comp,sorting=nyt,maxbibnames=99,maxcitenames=3,dashed=false,mincrossrefs=99]{biblatex}
\usepackage{amsmath,amssymb,amsthm,amsfonts}
\usepackage{mathabx} 
\usepackage[mathscr]{eucal}
\usepackage{bm}
\usepackage{xurl}

\newcommand{\citebf}[1]{#1}
\newcommand{\citeaddendum}{}

\makeatletter
\DeclareCiteCommand{\fullcite}
{\defcounter{maxnames}{\blx@maxbibnames}%
  \usebibmacro{prenote}}
{\usedriver
  {\DeclareNameAlias{sortname}{default}}
  {\thefield{entrytype}}.\citeaddendum{}}
{\multicitedelim}
{\usebibmacro{postnote}}
\makeatother

\makeatletter
\DeclareMathSizes{\@xipt}{11}{9}{8}
\makeatother

\usepackage{breqn} 
\usepackage{hyphenat} 
\usepackage{graphicx}
\usepackage{charter} 
\usepackage{multirow}
\usepackage{dirtytalk} 
\usepackage{paralist} 
\usepackage{nicefrac}
\usepackage{subfig}
\usepackage{array}
\usepackage[noabbrev]{cleveref}
\usepackage{pstricks}

\usepackage[ruled, vlined, linesnumbered]{algorithm2e}
\usepackage{pdfpages} 

\epigraphfontsize{\normalsize\itshape}
\usepackage{tkz-kiviat}
\usepackage{tkz-euclide}
\usetikzlibrary{arrows}
\usetikzlibrary{arrows.meta}
\usetikzlibrary{graphs,graphs.standard,calc}
\usetikzlibrary{shapes,snakes,decorations}
\tikzstyle{cross}=[cross out, draw=black, minimum size=2*(#1-\pgflinewidth), inner sep=0pt, outer sep=0pt]
\usepackage{varwidth}
\usepackage{parskip}
\usepackage{fontawesome}
\usepackage{verbatim}
\usepackage{pgfplots}

\usepackage{lipsum}

\usepackage[colorinlistoftodos,color=green!40]{todonotes}
\presetkeys{todonotes}{inline}{}

\chapterstyle{madsen}
\setsecnumdepth{subsection} 
\newcommand*\cleartoleftpage{%
  \clearpage
  \ifodd\value{page}\hbox{}\newpage\fi
}
\counterwithout{footnote}{chapter}
\newcommand{\chapterquote}[2]{\epigraph{#1}{\par\raggedleft--- \textup{#2}}}

\newboolean{kiviatFull}
\newboolean{kiviatBlue}
\newboolean{kiviatRed}
\definecolor{darkgreen}{rgb}{0.0, 0.2, 0.13}
\definecolor{darkolivegreen}{rgb}{0.33, 0.42, 0.18}
\definecolor{asparagus}{rgb}{0.53, 0.66, 0.42}
\definecolor{kiviatblue}{rgb}{0.4,0.4,1}
\newcommand{\LegendBox}[3][]{%
  \coordinate[#1] (LegendBox_anchor) at (#2) ;
  \foreach \col/\item [count=\hi from 0] in {#3} {
      \draw[line width=3mm,color=\col] ([yshift=\hi*8mm]LegendBox_anchor) -- ++(.5,0)
      node[anchor=west][color=black] {\item}
      ;}
}
\newcommand{\kiviatTopk}{2}
\newcommand{\kiviatGoal}{2}
\newcommand{\kiviatCost}{2}
\newcommand{\kiviatPredicates}{2}

\theoremstyle{plain}
\newtheorem{theorem}{Theorem}

\theoremstyle{definition}
\newtheorem{definition}{Definition}
\newtheorem{example}{Example}
\theoremstyle{remark}

\newcommand{\name}[1]{\textsc{#1}}
\newcommand{\algname}[1]{\textsc{#1}}

\newcommand{\pddl}[0]{\name{pddl}}

\newcommand{\domain}{{\Xi}}
\newcommand{\task}{\ensuremath {\Pi}\xspace}
\newcommand{\vars}{\ensuremath {\mathcal V}\xspace}
\newcommand{\vardomain}{\ensuremath {D}\xspace}
\newcommand{\operators}{\ensuremath {\mathcal O}\xspace}
\newcommand{\init}{\ensuremath {\mathcal I}\xspace}
\newcommand{\goal}{\ensuremath {\mathcal G}\xspace}
\newcommand{\pre}{\ensuremath {\mathit{pre}}\xspace}
\newcommand{\eff}{\ensuremath {\mathit{eff}}\xspace}
\newcommand{\states}{\ensuremath {\mathscr S}\xspace}
\newcommand{\goalstates}{\ensuremath S_\star}
\newcommand{\utility}{\ensuremath {\mathcal U}\xspace}
\newcommand{\plan}{\ensuremath {\pi}\xspace}
\newcommand{\cost}{\ensuremath {cost}}
\newcommand{\costfun}{\ensuremath {\mathscr C}}
\newcommand{\constcostfun}{\ensuremath {\mathcal C}}

\newcommand{\domaintuple}{\langle \vars, \operators,\constcostfun \rangle}

\newcommand{\limit}{\ensuremath {\mathcal B}}
\newcommand{\comp}{F}
\newcommand{\compfun}{{\text{\textbf{f}}}}
\newcommand{\compdsfun}{f_{\xi}}
\newcommand{\compifun}{f_{\iota}}
\newcommand{\compgfun}{f_{g}}

\newcommand{\size}[1]{\ensuremath{||#1||}\xspace}
\newcommand\mybar{\kern1pt\rule[-\dp\strutbox]{1pt}{\baselineskip}\kern1pt}

\newcommand{\rover}{r}

\newcommand{\navigate}[2]{\ensuremath{\textit{navigate-}#1\textit{-}#2}}
\newcommand{\fly}[4]{\ensuremath{\textit{fly-}#1\textit{-}#2}\textit{-to-}#3\textit{-}#4}
\newcommand{\takeImg}[2]{\ensuremath{\textit{take-img-}#1\textit{-}#2}}
\newcommand{\startDrone}[2]{\ensuremath{\textit{launch-}#1\textit{-}#2}}
\newcommand{\landDrone}[2]{\ensuremath{\textit{land-}#1\textit{-}#2}}

\newcommand{\derivedvars}{\ensuremath {\mathcal D}\xspace}
\newcommand{\axioms}{\ensuremath {\mathcal A}\xspace}
\newcommand{\extendedstates}{\ensuremath {\mathscr S_{E}}\xspace}

\DeclareRobustCommand{\rchi}{{\mathpalette\irchi\relax}}
\newcommand{\irchi}[2]{\raisebox{\depth}{\ensuremath{#1\chi}}} 
\newcommand{\charf}{\ensuremath{\rchi}}

\newcommand{\astar}[1]{\ensuremath{\text{A}_\text{\scriptsize{#1}}^\star}}
\newcommand{\bddastar}{\ensuremath{\text{BDDA}^\star}}
\newcommand{\addastar}{\ensuremath{\text{ADDA}^\star}}
\newcommand{\evmddastar}{\ensuremath{\text{EVMDDA}^\star}}

\newcommand{\heu}[2]{\ensuremath{h_\text{\scriptsize{#1}}^\text{\scriptsize{#2}}}}
\newcommand{\perfecth}{\heu{\ensuremath{\star}}{}}
\newcommand{\almostperfecth}[1]{\ensuremath{#1 \perfecth}}

\newcommand{\kstar}{K\ensuremath{^\star}}
\newcommand{\forbidk}{Forbid-k}

\newcommand{\FPSPACE}{\ensuremath{\mathsf{FPSPACE}}}
\newcommand{\PSPACE}{\ensuremath{\mathsf{PSPACE}}}
\newcommand{\FP}{\ensuremath{\mathsf{FP}}}
\newcommand{\NP}{\ensuremath{\mathsf{NP}}}
\newcommand{\coNP}{\ensuremath{\mathsf{co}\text{-}\mathsf{NP}}}
\newcommand{\PP}{\ensuremath{\mathsf{PP}}}
\newcommand{\NPSPACE}{\ensuremath{\mathsf{NPSPACE}}}
\newcommand{\EXPTIME}{\ensuremath{\mathsf{EXPTIME}}}

\DeclareBibliographyCategory{ignore}
\newboolean{attachPDFs}
\setboolean{attachPDFs}{false}

\begin{document}
\begin{titlingpage}

  \small \centering{\textsf{Dissertation zur Erlangung des Doktorgrades der Technischen Fakult{\"a}t der Albert-Ludwigs-Universit{\"a}t Freiburg im Breisgau}}

  \vspace{0.5cm}

  \rule{\textwidth}{0.4mm}\\%

  \huge \centering{\textbf{Symbolic Search for Optimal Planning with Expressive Extensions}}\\

  \vspace{-0.4cm}

  \rule{\textwidth}{0.4mm}\\%

  \vspace{2.5cm}

  \centering{\textbf{David Speck}}

  \vspace{1cm}
  \centering{\resizebox*{0.7\textwidth}{!}{
      \includegraphics{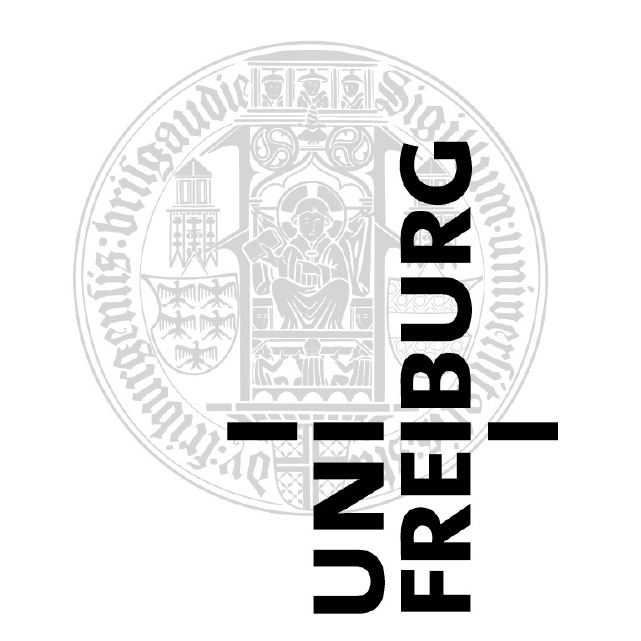}}}

  \vspace{1.5cm}

  \LARGE \centering{\textbf{2022}}

  \normalsize

  \clearpage

  \vspace*{\fill}
  \begin{flushleft}
    \noindent
    \textbf{Dean:}\\
    Prof. Dr. Roland Zengerle, \emph{University of Freiburg, Germany}\\

    \bigskip

    \noindent
    \textbf{PhD advisor and first reviewer:}\\
    Prof. Dr. Bernhard Nebel, \emph{University of Freiburg, Germany}\\

    \bigskip

    \noindent
    \textbf{Second reviewer:}\\
    Prof. Dr. {\'A}lvaro Torralba, \emph{Aalborg University, Denmark}

    \bigskip

    \noindent
    \textbf{Date of defense:}\\
    February 23, 2022

  \end{flushleft}
\end{titlingpage}

\frontmatter

\clearpage
\clearpage
\thispagestyle{empty}
\chapter*{Abstract}

In classical planning, the goal is to derive a course of actions that allows an intelligent agent to move from any situation it finds itself in to one that satisfies its goals.
Classical planning is considered domain-independent, i.e., it is not limited to a particular application and can be used to solve different types of reasoning problems.
In practice, however, some properties of a planning problem at hand require an expressive extension of the standard classical planning formalism to capture and model them. 
Although the importance of many of these extensions is well known, most planners, especially optimal planners, do not support these extended planning formalisms. 
The lack of support not only limits the use of these planners for certain problems, but even if it is possible to model the problems without these extensions, it often leads to increased effort in modeling or makes modeling practically impossible as the required problem encoding size increases exponentially.

In this thesis, we propose to use symbolic search for cost-optimal planning for different expressive extensions of classical planning, all capturing different aspects of the problem.
In particular, we study planning with axioms, planning with state-dependent action costs, oversubscription planning, and top-$k$ planning. 
For all formalisms, we present complexity and compilability results, highlighting that it is desirable and even necessary to natively support the corresponding features.
We analyze symbolic heuristic search and show that the search performance does not always benefit from the use of a heuristic and that the search performance can exponentially deteriorate even under the best possible circumstances, namely the perfect heuristic.
This reinforces that symbolic blind search is the dominant symbolic search strategy nowadays, on par with other state-of-the-art cost-optimal planning strategies.
Based on this observation and the lack of good heuristics for planning formalisms with expressive extensions, symbolic search turns out to be a strong approach.
We introduce symbolic search to support each of the formalisms individually and in combination, resulting in optimal, sound, and complete planning algorithms that empirically compare favorably with other approaches.

\clearpage
\clearpage
\thispagestyle{empty}
\selectlanguage{ngerman}

\chapter*{Zusammenfassung}
Bei der klassischen Planung besteht das Ziel darin, einen Handlungsablauf zu finden, der es einem intelligenten Agenten erm{\"o}glicht, aus jeder Situation, in der er sich befindet, in eine Situation zu gelangen, die seine Ziele erf{\"u}llt.
Die klassische Planung gilt als dom{\"a}nenunabh{\"a}ngig, d.h. sie ist nicht auf eine bestimmte Anwendung beschr{\"a}nkt und kann zur L{\"o}sung verschiedener Arten von Logikproblemen verwendet werden.
In der Praxis erfordern jedoch einige Eigenschaften eines vorliegenden Planungsproblems eine ausdrucksstarke Erweiterung des klassischen Standardplanungsformalismus, um sie zu erfassen und zu modellieren. 
Obwohl die Bedeutung vieler dieser Erweiterungen bekannt ist, unterst{\"u}tzen die meisten Planer, insbesondere optimale Planer, diese erweiterten Planungsformalismen nicht. 
Die fehlende Unterst{\"u}tzung schr{\"a}nkt nicht nur die Verwendung dieser Planer f{\"u}r bestimmte Probleme ein, sondern selbst wenn es m{\"o}glich ist, die Probleme ohne diese Erweiterungen zu modellieren, f{\"u}hrt dies oft zu einem erh{\"o}hten Aufwand bei der Modellierung oder macht die Modellierung praktisch unm{\"o}glich, da die erforderliche Problemkodierungsgr{\"o}{\ss}e exponentiell ansteigt.

In dieser Arbeit schlagen wir vor, die symbolische Suche f{\"u}r kostenoptimale Planung f{\"u}r verschiedene ausdrucksstarke Erweiterungen der klassischen Planung zu verwenden, die alle unterschiedliche Aspekte des Problems erfassen.
Insbesondere untersuchen wir die Planung mit Axiomen, die Planung mit zustandsabh{\"a}ngigen Aktionskosten, die \say{Oversubscription Planung} und die \say{Top-$k$ Planung}. 
F{\"u}r alle Formalismen pr{\"a}sentieren wir Ergebnisse zur Komplexit{\"a}t und Kompilierbarkeit, wobei wir hervorheben, dass es w{\"u}nschenswert und sogar notwendig ist, die entsprechenden Aspekte nativ zu unterst{\"u}tzen.
Wir analysieren die symbolische heuristische Suche und zeigen, dass die Suchleistung nicht immer von der Verwendung einer Heuristik profitiert und dass die Suchleistung selbst unter den bestm{\"o}glichen Umst{\"a}nden, n{\"a}mlich der perfekten Heuristik, exponentiell abnehmen kann.
Dies unterstreicht, dass die symbolische blinde Suche heutzutage die dominierende symbolische Suchstrategie ist, {\"a}nlich gut wie anderen modernen kostenoptimalen Planungsstrategien.
Basierend auf dieser Beobachtung und dem Mangel an guten Heuristiken f{\"u}r Planungsformalismen mit ausdrucksstarken Erweiterungen, erweist sich die symbolische Suche als ein starker Ansatz.
Wir erweitern die symbolische Suche, um jeden der Formalismen einzeln und in Kombination zu unterst{\"u}tzen, was zu optimalen, korrekten und vollst{\"a}ndigen Planungsalgorithmen f{\"u}hrt, die sich im Vergleich zu anderen Ans{\"a}tzen empirisch besser verhalten.
\selectlanguage{english}

\clearpage
\clearpage
\thispagestyle{empty}
\chapter*{Acknowledgments}

This thesis is the final product, which in itself required a lot of effort, but the years during the doctoral process laid the foundation for everything. 
During this time, I have had the pleasure of meeting, learning from, researching with, and being supported by many great people, for which I would like to express my gratitude.

First of all, I would like to thank Bernhard Nebel, who not only gave me the opportunity to do my own research that interested me, but at the same time kept inspiring me with his expertise and dedication to topics like complexity theory.
The Foundations of Artificial Intelligence led by Bernhard was a great place to work with many wonderful people such as Florian Gei{\ss}er, Gregor Behnke, Grigorios Mouratidis, Johannes Aldinger, Petra Geiger, Robert Mattm{\"u}ller, Rolf-David Bergdoll, Stefan Staeglich, Thorsten Engesser, Tim Schulte and Uli Jakob, all of whom I would like to thank for great discussions and support. 

I would especially like to thank Robert Mattm{\"u}ller, who has supported and guided me since my first academic steps.
There are few people who can inspire young people with great commitment and enthusiasm for science. 
Among them, there are even fewer that allow young scientists to be seen as independent researchers by not putting themselves in the center of the action.
Robert is one of those few people without whom this work would not have been possible. Thank you very much.

I would like to thank all my colleagues with whom I have had the pleasure of writing, revising, and publishing a paper during my doctoral process: {\'A}lvaro Torralba, Andr{\'e} Biedenkapp, Bernhard Nebel, David Borukhson, Dominik Drexler, Florian Gei{\ss}er, Frank Hutter, Gregor Behnke, Jendrik Seipp, Marius Lindauer, Michael Katz, Robert Mattm{\"u}ller, Sumitra Corraya and Thomas Keller.
A special thanks to the planning community that has welcomed me since I first attended ICAPS 2018. 
In particular, to {\'A}lvaro Torralba for many interesting discussions, and to Michael Katz for the opportunity to work with him and visit the IBM T. J. Watson Research Center in New York. 
Both treated me like a colleague from the first meeting and shared their expertise with me.

Being a doctoral student is a roller coaster ride that is only possible with a well-balanced social life. 
I would like to take this opportunity to thank all my friends with whom I have shared many great experiences in my gymnastics career. 
Thanks to my study friends Andr{\'e} Biedenkapp, Lukas Gemein and Rick Gelhausen, all of whom will (hopefully) soon receive their PhDs. 
Special thanks to Andr{\'e} Biedenkapp, with whom I have shared many personally great times and scientific successes but also some setbacks.
The many overtime hours with you became a pleasure thanks to your positive attitude and enthusiasm for science.

I would like to thank my family. 
My parents Thomas and Olga, who have always supported me and taught me from an early age that hard work will pay off in the long run. 
They together with Iva, Mario and their slowly but steadily growing family were the support that many deserve but few receive. 
I am grateful that I can consider you not only family but also friends.

Last but not least, I would like to thank Lena (and Balu), who joined my journey a little later but are now indispensable. 
It is hard to describe how much I enjoyed the last time, although it should actually be very exhausting to write such a thesis. 
Lena, your incredibly loving, understanding and motivating nature towards me and my work has helped me to accomplish this thesis.

\vfill{}

\begin{center}
    \begin{tikzpicture}[%
            costnode/.style={pos=0.6,rectangle,thick,solid,
                    inner sep=2pt,draw,fill=white,text=black,font=\small},%
            decisionnode/.style={circle,thick,minimum size=4mm,
                    inner sep=2pt,draw,fill=white!80!black,text=black,font=\normalsize},%
            xscale=2,yscale=1.5,>=latex]
        \begin{scope}{xshift=0cm}
            \node[] (before) at (0,4.6) {}; 
            \node[decisionnode, minimum size=0.7cm] (t) at (0.0,3.75) {\textbf{T}};
            \node[decisionnode, minimum size=0.7cm] (h) at (0.2,3.0) {\textbf{h}};
            \node[decisionnode, minimum size=0.7cm] (a) at (0.4,2.25) {\textbf{a}};
            \node[decisionnode, minimum size=0.7cm] (n) at (0.6,1.5) {\textbf{n}};
            \node[decisionnode, minimum size=0.7cm] (k) at (0.8,0.75) {\textbf{k}};
            \node[draw,thick,fill=white!80!black,rectangle, minimum size=0.55cm]
            (after0) at (-0.2,0) {{$0$}};
            \node[draw,thick,fill=white!80!black,rectangle, minimum size=0.55cm]
            (after1) at (1.0,0) {{$1$}};
            \draw[->, thick] (before) to (t);
            \draw[->, thick] (t) to[bend right=0,label distance=0mm,edge
            label={\small{$1$}},pos=0.6] (h);
            \draw[->, thick, dotted] (t) to[bend left=0,label distance=0mm,edge
            label={\small{$0$}},swap,pos=0.1] (after0);
            \draw[->, thick] (h) to[bend right=0,label distance=0mm,edge
            label={\small{$1$}},pos=0.6] (a);
            \draw[->, thick, dotted] (h) to[bend left=0,label distance=0mm,edge
            label={\small{$0$}},swap,pos=0.1] (after0);
            \draw[->, thick] (a) to[bend right=0,label distance=0mm,edge
            label={\small{$1$}},pos=0.6] (n);
            \draw[->, thick, dotted] (a) to[bend left=0,label distance=0mm,edge
            label={\small{$0$}},swap,pos=0.1] (after0);
            \draw[->, thick] (n) to[bend right=0,label distance=0mm,edge
            label={\small{$1$}},pos=0.6] (k);
            \draw[->, thick, dotted] (n) to[bend left=0,label distance=0mm,edge
            label={\small{$0$}},swap,pos=0.1] (after0);
            \draw[->, thick] (k) to[bend right=0,label distance=0mm,edge
            label={\small{$1$}},pos=0.6] (after1);
            \draw[->, thick, dotted] (k) to[bend left=0,label distance=0mm,edge
            label={\small{$0$}},swap,pos=0.1] (after0);
        \end{scope}
        \begin{scope}[xshift=3cm]
            \node[] (before) at (0.3,4.60) {}; 
            \node[decisionnode, minimum size=0.7cm] (y) at (0.3,3.75) {\textbf{Y}};
            \node[decisionnode, minimum size=0.7cm] (o) at (0.6,2.8125) {\textbf{o}};
            \node[decisionnode, minimum size=0.7cm] (u) at (0.9,1.875) {\textbf{u}};
            \node[decisionnode, minimum size=0.7cm] (x) at (1.1,0.9375) {\textbf{!}};
            \node[draw,thick,fill=white!80!black,rectangle, minimum size=0.55cm]
            (after0) at (0.3,0) {{$0$}};
            \draw[->, thick] (before) to[bend right=0,label distance=0mm,edge
            label={},swap,pos=0.3] node[costnode,pos=0.4] {$0$} (y);
            \draw[->, thick, dotted] (y) to[bend right=30,label distance=0mm,edge
            label={\small{$0$}},swap,pos=0.1] node[costnode,pos=0.4] {$\infty$} (after0);
            \draw[->, thick] (y) to[bend left=30,label distance=0mm,edge
            label={\small{$1$}},pos=0.01] node[costnode,pos=0.35] {$0$} (o);
            \draw[->, thick, dotted] (o) to[bend right=30,label distance=0mm,edge
            label={\small{$0$}},swap,pos=0.1] node[costnode,pos=0.4] {$\infty$} (after0);
            \draw[->, thick] (o) to[bend left=30,label distance=0mm,edge
            label={\small{$1$}},pos=0.01] node[costnode,pos=0.35] {$0$} (u);
            \draw[->, thick, dotted] (u) to[bend right=30,label distance=0mm,edge
            label={\small{$0$}},swap,pos=0.1] node[costnode,pos=0.4] {$\infty$} (after0);
            \draw[->, thick] (u) to[bend left=30,label distance=0mm,edge
            label={\small{$1$}},pos=0.01] node[costnode,pos=0.35] {$0$} (x);
            \draw[->, thick, dotted] (x) to[bend right=30,label distance=0mm,edge
            label={\small{$0$}},swap,pos=0.1] node[costnode,pos=0.4] {$\infty$} (after0);
            \draw[->, thick] (x) to[bend left=30,label distance=0mm,edge
            label={\small{$1$}},pos=0.01] node[costnode,pos=0.35] {$0$} (after0);
        \end{scope}
    \end{tikzpicture}
\end{center}

\vfill{}

\cleardoublepage

\setcounter{tocdepth}{1}
\tableofcontents*

\mainmatter
\chapter{Introduction}
\chapterquote{If you fail to plan, you are planning to fail.}{Benjamin Franklin}

Automated planning is the science of designing and engineering machines, in particular computer programs, that can automatically derive behaviors to achieve goals.
The generation of such strategies, i.e., thinking before acting, is understood as an intelligent behavior and is one of the oldest areas in the field of Artificial Intelligence.
Nowadays, a classical planning problem can be informally described as the problem of finding a course of actions that allows an intelligent agent to move from any situation it finds itself in to one that satisfies its goals.
Since planning is not limited to a specific application domain, it was originally labeled as general problem solving \autocite{helmert-2008,newell-simon-cat1963} and can be used for different types of reasoning problems, such as elevator control \autocite{koehler-schuster-aips2000}, greenhouse logistics \autocite{helmert-lasinger-icaps2010}, natural-language generation \autocite{koller-hoffmann-icaps2010}, robot control \autocite{nilsson-tr1984,speck-et-al-ral2017,karpas-magazzeni-2020}, risk management \autocite{sohrabi-et-al-aaai2018,katz-et-al-icaps2021wsfinplan}, and many more.

\setboolean{kiviatBlue}{true}
\setboolean{kiviatRed}{false}
\renewcommand{\kiviatTopk}{1}
\renewcommand{\kiviatGoal}{1}
\renewcommand{\kiviatCost}{1}
\renewcommand{\kiviatPredicates}{1}
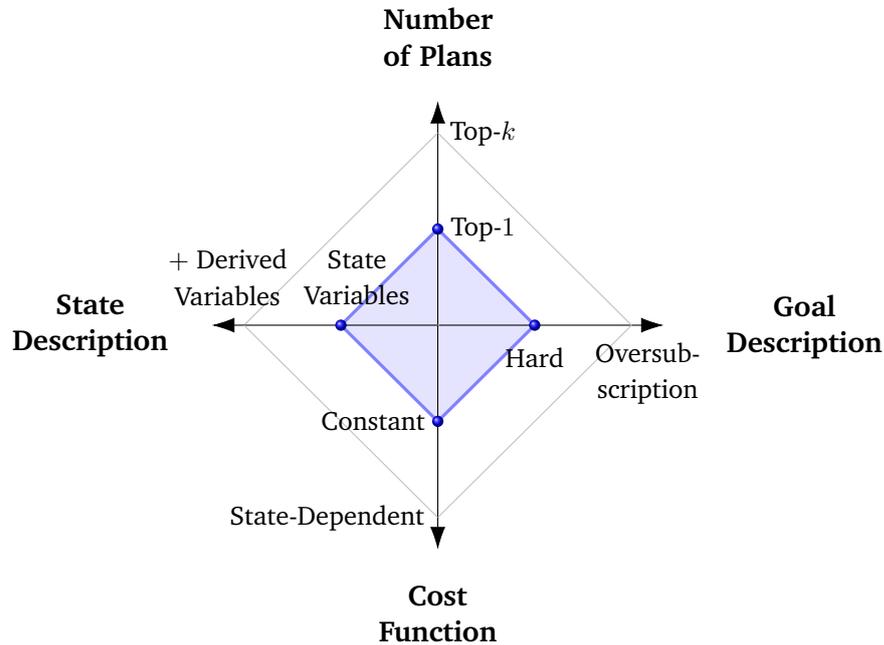
\begin{figure}[t]
    \begin{center}
        \begin{tikzpicture}
    \tkzKiviatDiagram[
    radial style/.style ={-{Latex[length=3mm, width=2mm]}},
    scale=0.85,
    label space=1.5,
    radial = 1,
    gap = 1.5,
    step = 1,
    lattice = 2]{
    \hspace{2.5cm}\mbox{\parbox{4cm}{{\begin{center}\textbf{Goal}\\\textbf{Description} \end{center}}}},
    {\textbf{Number of Plans}},
    \hspace{-1.5cm}\mbox{\parbox{3cm}{{\begin{center}\textbf{State}\\\textbf{Description}\end{center}}}},
    \textbf{Cost Function}}

    \ifbool{kiviatRed}{
        \tkzKiviatLine[very thick,color=red!50,
            fill=red!30,
            opacity=.35,
            mark=ball,
            mark size=2.5pt,
            ball color=red](\kiviatGoal{},\kiviatTopk{},\kiviatPredicates{},\kiviatCost{})
    }

    \ifbool{kiviatBlue}{
        \tkzKiviatLine[very thick,color=blue!50,
            fill=blue!30,
            opacity=.35,
            mark=ball,
            mark size=2.5pt,
            ball color=blue](1,1,1,1)
    }

    \ifbool{kiviatFull}{
    \tkzKiviatLine[very thick,color=darkgreen!50,
        fill=green!30,
        opacity=.35,
        mark=ball,
        mark size=2.5pt,
        ball color=darkgreen](1,2,2,2)
    \tkzKiviatLine[very thick,color=blue!50,
        fill=blue!30,
        opacity=.35,
        mark=ball,
        mark size=2.5pt,
        ball color=blue](1,1,1,1)
    \tkzKiviatLine[very thick,color=red!50,
        mark=ball,
        mark size=2.5pt,
        ball color=red](\kiviatGoal{},\kiviatTopk{},\kiviatPredicates{},\kiviatCost{})
    \LegendBox[shift={(-0cm,-2.25cm)}]{current bounding box.south west}%
    {
    red!100/{Forward symbolic search (contribution of this thesis)},
    asparagus!100/{Bidirectional symbolic search (contribution of this thesis)},
    blue!100/{Bidirectional symbolic search (previous state of the art)}}
    }
    {}

    \draw[] node[align=center,fill=none] at (1.5,-0.5) {\small Hard};
    \draw[] node[align=center,fill=none] at (3.25,-0.75) {\small Oversub-\\ \small scription};

    \draw[] node[align=center,fill=none] at (0.7,1.5) {\small Top-$1$};
    \draw[] node[align=center,fill=none] at (0.7,3.0) {\small Top-$k$};

    \draw[] node[align=center,fill=none] at (-1.25,0.75) {\small State\\ \small Variables};
    \draw[] node[align=center,fill=none] at (-3.25,0.75) {\small $+$ \small Derived\\ \small Variables};

    \draw[] node[align=center,fill=none] at (-1,-1.5) {\small Constant};
    \draw[] node[align=center,fill=none] at (-1.6,-3.0) {\small State-Dependent\phantom{0}};
\end{tikzpicture}
    \end{center}
    \caption[Overview of extension for classical planning.]{
        Overview of extensions for classical planning, where the blue color denotes the planning formalism supported by symbolic search approaches from the literature.
    }\label{fig:intro:kiviat}
\end{figure}
\renewcommand{\kiviatTopk}{1}
\renewcommand{\kiviatGoal}{1}
\renewcommand{\kiviatCost}{1}
\renewcommand{\kiviatPredicates}{1}
\setboolean{kiviatBlue}{false}
\setboolean{kiviatRed}{true}

In classical planning, some properties of a planning problem at hand require an expressive extension of the standard classical planning formalism to capture them in a concise way.
In this thesis, we consider four different extensions of classical planning (\Cref{fig:intro:kiviat}), all of which capture and extend different aspects of classical planning while retaining the core of the formalism: Single-agent planning problems in a fully observable, static, discrete, deterministic, and fully known environment \autocite{russell-norvig-2003}.
Unfortunately, most classical planners do not support any of these expressive extensions because they are usually based on heuristic search and it appears very challenging to design informative and fast to compute heuristics (goal-distance estimators) that consider additional problem properties.
This is especially true for cost-optimal planners, which additionally require that a heuristic is admissible, i.e., that the heuristic never overestimates the cost of reaching a goal state.
However, it is well known that such extensions are critical for an elegant and compact modeling of many real-world problems.
For example, extending the state description to include derived variables allows aspects of a planning problem that are not directly affected by the actions but are derived from the values of other variables to be modeled concisely using a set of logical axioms \autocite{thiebaux-et-al-aij2005}.
While in probabilistic planning a concise encoding of action costs or rewards in form of Markov decision processes has long been standard, in classical planning it is common to consider a potentially exponentially larger representation with constant action costs \autocite{geisser-phd2018}.
In many real-world applications, the description of goals is oversubscribed, i.e., there are a large number of desirable, often competing goals of varying value, and a system may not be able to achieve all of them with the available resources \autocite{smith-icaps2004}.
Unfortunately, conventional classical planning cannot handle such a scenario, because the goal is treated as a hard constraint that can either be accomplished or not.
Finally, in some cases, a single plan may be sufficient, but in practice it is often better to have many good alternatives to choose from in order to be more flexible \autocite{nguyen-et-al-aij2012}.

In this thesis, we investigate and discuss the computational complexity and compilability of classical planning incorporating the four expressive extensions mentioned earlier.
This analysis shows and highlights that for many real-world problems, it is desirable and even necessary to support these features natively.
We propose for the first time to use symbolic search for cost-optimal planning for those four expressive extensions (see \Cref{fig:intro:kiviat}).
Symbolic search provides a good basis for native support of these planning formalisms for several reasons.
A theoretical and empirical analysis of symbolic heuristic search in the form of \bddastar{} shows that the use of a heuristic in symbolic search does not always improve search performance.
This observation strengthens the fact that nowadays symbolic blind search, i.e., without heuristics, is the dominant search strategy of symbolic search, on par with explicit heuristic search in cost-optimal planning.
Thus, since symbolic search does not necessarily need heuristics to be efficient, the search efficiency does not suffer from the lack of efficient and informative heuristics for the proposed planning scenarios.
Based on these observations, we enhance symbolic search to obtain optimal, sound, and complete algorithms for planning with the expressive extensions under consideration.
Our empirical evaluations show that the presented symbolic search approaches perform favorably in all these planning settings compared with other state-of-the-art approaches.
Finally, we show that the proposed symbolic search approach is able to support planning tasks that use and require all expressive extensions at once.

\section{Outline}
This thesis is structured as follows.
In \Cref{ch:background}, we introduce the relevant background for this thesis concerning classical planning. In particular, we introduce the concepts of complexity, compilability and expressive power and present symbolic search for classical planning
together with different types of decision diagrams.
In \Cref{ch:symbolic_heuristic_search}, we investigate the question, why heuristics do not seem to pay off in symbolic search.
Specifically, we theoretically and empirically analyze the search behavior of symbolic heuristic search in form of \bddastar{} \autocite{speck-et-al-icaps2020}.
In \Cref{ch:axioms}, we summarize computational complexity and compilability results for planning with axioms from the literature.
We introduce three ways to extend symbolic search algorithms to support axioms natively and present an empirical evaluation \autocite{speck-et-al-icaps2019}.
In \Cref{ch:sdac}, we present computational complexity and compilability results for planning with state-dependent action costs \autocite{speck-et-al-icaps2021}.
Then, symbolic search algorithms for planning with state-dependent action costs are presented and an empirical evaluation is conducted \autocite{speck-et-al-icaps2018}.
In \Cref{ch:osp}, we discuss computational complexity and compilability results for oversubscription planning.
A symbolic search approach is presented and explained for oversubscription planning and an empirical evaluation is conducted \autocite{speck-katz-aaai2021}.
In \Cref{ch:topk}, we present computational complexity and compilability results for top-$k$ planning and introduce symbolic search for it, which we analyze theoretically and empirically \autocite{speck-et-al-aaai2020}.
In \Cref{ch:discussion}, we discuss the combination of the previously analyzed extensions for classical planning and the use of symbolic search for such a setting.
Finally, possible future work related to this thesis is
discussed.
\Cref{ch:conclusion} concludes and summarizes this thesis.

\section{Published Work}
The core results presented in this work have been published at leading AI and automated planning conferences.
In particular, the following publications form the backbone of this work.
At the beginning of each chapter, we indicate which publications form the basis for the content of the chapter.
\ifbool{attachPDFs}{
    Each of these core papers is attached in its publication form at the end of the thesis.
}

\renewcommand{\citebf}[1]{\textbf{#1}}
\begin{itemize}
    \item \renewcommand{\citeaddendum}{\\\textbf{(Best Student Paper Runner-Up Award)}}\fullcite{speck-et-al-icaps2021}\renewcommand{\citeaddendum}{}
    \item \fullcite{speck-katz-aaai2021}
    \item \fullcite{speck-et-al-icaps2020}
    \item \fullcite{speck-et-al-aaai2020}
    \item \fullcite{speck-et-al-icaps2019}
    \item \renewcommand{\citeaddendum}{\\\textbf{(Partly based on ideas from my master thesis)}}\fullcite{speck-et-al-icaps2018}\renewcommand{\citeaddendum}{}
\end{itemize}
\renewcommand{\citebf}[1]{#1}

The following publications also resulted from my doctoral research, but are not a central part of this thesis.
However, some of the ideas and concepts presented in this thesis were used in these research papers.

\renewcommand{\citebf}[1]{\textbf{#1}}
\begin{itemize}
    \item \fullcite{speck-et-al-icaps2021b}
          \begin{itemize}
              \item \renewcommand{\citeaddendum}{\textbf{ (superseded)}}\fullcite{speck-et-al-icaps2020wsprl}\renewcommand{\citeaddendum}{}
          \end{itemize}
    \item \fullcite{drexler-et-al-icaps2021}
          \begin{itemize}
              \item \renewcommand{\citeaddendum}{\textbf{ (superseded)}}\fullcite{drexler-et-al-icaps2020wshsdip}\renewcommand{\citeaddendum}{}
          \end{itemize}
    \item \fullcite{behnke-speck-aaai2021}
    \item \fullcite{geisser-et-al-socs2020}
    \item \fullcite{corraya-et-al-ki2019}
    \item \fullcite{geisser-et-al-icaps2019wsipc}
    \item \fullcite{speck-et-al-ipc2018}
    \item \renewcommand{\citeaddendum}{\textbf{ (Winner of the Discrete Markov Decision Process Track)}}\fullcite{geisser-speck-ippc2018}\renewcommand{\citeaddendum}{}
\end{itemize}
\renewcommand{\citebf}[1]{#1}

Finally, the following patents were filed based on work conducted during my doctoral process.

\renewcommand{\citebf}[1]{\textbf{#1}}
\begin{itemize}
    \item \fullcite{speck-et-al-patentus2021}
    \item \fullcite{speck-et-al-patenteu2020}
\end{itemize}
\renewcommand{\citebf}[1]{#1}

\section{Awards}
Together with my colleagues, I received the following awards for my work during my doctoral process.

\renewcommand{\citebf}[1]{\textbf{#1}}
\begin{itemize}
    \item \textbf{Best Student Paper Runner-Up Award} for the paper \say{On the Compilability and Expressive Power of State-Dependent Action Costs} \autocite{speck-et-al-icaps2021} at ICAPS 2021
    \item \textbf{Winner of the Discrete Markov Decision Process Track} for the planning system \say{Prost-DD} \autocite{geisser-speck-ippc2018} at the Sixth International Probabilistic Planning Competition at ICAPS 2018
\end{itemize}
\renewcommand{\citebf}[1]{#1}

\chapter{Background to Classical Planning}\label{ch:background}
\chapterquote{[Binary Decision Diagrams are] one of the only really fundamental data structures that came out in the last twenty-five years.}{Donald Knuth (2008)}

In AI planning, the objective is to automatically find a solution (a plan) to a given problem (a planning task).
In this chapter, we formally define classical planning and discuss the concepts of complexity, compilability and expressive power in this context.
Finally, symbolic search for classical planning is presented and different types of decision diagrams are introduced.

\section{Formalism}
In this thesis, we consider classical planning domains and tasks that are characterized by the \name{sas}$^+$ formalism \autocite{backstrom-nebel-compint1995}.

\begin{definition}[Planning Domain]
  \label{def:planning-domain}
  A \emph{planning domain} is a tuple $\domain = \domaintuple$. $\vars$ is a finite set of state variables, each associated with a finite domain $\vardomain{}_v = \{0, \dots, |\vardomain{}_v| - 1\}$. A \emph{fact} is a pair $(v, d)$, where $v \in \vars$ and $d \in \vardomain_v$. For binary variables we also write $v$ for $(v,1)$ and $\lnot v$ for $(v,0)$. A \emph{partial state} or \emph{partial variable assignment} $s$ over $\vars$ is a function on some subset of $\vars$ such that $s(v) \in \vardomain{}_v$, wherever $s(v)$ is defined. If $s$ assigns a value to each variable $v \in \vars$, $s$ is called a \emph{state}. With $\states$ we refer to the set of all possible states defined over $\vars$. $\operators$ is a finite set of \emph{operators/actions}, where each operator is a pair $o = \langle \pre_o, \eff_o \rangle$ of partial variable assignments, called \emph{preconditions} and \emph{effects}. An operator $o \in \operators$ is applicable in a state $s$ iff $\pre_o$ is satisfied in $s$, i.e., $\pre_o \subseteq s$. Applying operator $o$ to state $s$ yields the state $s[o]=s'$, where $s'(v) = \eff_o(v)$ for all variables $v \in \vars$ for which $\eff_o$ is defined and $s'(v) = s(v)$ otherwise. Finally, $\constcostfun: \operators \to \mathbb{N}_0$ is the \emph{cost function} of $\domain$. The size of $\domain$ is $\size{\domain} = \sum_{v \in \vars}|\vardomain{}_v| + \sum_{o \in \operators}(|\pre_o| + |\eff_o|) + \size{\constcostfun}$, where $\size{\constcostfun}$ is bounded by $|\operators| \cdot \lceil \log_2 N \rceil$, where $N$ is the highest action cost value.
\end{definition}

Planning domains together with initial states and goal descriptions form planning tasks.

\begin{definition}[Planning Task]
  \label{def:planning-task}
  A \emph{planning task} is a tuple $\task = \langle \domain, \init, \goal, \limit \rangle$ consisting of a \emph{planning domain}, an \emph{initial state} $\init \in \states$, a partial variable assignment $\goal$ (\emph{goal condition}), which defines all possible goal states $\goalstates \subseteq \states$, and a \emph{cost bound} $\limit \in \mathbb{N}_0 \cup \{ \infty \}$. For simplicity, we also write $\langle \vars, \operators, \constcostfun, \init, \goal \rangle$ for a planning task with domain $\domain = \domaintuple$ and $\limit = \infty$.
  The \emph{size} of $\task$ is $\size{\task} = \size{\domain} + |\goal| + |\init| + \size{\limit}$, where $\size{\limit} = \lceil \log_2 B \rceil$ is the binary encoding size of the cost bound.
\end{definition}

A planning task is often defined without a cost bound $\limit$, which is equivalent to setting the cost bound to infinity (or a sufficient upper bound). However, we use planning tasks with cost bounds to perform complexity and compilability analyses in this dissertation.

The objective of classical planning is to determine plans, which are sequences of applicable actions leading from the initial state to a goal state, respecting the given cost bound $\limit$.

\begin{definition}[Plan]
  \label{def:plan}
  A \emph{plan} $\plan = \langle o_0, \dots, o_{n-1} \rangle$ for planning task
  $\task$ is a sequence of applicable operators that generates a sequence of
  states $s_0, \dots, s_n$, where  $s_0 = \init$, $s_n \in \goalstates$ is a goal state, $s_{i+1} = s_i[o_i]$ for all $i = 0, \dots, n-1$, and the \emph{cost} of $\pi$ is less than or equal to the cost bound $\limit$, i.e., $\cost(\plan) = \sum_{i=0}^{{n-1}} \constcostfun(o_i) \leq \limit$. Plan $\pi$ is \emph{optimal} if
  there is no cheaper plan. We call $\size{\plan} = n$ the \emph{length} or \emph{size} of $\plan$. With $P_\task$ we refer to the (possibly infinite) \emph{set of all plans} for a planning task $\task$.
\end{definition}

The search for an optimal plan, i.e., a plan with the lowest cost, is called cost-optimal planning, or optimal planning for short, and is the focus of this thesis. 
\Cref{ex:rover} describes a classical planning domain and task of a Mars rover, which we will use in this work to motivate and explain expressive features that extend the classical planning formalism.

\begin{example}\label{ex:rover}
  Consider a Mars rover like Perseverance\footnote{\label{note1}\url{https://mars.nasa.gov/mars2020/mission/overview/} (Accessed: 2021-10-12)} equipped with a drone like Ingenuity\footref{note1}, which are designed to perform autonomous tasks. 
  Such a scenario is illustrated in \Cref{fig:mars_example}. 
  The dynamics of this example, i.e., the domain, are as follows.

  The rover can navigate between adjacent cells if they are free (impassable cells are highlighted in red). Navigating the rover between cells has no cost, i.e., a cost of $0$. The drone can launch and land on the rover if both are in the same location. Launching and landing each has a cost of $5$. In addition, the drone can take an image at its current position at a cost of $2$ and fly between any two cells, ignoring the impassability of cells to the rover. The cost of flying from one cell $\textit{from} = (x,y)$ to another cell $\textit{to} = (x',y')$ is the Manhattan distance between these two cells, i.e., $|x-x'|+|y-y'|$.

  In this particular planning task, the rover along with the drone is initially located at (7,3), the actual landing site of Perseverance. The goal is to take pictures at (6,1) and (10,1) and travel with the drone equipped to (0,5), a location known as ``Three Forks'' from which new missions can be commenced.  The cost bound for this task is infinite.

  An optimal plan for this task is $\pi^\rover{} = \langle$\navigate{7}{2},  \navigate{7}{1}, \startDrone{7}{1}, \fly{7}{1}{6}{1}, \takeImg{6}{1}, \fly{6}{1}{10}{1}, \takeImg{10}{1}, \fly{10}{1}{7}{1}, \landDrone{7}{1}, \navigate{7}{2}, \navigate{7}{3}, \dots, \navigate{0}{5}$\rangle$ with a cost of $\cost(\pi^\rover{}) = 5 + 5 + 2 \cdot 2 + 8 = 24$, since navigating the rover costs $0$, the drone launches and lands once with a cost of $5$ each, takes images twice with a cost of $2$ each, and flies a total distance of $8$ cells.


\end{example}

\begin{figure}
  \begin{center}
    \input{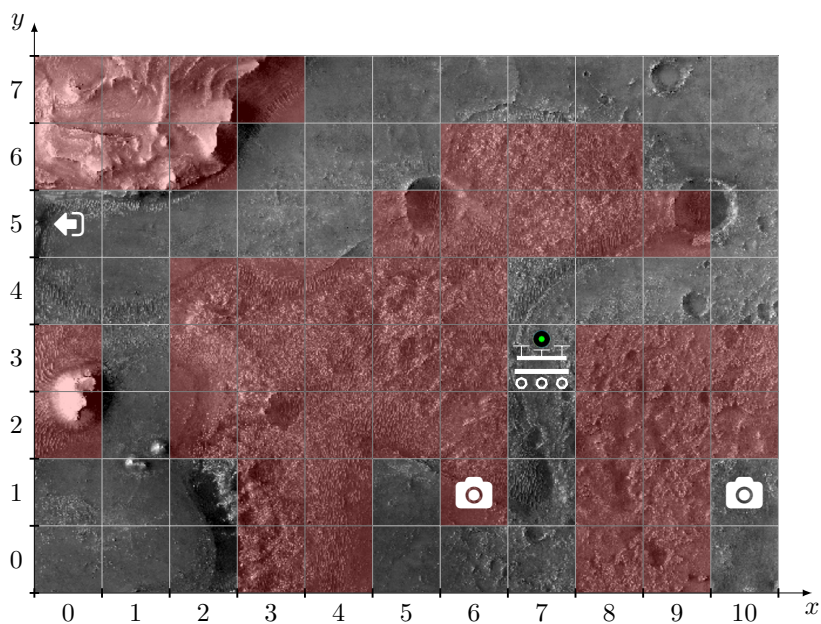}
  \end{center}
  \caption[Visualization of the Mars rover planning task (running example).]{Visualization of the Mars rover planning task used as a running example. The original image is from NASA/JPL-Caltech/University of Arizona\protect\footnotemark{} and shows the Jezero crater, where the green dot indicates the actual landing site of the Perseverance rover.
  Rover, drone, grid lines, red coloring (impassable cells for the rover), camera (goal: take pictures), and arrow (goal: cell to which the rover should travel with the drone) are added.}
  \label{fig:mars_example}
\end{figure}
\footnotetext{\url{https://mars.nasa.gov/resources/25621/perseverances-landing-spot-in-jezero-crater/} (Accessed: 2021-10-12)}

\section{Complexity and Compilations}
A natural question that arises is how hard planning is in general. More precisely, \emph{the bounded plan existence problem} of planning is the problem of deciding for a given planning task \task{} whether there exists a plan. It was shown that this problem is \PSPACE{}-complete for \name{strips} \autocite{bylander-aaai1997} and \name{sas}$^+$ planning tasks \autocite{backstrom-nebel-compint1995}.

\begin{theorem}[\cite{bylander-aij1994}]
  \label{thm:bounded-plan-existence}
  Bounded plan existence of planning is \PSPACE{}-complete. \qed
\end{theorem}

Interestingly, if two planning formalisms belong to the same complexity class, it does not directly mean that both have the same expressive power.
Expressive power is a measure of how concisely planning domains and plans can be expressed in a given formalism with compilation schemes \autocite{nebel-jair2000}. In the further course of this work, we will use compilation schemes to show whether and under which conditions certain model extensions such as derived variables or state-dependent action costs can be compiled away, i.e., expressed in the original formalism as a classical planning task.

The following definition formalizes compilation schemes, which translate from one planning formalism to another while preserving plan existence and polynomially preserving task sizes.

\begin{definition}[Compilation Scheme]
  \label{def:compilation-scheme}
  A compilation scheme or, in short, a compilation is a tuple of functions $\compfun = \langle \compdsfun, \compifun, \compgfun \rangle$ on planning domains that induces a function on planning tasks as follows:
  \[\comp(\task) =  \bigr\langle
    \compdsfun(\domain), \init \cup \compifun(\domain) ,
    \goal \cup \compgfun(\domain), \limit \bigr\rangle\text{,}\] where $\task = \langle \domain, \init, \goal, \limit \rangle$, satisfying the following conditions:
  \begin{enumerate}
    \item\label{I:compilation:reduction}
          there exists a plan for $\comp(\task)$ iff there exists a plan for $\task$, and
    \item\label{II:compilation:size}
          the size of the results of $\compdsfun, \compifun$, and $\compgfun$ is polynomial in the size of the arguments.
  \end{enumerate}
\end{definition}

Note that our definition of a compilation is \say{cost-sensitive}. 
By not allowing the cost bound to change, we ensure that for every plan in the original task, there must be a plan with cost at most $\limit$ in the target task, and vice versa.
Besides preserving planning task sizes polynomially, another desirable property is preservation of plan lengths. This is captured by the following definition.

\begin{definition}[Compilations Preserving Plan Length]
  \label{def:compilation-preserving-plan-length}
  A compilation scheme $\compfun$ is said to \emph{preserve plan length exactly} (modulo an additive constant) if for every plan $\plan{}$ solving an instance $\task$, there exists a plan $\plan'$ solving $\comp(\task)$ with $\size{\plan{}'}~ \leq \size{\plan{}} + k$ for some constant $k \in \mathbb{N}_0$. It is said to \emph{preserve plan length linearly} if $\size{\plan'}~ \leq c \cdot \size{\plan} + k$ for constants $c \in \mathbb{N}_0$ and $k \in \mathbb{N}_0$, and to \emph{preserve plan length polynomially} if $\size{\plan'}~ \leq p(\size{\plan},\size{\task})$ for some polynomial $p$.
\end{definition}

Intuitively, compilability while exactly preserving the plan length shows that the planning formalism we use as the target formalism is at least as expressive as the source formalism. If compilation requires polynomial or even exponential plan growth, this can lead to an infeasible challenge for planning algorithms that indicates an increase in expressive power \autocite{nebel-jair2000,thiebaux-et-al-aij2005}.

\section{Symbolic Search}
Symbolic search is a state space exploration technique that has its origin in
the field of Model Checking \autocite{mcmillan-1993}. Symbolic search algorithms resemble
their explicit counterparts, but expand and generate whole sets of states in contrast to individual states.
In symbolic search, a set of states $S \subseteq \states$ is represented by its \emph{characteristic function}
$\charf_S$, which is a Boolean function $\charf_S : \states \to \{0,1\}$ that represents whether a given
state belongs to $S$ or not.
More precisely, states contained in $S$ are mapped to $1$ and all others to $0$, i.e., $\charf_S(s) = 1$ if $s \in S$ and $\charf_S(s) = 0$ otherwise.
Similarly, operators $\operators$ can be represented as so-called \emph{transition relations (TRs)}, which are sets of state pairs, namely predecessor and successor states. The characteristic function of a transition relation $T_O$ representing a set of operators $O \subseteq \operators$ is a function $\charf_{T_O}: \states \times \states \to \{ 0,1\}$ that maps all pairs of states $(s,s')$ to true iff successor $s' \in \states$ is reachable from predecessor $s \in \states$ by applying an operator $o \in O$. Given a set of states $S$ and a TR $T_O$, the \emph{image} (\emph{preimage}) operator computes the set of successor (predecessor) states $S'$ of $S$ through $T_O$.
Note that a single TR can in general represent multiple operators with the same cost \autocite{torralba-et-al-icaps2013,torralba-et-al-aij2017}.

\emph{Symbolic (blind) search} describes a symbolic version of uniform cost search, also known as Dijkstra's algorithm \autocite{dijkstra-nummath1959}, which can be performed in different search directions.
\emph{Symbolic forward (blind) search} (progression) starts from the representation of the initial state $\charf_\init$, and iteratively computes the image until a set of states $S$ is found whose intersection with the goal $\charf_\goal$ is non-empty, i.e., $\charf_S \land \charf_\goal \neq \bot$.
The \emph{open} and \emph{closed list} are represented as lists of state sets partitioned into subsets with identical $g$-values, where the \emph{$g$-value} describes the cost required to reach these states.
During the search, the closed list is used to track and prune states that have already been expanded.
\emph{Symbolic backward (blind) search} (regression) can be realized by starting with the goal states, applying the preimage operation until the initial state is found.
In \emph{symbolic bidirectional (blind) search}, both forward and backward symbolic search are performed simultaneously, maintaining two symbolic searches with separate open and closed lists. A search step consists either of a backward or a forward search step (and modifies the respective open and closed lists).  If a state of the current search direction is expanded, which is already contained in the closed list of the search in the opposite direction, a goal path is found.
In general, all strategies, which switch iteratively between both search
directions, guarantee optimality if the termination criterion is chosen accordingly \autocite{pohl-tr1969}.

Finally, a \emph{plan reconstruction} \autocite{torralba-phd2015} is performed to obtain the final plan.
In explicit search, each (search) node keeps track of its parent node, making it easy to construct a plan when a goal state is found.
In symbolic search, however, the parents are not directly known, but all parents are stored in the
closed list with their reachability costs. 
Therefore, it is possible to perform a greedy search, which opposes the actual search direction with the perfect heuristic (cost of shortest path) obtained by the corresponding closed list.
In forward (backward) symbolic search the plan is constructed by a greedy backward (forward) search starting with a found goal state (the initial state). The plan reconstruction procedure iterates over all operators (descending cost)
and selects an explicit predecessor (successor) contained in the closed
list. The latter process is repeated until the initial state (a goal state) is
reached. 
Note that a greedy search in combination with the perfect heuristic leads the search directly from a starting state to a target state, making the runtime of the plan reconstruction negligible with respect to the actual search.
For bidirectional search, a greedy best-first search is performed twice, both opposing the actual search direction. More specifically, both plan reconstructions are initialized with the meeting point and one search is a regression to the initial state, while the other search is a progression to the goal states.

\section{Decision Diagrams}
In symbolic search the most prominent way to represent (characteristic) functions are decision diagrams such as Binary Decision Diagrams (BDDs) \autocite{bryant-dac1985}, Algebraic Decision
Diagrams (ADDs) \autocite{bahar-et-al-fmsd1997} or Edge-Valued Binary
Decision Diagrams (EVBDDs) \autocite{lai-et-al-ieeetc1996} which are data structures that offer a compromise between conciseness of representation and efficiency of manipulation \autocite{drechsler-becker-1998}. The main idea is to break down a function $f$ into subfunctions, so that $f$ can be reassembled from them.

\begin{definition}[Binary Decision Diagram]\label{def:bdd}
  A \emph{Binary Decision Diagram (BDD)} $B_S$ is a directed acyclic graph with a single root node and two terminal nodes: the 0-sink and the 1-sink. Each inner node corresponds to a binary\footnote{Note that each finite domain variable $v \in \vars$ can be represented by $\lceil \log_2 |\vardomain_v| \rceil$ binary variables.
    Although there are generalizations for each type of decision diagram to support multi-valued variables directly, this dissertation mainly considers decision diagrams that support only binary variables, unless explicitly stated otherwise.
  }
  variable $v \in \vars$ and has two successors, where the \emph{low edge} represents that variable $v$ is false, while the \emph{high edge} represents that variable $v$ is true. By traversing the BDD according to a given assignment, the represented Boolean function $\charf_S$ can be evaluated.
\end{definition}

ADDs and EVBDDs have been successfully used to represent numerical functions $f : \states \to \mathbb{Q} \cup \{\infty\}$ \autocite{hansen-et-al-sara2002,torralba-et-al-ijcai2013,speck-et-al-ipc2018} in symbolic planning. An \emph{Algebraic Decision Diagrams (ADD)} $A_f$ is similar to a BDD, but has an arbitrary number of terminal nodes with different discrete values including real numbers. An \emph{Edge-Valued Binary Decision Diagram (EVBDD)} $E_f$ is a rooted directed acyclic graph with edge values, a dangling incoming edge to the root node and a single terminal node. The function $f$ can be evaluated by traversing the graph according to the variable assignment and simultaneously adding up the edge weights. The resulting sum is the function value for the corresponding variable assignment. In practice, the generalization of EVBDDs, so-called \emph{Edge-Valued Multi-valued Decision Diagrams (EVMDDs)} \autocite{ciardo-siminiceanu-fmcad2002}, where variables can be multi-valued, is more common in planning than EVBDDs \autocite{geisser-et-al-ijcai2015,geisser-et-al-icaps2016,mattmueller-et-al-aaai2018}.

Decision diagrams are typically considered in a reduced and ordered form \autocite{becker-molitor2008} and can represent exponentially many states requiring only polynomial space. A decision diagram is called \emph{ordered} if on all paths from the root to a sink variables appear in the same order. A decision diagram is called \emph{reduced} if isomorphic subgraphs are merged and any node is eliminated whose two children are identical. For fixed variable orders, reduced and ordered decision diagrams are unique \autocite{bryant-ieeecomp1986,bahar-et-al-fmsd1997,lai-et-al-ieeetc1996}.
Note that for EVMDDs the corresponding edge values must be taken into account. From now on we only talk about reduced and ordered decision diagrams and assume a fixed variable order.

\begin{definition}[Decision Diagram Size]\label{def:dd-size}
  The \emph{size} $|D|$ of a decision diagram $D$ is the number of nodes of $D$.
\end{definition}

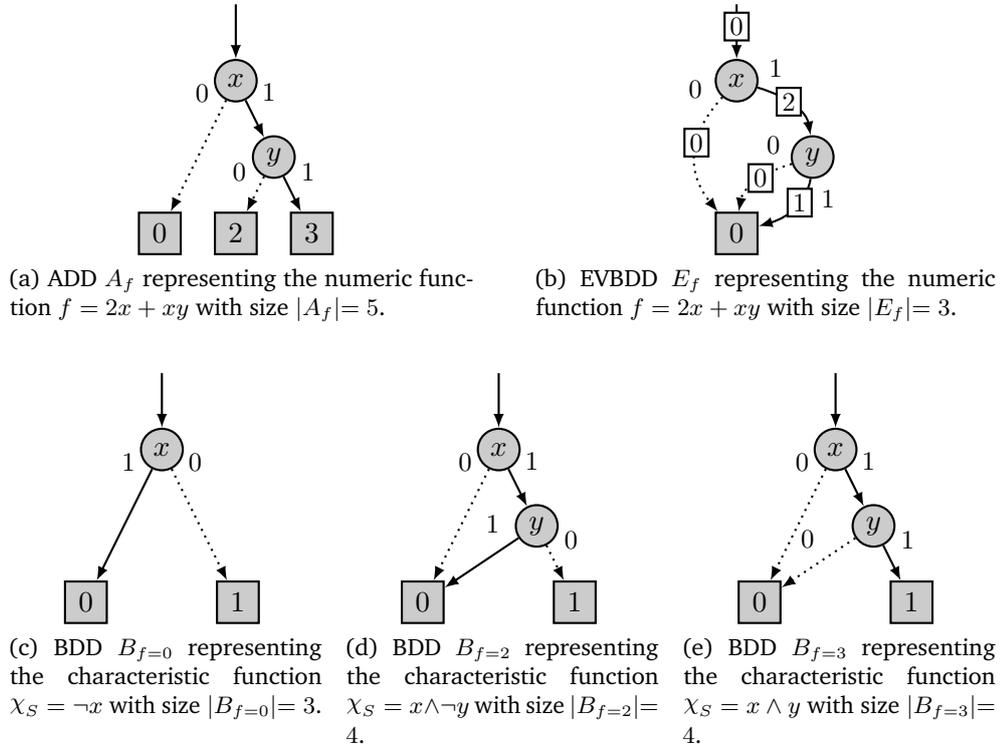
\begin{figure}[tbp]
  \centering
  \subfloat[ADD $A_f$ representing the numeric function $f = 2x + xy$ with size $|A_f|=5$.\label{fig:add}]{
    \centering
    \makebox[0.45\textwidth][c]{
      \begin{tikzpicture}[%
        costnode/.style={pos=0.6,rectangle,thick,
                inner sep=2pt,draw,fill=white,text=black,font=\normalsize},%
        decisionnode/.style={circle,thick,minimum size=4mm,
                inner sep=2pt,draw,fill=white!80!black,text=black,font=\normalsize},%
        xscale=2,yscale=1.35,>=latex]
    \node[] (before) at (0,2.35) {}; 
    \node[decisionnode, minimum size=0.55cm] (x) at (0,1.5) {$x$};
    \node[decisionnode, minimum size=0.55cm] (y) at (0.25,0.75) {$y$};
    \node[draw,thick,fill=white!80!black,rectangle, minimum size=0.55cm]
    (after0) at (-0.5,0) {{$0$}};
    \node[draw,thick,fill=white!80!black,rectangle, minimum size=0.55cm]
    (after1) at (0,0) {{$2$}};
    \node[draw,thick,fill=white!80!black,rectangle, minimum size=0.55cm]
    (after2) at (0.5,0) {{$3$}};
    \draw[->, thick] (before) to (x);
    \draw[->, thick, dotted] (x) to[bend right=0,label distance=0mm,edge
    label={\small{$0$}},swap,pos=0.1] (after0);
    \draw[->, thick] (x) to[bend left=0,label distance=0mm,edge
    label={\small{$1$}},pos=0.3] (y);
    \draw[->, thick, dotted] (y) to[bend right=0,label distance=0mm,edge
    label={\small{$0$}},swap,pos=0.4] (after1);
    \draw[->, thick] (y) to[bend left=0,label distance=0mm,edge
    label={\small{$1$}},pos=0.4] (after2);
\end{tikzpicture}
    }
  }\hfill
  \subfloat[EVBDD $E_f$ representing the numeric function $f = 2x + xy$ with size $|E_f| = 3$.\label{fig:evbdd}]{
    \centering
    \makebox[0.45\textwidth][c]{
      \begin{tikzpicture}[%
                costnode/.style={pos=0.6,rectangle,thick,solid,
                                inner sep=2pt,draw,fill=white,text=black,font=\small},%
                decisionnode/.style={circle,thick,minimum size=4mm,
                                inner sep=2pt,draw,fill=white!80!black,text=black,font=\normalsize},%
                xscale=2,yscale=1.35,>=latex]
        \node[] (before) at (0,2.35) {}; 
        \node[decisionnode, minimum size=0.55cm] (x) at (0,1.5) {$x$};
        \node[decisionnode, minimum size=0.55cm] (y) at (0.5,0.75) {$y$};
        \node[draw,thick,fill=white!80!black,rectangle, minimum size=0.55cm]
        (after0) at (0,0) {{$0$}};
        \draw[->, thick] (before) to[bend right=0,label distance=0mm,edge
        label={},swap,pos=0.3] node[costnode,pos=0.4] {$0$} (x);
        \draw[->, thick, dotted] (x) to[bend right=30,label distance=0mm,edge
        label={\small{$0$}},swap,pos=0.1] node[costnode,pos=0.4] {$0$} (after0);
        \draw[->, thick] (x) to[bend left=30,label distance=0mm,edge
        label={\small{$1$}},pos=0.01] node[costnode,pos=0.4] {$2$} (y);
        \draw[->, thick, dotted] (y) to[bend right=30,label distance=0mm,edge
        label={\small{$0$}},swap,pos=0.01] node[costnode,pos=0.4] {$0$} (after0);
        \draw[->, thick] (y) to[bend left=30,label distance=0mm,edge
        label={\small{$1$}},pos=0.01] node[costnode,pos=0.4] {$1$} (after0);
\end{tikzpicture}
    }
  }
  \\
  \subfloat[BDD $B_{f=0}$ representing the characteristic function $\charf_S = \lnot x$ with size $|B_{f=0}| = 3$.\label{fig:bdd_0}]{
  \centering
  \makebox[0.3\textwidth][c]{
    \begin{tikzpicture}[%
                costnode/.style={pos=0.6,rectangle,thick,
                                inner sep=2pt,draw,fill=white,text=black,font=\normalsize},%
                decisionnode/.style={circle,thick,minimum size=4mm,
                                inner sep=2pt,draw,fill=white!80!black,text=black,font=\normalsize},%
                xscale=2,yscale=1.35,>=latex]
        \node[] (before) at (0,2.35) {}; 
        \node[decisionnode, minimum size=0.55cm] (x) at (0,1.5) {$x$};
        \node[draw,thick,fill=white!80!black,rectangle, minimum size=0.55cm]
        (after0) at (-0.5,0) {{$0$}};
        \node[draw,thick,fill=white!80!black,rectangle, minimum size=0.55cm]
        (after1) at (0.5,0) {{$1$}};
        \draw[->, thick] (before) to (x);
        \draw[->, thick, dotted] (x) to[bend right=0,label distance=0mm,edge
        label={\small{$0$}},pos=0.1] (after1);
        \draw[->, thick] (x) to[bend left=0,label distance=0mm,edge
        label={\small{$1$}},swap,pos=0.1] (after0);
\end{tikzpicture}
  }
  }\hfill
  \subfloat[BDD $B_{f=2}$ representing the characteristic function $\charf_S = x \land \lnot y$ with size $|B_{f=2}| = 4$.\label{fig:bdd_2}]{
  \centering
  \makebox[0.3\textwidth][c]{
    \begin{tikzpicture}[%
                costnode/.style={pos=0.6,rectangle,thick,
                                inner sep=2pt,draw,fill=white,text=black,font=\normalsize},%
                decisionnode/.style={circle,thick,minimum size=4mm,
                                inner sep=2pt,draw,fill=white!80!black,text=black,font=\normalsize},%
                xscale=2,yscale=1.35,>=latex]
        \node[] (before) at (0,2.35) {}; 
        \node[decisionnode, minimum size=0.55cm] (x) at (0,1.5) {$x$};
        \node[decisionnode, minimum size=0.55cm] (y) at (0.25,0.75) {$y$};
        \node[draw,thick,fill=white!80!black,rectangle, minimum size=0.55cm]
        (after0) at (-0.5,0) {{$0$}};
        \node[draw,thick,fill=white!80!black,rectangle, minimum size=0.55cm]
        (after1) at (0.5,0) {{$1$}};
        \draw[->, thick] (before) to (x);
        \draw[->, thick, dotted] (x) to[bend right=0,label distance=0mm,edge
        label={\small{$0$}},swap,pos=0.1] (after0);
        \draw[->, thick] (x) to[bend left=0,label distance=0mm,edge
        label={\small{$1$}},pos=0.3] (y);
        \draw[->, thick, dotted] (y) to[bend right=0,label distance=0mm,edge
        label={\small{$0$}},pos=0.4] (after1);
        \draw[->, thick] (y) to[bend left=0,label distance=0mm,edge
        label={\small{$1$}},swap,pos=0.1] (after0);
\end{tikzpicture}
  }
  }\hfill
  \subfloat[BDD $B_{f=3}$ representing the characteristic function $\charf_S = x \land y$ with size $|B_{f=3}| = 4$.\label{fig:bdd_3}]{
  \centering
  \makebox[0.3\textwidth][c]{
    \begin{tikzpicture}[%
                costnode/.style={pos=0.6,rectangle,thick,
                                inner sep=2pt,draw,fill=white,text=black,font=\normalsize},%
                decisionnode/.style={circle,thick,minimum size=4mm,
                                inner sep=2pt,draw,fill=white!80!black,text=black,font=\normalsize},%
                xscale=2,yscale=1.35,>=latex]
        \node[] (before) at (0,2.35) {}; 
        \node[decisionnode, minimum size=0.55cm] (x) at (0,1.5) {$x$};
        \node[decisionnode, minimum size=0.55cm] (y) at (0.25,0.75) {$y$};
        \node[draw,thick,fill=white!80!black,rectangle, minimum size=0.55cm]
        (after0) at (-0.5,0) {{$0$}};
        \node[draw,thick,fill=white!80!black,rectangle, minimum size=0.55cm]
        (after1) at (0.5,0) {{$1$}};
        \draw[->, thick] (before) to (x);
        \draw[->, thick, dotted] (x) to[bend right=0,label distance=0mm,edge
        label={\small{$0$}},swap,pos=0.1] (after0);
        \draw[->, thick] (x) to[bend left=0,label distance=0mm,edge
        label={\small{$1$}},pos=0.3] (y);
        \draw[->, thick, dotted] (y) to[bend right=0,label distance=0mm,edge
        label={\small{$0$}},swap,pos=0.4] (after0);
        \draw[->, thick] (y) to[bend left=0,label distance=0mm,edge
        label={\small{$1$}},pos=0.4] (after1);
\end{tikzpicture}
  }
  }\hfill
  \caption{Visualization of different decision diagrams.}
  \label{fig:bdd_add}
\end{figure}

The size of a decision diagram depends strongly on the variable order, so that a good order can lead to an exponentially more compact decision diagram \autocite{edelkamp-kissmann-aaai2011}. For some functions the size of the corresponding decision diagram is exponential, independent of the underlying variable order \autocite{bryant-ieeecomp1986,edelkamp-kissmann-aaai2011}. Comparing the different types of decision diagrams, we can see that an EVBDD can be exponentially more compact than an ADD \autocite{roux-siminiceanu-nfm2010} representing additively separable functions such as $f: \{ 0,1 \}^{n+1} \to \{0, \dots, 2^{n+1}-1\}$ with $f(x_0, \dots, x_n) = \sum_{i=0}^{n} 2^i x_i$. Moreover, an ADD can be efficiently disassembled into multiple BDDs, one for each terminal node, in polynomial time and memory with respect to the ADD size by substituting terminal nodes \autocite{torralba-phd2015}. In practice, the main advantage of using BDDs over ADDs (and EVBDDs) is that decision diagram libraries such as \name{cudd} \autocite{somenzi-cudd2015} use techniques such as complement edges to store BDDs more compactly \autocite{brace-et-al-acm1990} and allow for more efficient operations \autocite{burch-et-al-tcad1994}.

\begin{example}\label{ex:dds}
  \Cref{fig:add,fig:evbdd} show the ADD $A_f$ and EVBDD $E_f$ representing the numeric function $f = 2x + xy$. 
  The function $f$ can also be represented as multiple BDDs by disassembling the ADD $A_f$ into three different BDDs, one for each terminal node. \Cref{fig:bdd_0,fig:bdd_2,fig:bdd_3} depict the BDDs $B_{f=z}$ representing all states for which the evaluation of function $f = 2x + xy$ is $z$. The variable order for all decision diagrams is $x \succ y$, i.e., x appears before y on each path.
\end{example}

From now on, 
when we refer to state sets, characteristic functions and numerical functions, we assume that they are represented as one of the corresponding decision diagrams and all logical or numerical operations are realized with the efficient and appropriate decision diagram-based operations using the \algname{apply} algorithm \autocite{bryant-ieeecomp1986,bahar-et-al-fmsd1997,lai-et-al-ieeetc1996}. If it is important which type of decision diagram we have used to represent certain functions, we will specify it explicitly.


\chapter{Symbolic Heuristic Search}\label{ch:symbolic_heuristic_search}
\chapterquote{While the main observations of this paper are both intuitive and pretty obvious, I still consider the work a significant contribution in focusing attention on a key and often overlooked difference between explicit and symbolic search [\dots].}{Reviewer \#3 (2020)}

\section*{Core Publication of this Chapter}
\renewcommand{\citebf}[1]{\textbf{#1}}
\begin{itemize}
    \item \fullcite{speck-et-al-icaps2020}
\end{itemize}
\renewcommand{\citebf}[1]{#1}

Explicit search and symbolic search have been shown to be strong and competitive approaches for optimal classical planning.
While explicit search usually occurs as variants of forward \astar{} search \autocite{hart-et-al-ieeessc1968}, which are equipped with strong and efficient goal-distance heuristics, symbolic search is usually performed as (bidirectional) blind search, i.e., without heuristics.
This naturally raises the question of why not combine the two techniques to obtain a planning algorithm that combines the state pruning of explicit heuristic search with the concise representation of symbolic search.
This combination is referred to as symbolic heuristic search, which includes a variety of symbolic generalizations of \astar{} for different decision diagrams such as \bddastar{} \autocite{edelkamp-reffel-ki1998}, \addastar{} \autocite{hansen-et-al-sara2002} and \evmddastar{} \autocite{speck-et-al-icaps2018}.

The underlying idea of symbolic \astar{} search is to precompute and represent a \emph{heuristic function} $h$ with decision diagrams, where $h: \states \to \mathbb{N}_0 \cup \{ \infty \}$ such that $h(s_\star) = 0$ for all $s_\star \in \goalstates$.
Symbolic operations are then used to assign the corresponding $f$-values ($f = g + h$) to each state of a set of states $S_g$ reachable with cost $g$.
As usual with \astar{}, symbolic \astar{} expands the states in ascending order of the $f$-value and uses a tie-breaking rule in favor of states with smaller $g$-values \autocite{kissmann-edelkamp-aaai2011,torralba-phd2015}.
To guarantee optimal solutions, a \emph{consistent heuristic} is assumed, i.e., its estimate is always less than or equal to the estimated distance from each direct successor to the goal, plus the cost to reach that successor, i.e., $h(s) \leq h(s[o]) + \constcostfun(o)$.
Note that any consistent heuristic is also admissible, i.e., never overestimates the cost of reaching a goal state \autocite{pearl-1984}.
Two important examples of consistent heuristics and at the same time the extreme cases are the perfect heuristic $\perfecth{}$, which maps any state $s$ to the cost of the cheapest path from $s$ to any goal state, and the blind heuristic $\heu{blind}{}$, which maps any state $s$ to $0$.

\begin{figure}
    \centering
    \subfloat[First step.\label{fid:bddastar_example_a}]{
        \centering
        \makebox[0.5\textwidth][c]{
            \begin{tikzpicture} [darkstyle/.style={draw,fill=gray!40,minimum size=20},scale=1]
    \path [draw, ->, thick] (0, -0.5) to (4, -0.5);
    \path [draw, ->, thick] (-0.75, 0) to (-0.75, 4);

    \node at (-0.5, -1) {$g$};
    \node at (-1.2, -0.5) {$h$};

    \foreach \x in {0,...,4}
    \node  at (1*\x,-1) {\x};

    \foreach \y in {0,...,4}
    \node  at (-1.2,1*\y) {\y};

    \foreach \x/\y/\i/\t in {0/2/1/ \init}{
            \node [darkstyle]  (\x\y) at (1*\x,1*\y) {\t};
            \node [] at (-0.225+\x, -0.2+\y)  {\small{$\i$}};
        }
    \foreach \x/\y in {1/1,1/2,1/3,1/4}{
            \node [draw, minimum size=20]  (\x\y) at (1*\x,1*\y) {};
        }
    \foreach \x/\y in {02/12, 02/11, 02/13, 02/14}{
            \path[draw, ->, thick] (\x) to (\y);
        }

\end{tikzpicture}
        }
    }
    \subfloat[Last step.\label{fid:bddastar_example_b}]{
        \centering
        \makebox[0.5\textwidth][c]{
            \begin{tikzpicture} [darkstyle/.style={draw,fill=gray!40,minimum size=20},scale=1]
    \path [draw, ->, thick] (0, -0.5) to (4, -0.5);
    \path [draw, ->, thick] (-0.75, 0) to (-0.75, 4);

    \node at (-0.5, -1) {$g$};
    \node at (-1.2, -0.5) {$h$};

    \foreach \x in {0,...,4}
    \node  at (1*\x,-1) {\x};

    \foreach \y in {0,...,4}
    \node  at (-1.2,1*\y) {\y};

    \foreach \x/\y/\i/\t in {0/2/1/ \init}{
            \node [darkstyle]  (\x\y) at (1*\x,1*\y) {\t};
            \node [] at (-0.225+\x, -0.2+\y)  {\small{$\i$}};
        }

    \foreach \x/\y/\i/\t in {1/1/2/, 1/2/3/, 2/1/4/, 1/3/5/, 2/2/6/, 3/1/7/,4/0/8/\goal}{
            \node [darkstyle]  (\x\y) at (1*\x,1*\y) {\t};
            \node [] at (-0.225+\x, -0.2+\y)  {\small{$\i$}};
        }

    \foreach \x/\y in {1/4,2/3,2/4,3/3,3/4,4/1,1/4,3/2}{
            \node [draw, minimum size=20]  (\x\y) at (1*\x,1*\y) {};
        }

    \foreach \x/\y in {02/12, 02/11, 02/13, 02/14, 11/21, 11/22, 12/21, 12/22, 12/23, 13/23,13/22,13/24,21/31,22/31,22/32,22/33,22/34,31/40,31/41}{
            \path[draw, ->, thick] (\x) to (\y);
        }

    \draw[dashed] (-0.0, 3.0) node [above] {\footnotesize $f=3$} -- (3.0, -0.0) ;
    \draw[dashed] (-0.0, 4.0) node [above] {\footnotesize $f=4$} -- (4.0, -0.0) ;

    \foreach \x/\y in {1/4,2/3,2/4,3/3,3/4,4/1,1/4,3/2}{
            \node [draw=red, sloped, cross out, line width=.5ex, minimum width=1.5ex, minimum height=1ex, anchor=center]  (\x\y) at (1*\x,1*\y) {};
        }

\end{tikzpicture}
        }
    }
    \caption[Illustration of BDDA$^\star$.]{Illustration of BDDA$^\star$, where the cells represent state sets $S_{g,h}$ and the arrows denote successor state sets. The gray cells are expanded in the order indicated by the numbers \autocite{torralba-gnad-icaps2017tutorial,torralba-phd2015}.
    }
    \label{fid:bddastar_example}
\end{figure}
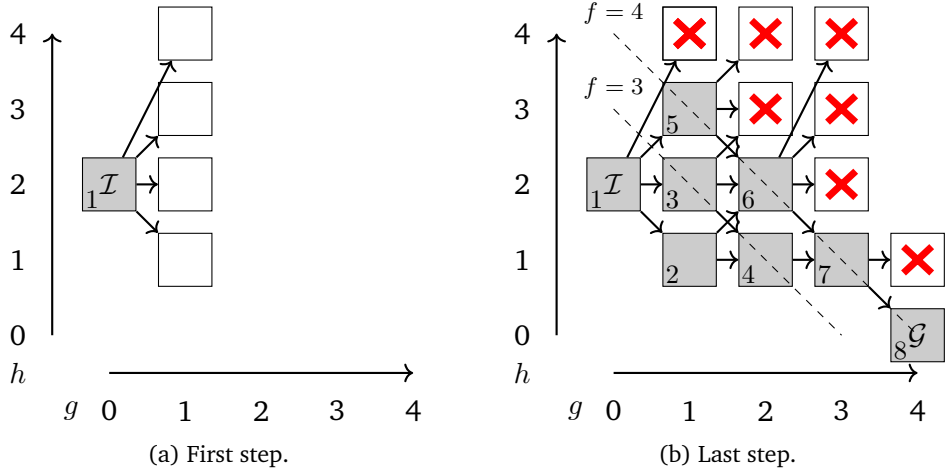

In \bddastar{}, a consistent heuristic function is precomputed and represented by multiple BDDs, one for each heuristic value $h$, where each BDD is used to represent the states $S_h$ with heuristic value $h$.
During search, states $S_g$ reachable with cost $g$ are partitioned according to their heuristic value by computing the intersections of $S_g$ and $S_h$ for each heuristic value $h$ resulting in sets of states $S_{g,h}$.
%
\Cref{ex:bddastar} exemplifies how \bddastar{} works.

\begin{example}\label{ex:bddastar}
    We consider a planning task with unit costs. \bddastar{} starts with the state set $S_{0,2} = \{ \init \}$, which contains only the initial state with a $g$-value of 0 and an $h$-value of $2$ (\Cref{fid:bddastar_example_a}).
    The expansion of $S_{0,2}$ leads to sets of states with $g$-value of $1$ and different $h$-values.
    Next, the state set $S_{1,1}$ is expanded, then $S_{1,2}$, and so on until a goal state is contained in $S_{4,0}$ (\Cref{fid:bddastar_example_b}).
\end{example}

While \bddastar{} utilizing a heuristic expands fewer states, symbolic blind search (\bddastar{} with $\heu{blind}{}$) potentially expands fewer sets of states.
Moreover, as \textcite{speck-et-al-icaps2020} have shown and we will examine in this chapter, the BDD representation of the state sets can deteriorate if the states are partitioned according to their $h$-values.

There are several heuristics that can be precomputed and represented using symbolic search and decision diagrams, leading to cutting-edge performance in explicit heuristic search \autocite{edelkamp-aips2002,franco-et-al-ijcai2017,franco-et-al-ipc2018b,moraru-et-al-ki2019}.
Thus, all the ingredients are present to allow a symbolic planner utilizing heuristics, as explicit planners do.
However, \textcite{jensen-et-al-aij2008} identified the partitioning of state sets according to their heuristic values as a bottleneck, because multiple arithmetic operations have to be performed during search.
This leads to different extension of \bddastar{} to overcome this bottleneck such as Lazy \bddastar{} \autocite{torralba-phd2015}, which delays the heuristic evaluation as long as possible, or Set\astar{} \autocite{jensen-et-al-aij2008}, which encodes the heuristic values as preconditions of actions resulting in multiple actions with costs according to the heuristic values.
However, empirical evaluations show that all versions of \bddastar{} perform better than symbolic blind search in some domains but overall symbolic bidirectional search without any heuristic  performs best \autocite{torralba-et-al-ijcai2016}.

In the remainder of this chapter, we describe and summarize the results of \textcite{speck-et-al-icaps2020}, which theoretically and empirically evaluate the search behavior of \bddastar{}.
On the theoretical side, this study reveals another fundamental problem of symbolic heuristics, namely that the use of a heuristic does not always improve the search performance of \bddastar{}, as it may affect the size of the representation.
In general, even the \emph{perfect heuristic} can exponentially deteriorate the search performance of symbolic \astar{}.
The empirical evaluation is consistent with these theoretical results.
Finally, we conclude this chapter with a discussion of the implications of these
findings.

\begin{figure}
    \begin{center}
        \begin{tikzpicture}[%
		costnode/.style={pos=0.6,rectangle,thick,
				inner sep=2pt,draw,fill=white,text=black,font=\normalsize},%
		decisionnode/.style={circle,thick,minimum size=4mm,
				inner sep=2pt,draw,fill=white!80!black,text=black,font=\normalsize},%
		xscale=2,yscale=1.35,>=latex]
	\newcommand*{\lh}{0.6}%
	\newcommand*{\vh}{0.2}%

	\node[draw,thick,fill=white!80!black,rectangle, minimum size=0.55cm] (after0) at (1-7*\vh,0) {$0$};
	\node[decisionnode, minimum size=0.55cm] (x7) at (1-0*\vh,1*\lh) {$x_7$};
	\node[decisionnode, minimum size=0.55cm] (x6) at (1-1*\vh,2*\lh) {$x_6$};
	\node[decisionnode, minimum size=0.55cm] (x5) at (1-2*\vh,3*\lh) {$x_5$};
	\node[decisionnode, minimum size=0.55cm] (x4) at (1-3*\vh,4*\lh) {$x_4$};
	\node[decisionnode, minimum size=0.55cm] (x3) at (1-4*\vh,5*\lh) {$x_3$};
	\node[decisionnode, minimum size=0.55cm] (x2) at (1-5*\vh,6*\lh) {$x_2$};
	\node[decisionnode, minimum size=0.55cm] (x1) at (1-6*\vh,7*\lh) {$x_1$};
	\node[decisionnode, minimum size=0.55cm] (x0) at (1-7*\vh,8*\lh) {$x_0$};

	\node[] (tmp) at (1.1,2.3*\lh) {};

	\draw[->, thick] (1-7*\vh,9*\lh) -- (x0);

	\draw[->, thick,dotted] (x6.west) -- (after0);
	\draw[->, thick,dotted] (x5.west) -- (after0);
	\draw[->, thick,dotted] (x4.west) -- (after0);
	\draw[->, thick,dotted] (x3.west) -- (after0);
	\draw[->, thick,dotted] (x2.west) -- (after0);
	\draw[->, thick,dotted] (x1.west) -- (after0);
	\draw[->, thick,dotted] (x0.west) -- (after0);

	\draw[->, thick] (x7) -- (after0);
	\draw[->, thick] (x6.east) to[bend left=10] (x7);
	\draw[->, thick] (x5.east) to[bend left=10] (x6);
	\draw[->, thick] (x4.east) to[bend left=10] (x5);
	\draw[->, thick] (x3.east) to[bend left=10] (x4);
	\draw[->, thick] (x2.east) to[bend left=10] (x3);
	\draw[->, thick] (x1.east) to[bend left=10] (x2);
	\draw[->, thick] (x0.east) to[bend left=10] (x1);

	\renewcommand{\vh}{2.4}

	\node[draw,thick,fill=white!80!black,rectangle, minimum size=0.55cm] (blind) at (\vh,0*\lh) {$1$};


	\renewcommand*{\lh}{0.8}%
	\renewcommand{\vh}{4.2}

	\node[draw,thick,fill=white!80!black,rectangle, minimum size=0.55cm] (after0) at (\vh-0.5,0*\lh) {$0$};
	\node[draw,thick,fill=white!80!black,rectangle, minimum size=0.55cm] (after1) at (\vh+0.5,0*\lh) {$1$};

	\node[decisionnode, minimum size=0.55cm] (v6) at (\vh-0.25,1*\lh) {$v_6$};

	\node[decisionnode, minimum size=0.55cm] (v50) at (\vh-0.25,2*\lh) {$v_5$};
	\node[decisionnode, minimum size=0.55cm] (v51) at (\vh+0.25,2*\lh) {$v_5$};

	\node[decisionnode, minimum size=0.55cm] (v40) at (\vh-0.25,3*\lh) {$v_4$};
	\node[decisionnode, minimum size=0.55cm] (v41) at (\vh+0.25,3*\lh) {$v_4$};
	\node[decisionnode, minimum size=0.55cm] (v42) at (\vh+0.75,3*\lh) {$v_4$};
	\node[decisionnode, minimum size=0.55cm] (v43) at (\vh+1.25,3*\lh) {$v_4$};

	\node[decisionnode, minimum size=0.55cm] (v30) at (\vh+-1.25,4*\lh) {$v_3$};
	\node[decisionnode, minimum size=0.55cm] (v31) at (\vh+-0.75,4*\lh) {$v_3$};
	\node[decisionnode, minimum size=0.55cm] (v32) at (\vh+-0.25,4*\lh) {$v_3$};
	\node[decisionnode, minimum size=0.55cm] (v33) at (\vh+0.25,4*\lh) {$v_3$};

	\node[decisionnode, minimum size=0.55cm] (v20) at (\vh+-1,5*\lh) {$v_2$};
	\node[decisionnode, minimum size=0.55cm] (v21) at (\vh+0,5*\lh) {$v_2$};

	\node[decisionnode, minimum size=0.55cm] (v1) at (\vh+-0.5,6*\lh) {$v_1$};

	\draw[->, thick,dotted] (v6) -- (after0);
	\draw[->, thick] (v6) -- (after1);

	\draw[->, thick,dotted] (v50) to[bend right=10]  (after0);
	\draw[->, thick] (v50) -- (after1);

	\draw[->, thick,dotted] (v51) -- (v6);
	\draw[->, thick] (v51) -- (after1);

	\draw[->, thick,dotted] (v40) to[bend right=15] (after0);
	\draw[->, thick] (v40) -- (after1);

	\draw[->, thick,dotted] (v41) -- (v6);
	\draw[->, thick] (v41) to[bend left=10]  (after1);

	\draw[->, thick,dotted] (v42) -- (v50);
	\draw[->, thick] (v42) -- (after1);

	\draw[->, thick,dotted] (v43) -- (v51);
	\draw[->, thick] (v43) -- (after1);

	\draw[->, thick,dotted] (v30) -- (after0);
	\draw[->, thick] (v30) -- (v6);

	\draw[->, thick,dotted] (v31) -- (v50);
	\draw[->, thick] (v31) to[bend left=24] (v51);

	\draw[->, thick,dotted] (v32) -- (v40);
	\draw[->, thick] (v32) -- (v41);

	\draw[->, thick,dotted] (v33) -- (v42);
	\draw[->, thick] (v33) -- (v43);

	\draw[->, thick,dotted] (v20) -- (v30);
	\draw[->, thick] (v20) -- (v31);

	\draw[->, thick,dotted] (v21) -- (v32);
	\draw[->, thick] (v21) -- (v33);

	\draw[->, thick,dotted] (v1) -- (v20);
	\draw[->, thick] (v1) -- (v21);

	\draw[->,thick,dotted] (x7.east) |- node[costnode, pos=0.75] {with $\perfecth$} (v1);

	\draw[->,thick,dotted] (tmp) -| node[costnode, pos=0.25] {with $\heu{blind}{}$} (blind.north);

\end{tikzpicture}
    \end{center}
    \caption[Visualization of two BDDs with exponential size difference.]{
        Visualization of two BDDs $B_S$ (with $h_{\text{\scriptsize blind}}$) and $B_{S'}$ (with $h_{\scriptsize \star}$) representing state sets $S$ and $S'$, where $B_{S'}$ is exponentially larger in the number of variables than $B_{S}$, although $S' \subsetneq S$ \autocite{speck-et-al-icaps2020}.}
    \label{fid:bddastar_proof}
\end{figure}
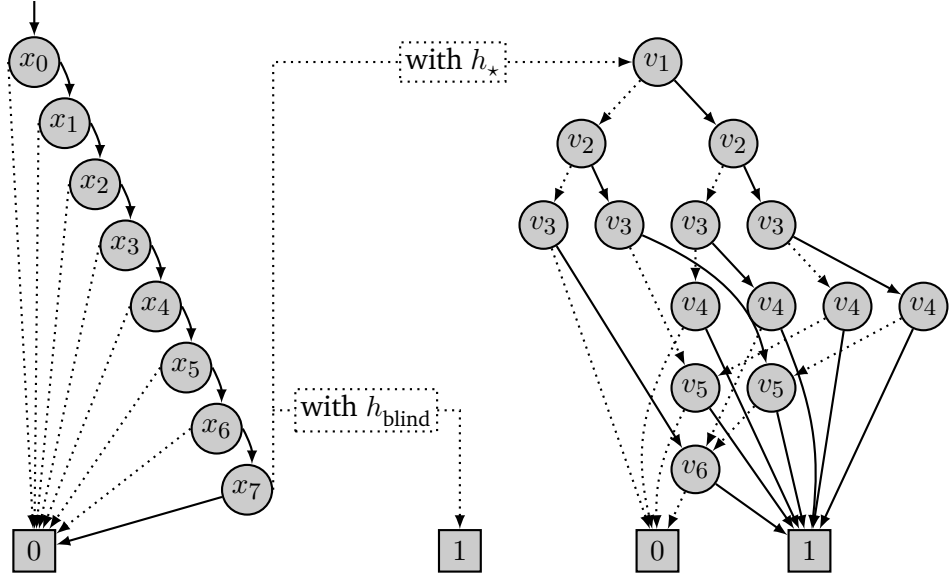

\section{Theoretical Results}

\textcite{speck-et-al-icaps2020} show that good goal-distance estimations in the form of heuristics are not the appropriate quantity to improve the search performance of symbolic heuristic search.
The reason for this is that, in contrast to explicit \astar{}, where every
consistent heuristic can only reduce the number of necessary node expansions (up to tie-breaking) and thus the search effort compared to blind search, in \bddastar{} no such guarantee exists.
In \bddastar{}, the size of expanded BDDs, i.e., number of BDD nodes, representing expanded states determines the search effort and thus the runtime.
As a BDD $B_{S'}$ can be exponentially larger than a BDD $B_{S}$ although the set of states $S'$ is a strict subset of $S$, i.e., $S' \subsetneq S$, it is not always beneficial to represent and expand fewer states (\Cref{fid:bddastar_proof}).
In other words, in explicit search, where the most promising states (search fringe) are simply enumerated explicitly, reducing the number of states to expand directly improves search performance.
However, in symbolic search, reducing the number of states to be expanded can lead to a larger representation size in form of BDDs and a fragmentation of the search fringe as the search progresses.

Similar to \say{must-expand} nodes in explicit search, \textcite{speck-et-al-icaps2020} introduce the notion of \emph{expansion size} for \bddastar{} as the cumulative size of BDDs that must always be expanded by \bddastar{} before finding an optimal solution to measure search effort.
Using this notion it is possible to prove that even under the best possible and unrealistic circumstances, namely the perfect heuristic, the search effort of \bddastar{} can be exponentially larger than the search effort of symbolic
search without heuristic, and vice versa (\Cref{thm:bddastar}).

\begin{theorem}[\cite{speck-et-al-icaps2020}]\label{thm:bddastar}
    Using the perfect heuristic $\perfecth$ instead of the blind heuristic $\heu{blind}{}$ can \emph{decrease} or \emph{increase} the expansion size of \bddastar{} exponentially in the size of the planning task for a given variable ordering. \qed
\end{theorem}

While the result that the search performance of \bddastar{} can be exponentially improved when $\perfecth$ is used instead of $\heu{blind}{}$ is less surprising, the opposite result is very surprising.
This highlights the difference between explicit and symbolic search: the representation of the search fringe must be concise, and in symbolic search the decision diagram size is not directly related to the number of states represented.
To prove that using $\perfecth$ instead of $\heu{blind}{}$ can exponentially increase the expansion size and thus exponentially deteriorate search performance, \textcite{speck-et-al-icaps2020} constructed a family of planning tasks $\task_n$ parameterized over the number of (relevant) variables.
Solving tasks of this family of planning tasks $\task_n$ with \bddastar{} using the perfect heuristic prunes states from the expanded set of states (search fringe) so that the representation size increases exponentially.
\Cref{fid:bddastar_proof} shows the core idea of the proof, with two BDDs representing the search fringe, i.e., the expanded state set, in the last expansion step of \bddastar{} using $\perfecth$ or $\heu{blind}{}$ when solving $\task_3$.
The key observation is that when $\heu{blind}{}$ is used, no states are pruned according to their values over variables $v_i$, resulting in a compact BDD representation.
However, if $\perfecth$ is used, states are pruned according to their values over the variables $v_i$ and the function $(v_1 \land v_{n+1}) \lor \dots \lor (v_{n} \land v_{2n})$ has to be represented, which under certain variable orders requires a BDD with exponentially many nodes \autocite{kissmann-phd2012}.

Finally, \textcite{speck-et-al-icaps2020} show that these theoretical results also hold for several other symbolic \astar{} variants such as Lazy \bddastar{}, Set\astar{}, \addastar{}, and \evmddastar{}, as well as for their bidirectional extensions.

\begin{figure}
    \subfloat[Search Effort (BDD nodes)\label{fig:bddastar_data_a}]{
        \centering
        \makebox[1\textwidth][c]{
            \resizebox{0.95\textwidth}{!}{
                \input{figures/bdd-nodes-fwd.tex}
                \hfill
                \input{figures/bdd-nodes-bid.tex}
            }
        }
    }\\
    \subfloat[Image Time (seconds)\label{fig:bddastar_data_b}]{
        \centering
        \makebox[1\textwidth][c]{
            \resizebox{0.95\textwidth}{!}{
                \includegraphics{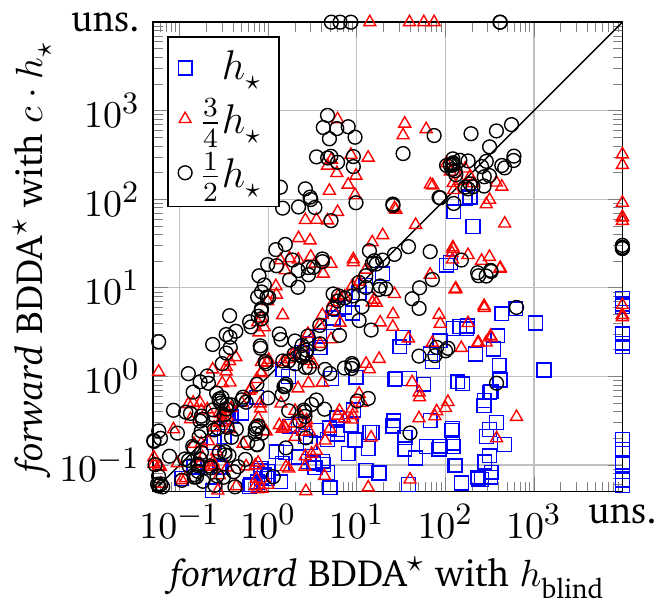}
                \hfill
                \includegraphics{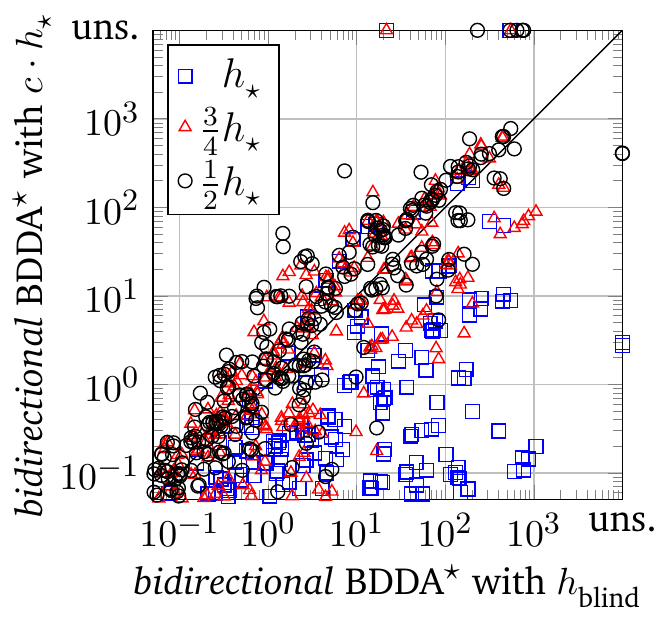}
            }
        }
    }\\
    \subfloat[Overall Runtime (seconds)\label{fig:bddastar_data_c}]{
        \centering
        \makebox[1\textwidth][c]{
            \resizebox{0.95\textwidth}{!}{
                \includegraphics{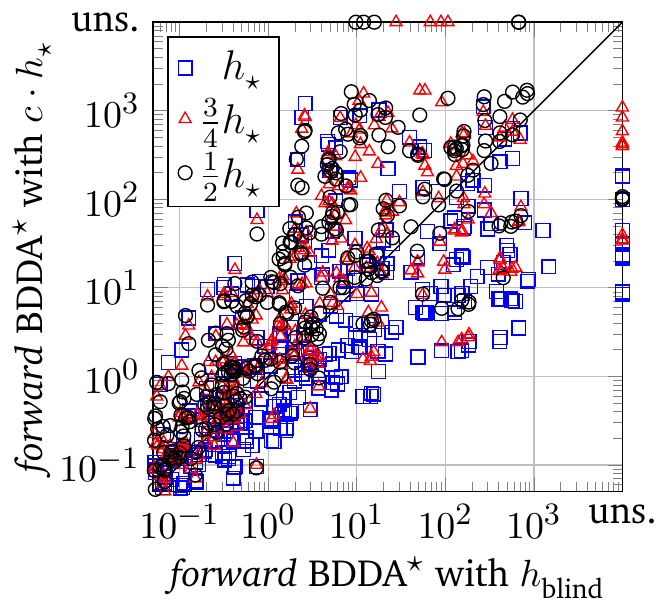}
                \hfill
                \includegraphics{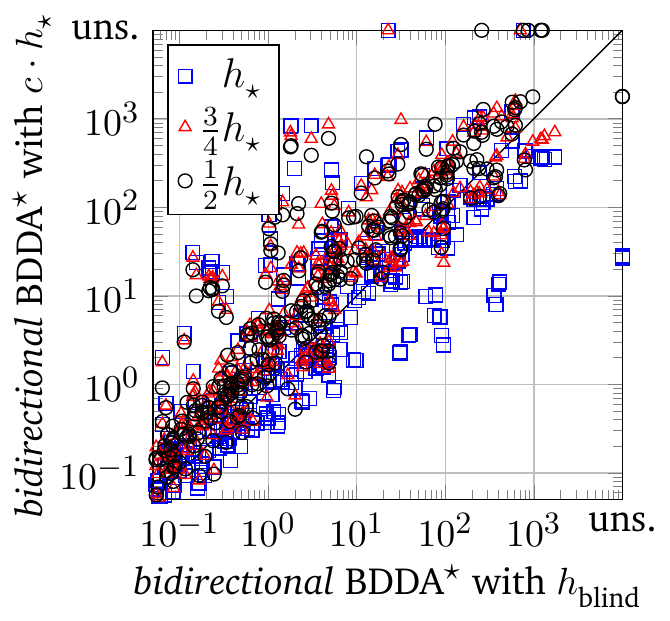}
            }
        }
    }
    \vspace{-0.2cm}
    \caption[Comparison of BDDA$^\star$ with two different heuristics.]{
        Comparison of BDDA$^\star$ with the blind heuristic and with given (fraction) perfect heuristics \autocite{speck-et-al-icaps2020}.
        With \emph{uns.} we refer to instances that were not solved within the time and memory limits.
    }
    \label{fig:bddastar_data}
\end{figure}

\section{Empirical Results}
\textcite{speck-et-al-icaps2020} empirically investigated forward and bidirectional \bddastar{} with precomputed fraction perfect heuristics on domains from the optimal track from the International Planning Competitions between 1998 and 2018.
A heuristic $c \cdot h$ is called fraction perfect if it assigns to all states the values of the perfect heuristic multiplied by a constant $0 \leq c \leq 1$, where $0\perfecth = \heu{blind}{}$ and $1\perfecth = \perfecth$ are important extreme cases.
Note that \bddastar{} with blind heuristic $0\perfecth = \heu{blind}{}$ corresponds to symbolic blind search.

\Cref{fig:bddastar_data_a} compares the search effort of \bddastar{} with the blind heuristic and the search effort of \bddastar{} with the fraction perfect heuristics.
While the unrealistic case of the perfect heuristic almost always leads to a reduction in search effort in practice, when we consider more realistic cases, namely the fraction perfect heuristic, we can see that the search effort of \bddastar{} can improve or deteriorate to the same extent.

\Cref{fig:bddastar_data_b} shows the expansion time of \bddastar{}, i.e., the cumulative time required to generate all successors with the image operation.
A correlation is observed between the search effort and the expansion time, which also empirically shows that the introduced notion of search effort is an adequate quantity to measure the performance of \bddastar{}.
This highlights the fundamental problem highlighted by \textcite{speck-et-al-icaps2020}: using a heuristic does not always improve the search performance of \bddastar{}.

Finally, \Cref{fig:bddastar_data_c} shows the total running time of \bddastar{} without the time for computing the corresponding heuristic.
Comparing the expansion time with the total running time, we find that \bddastar{} with fraction perfect heuristics has a higher time increase than \bddastar{} with $\heu{blind}{}$, which is due to the time-consuming partitioning of the state sets by heuristic values \autocite{jensen-et-al-aij2008}.

Overall, these empirical results show that also in practice a heuristic can improve or deteriorate the search effort of \bddastar{} to the same extent.
These empirical results are consistent with the presented theoretical results.
Furthermore, \textcite{speck-et-al-icaps2020} show that it appears to be domain
dependent whether a heuristic helps \bddastar{} as the structure of the reachable search space appears to play a central role.

\section{Discussion}\label{sec:symbolic_search_discussion}
As the results of \textcite{speck-et-al-icaps2020} show, using goal-distance estimators, i.e., heuristics, to prune states in symbolic heuristic search does not always pay off.
Indeed, even the perfect heuristic can exponentially deteriorate the search performance of \bddastar{} and other symbolic \astar{} variants.
Since the search performance of \bddastar{} is not directly related to the number of expanded states, but to the size of the involved BDDs during the search, \textcite{speck-et-al-icaps2020} suggest using heuristic functions that can provide a nice structure of the involved decision diagrams or even give a size guarantee.
One possible candidate are potential heuristics \autocite{pommerening-et-al-aaai2015}, which have recently been shown to be encodable as operator potentials within symbolic heuristic search, resulting in state-of-the-art performance \autocite{fiser-et-al-icaps2021wshsdip}.

The fact that symbolic blind search does not use heuristics and yet can compete with explicit heuristic search in cost-optimal planning \autocite{edelkamp-et-al-aaai2015,torralba-et-al-aij2017,speck-et-al-icaps2020} has the advantage that it is not constrained by the limits of heuristics.
Considering expressive extensions of classical planning, such as conditional effects, axioms, and many more, explicit heuristic search planners rarely support them.
The reason is that it is very challenging to design admissible heuristics that are both informative and fast to compute when considering such extensions.
Symbolic blind search, however, provides a good basis to support these extensions, since it does not rely on heuristics and performs an efficient exhaustive search.
In the next chapters, we will take advantage of this and show how symbolic search can be adapted to support expressive extensions in cost-optimal planning.

\chapter{Axioms and Derived Variables}\label{ch:axioms}
\chapterquote{The grand aim of all science [is] to cover the greatest number of empirical facts by logical deduction from the smallest possible number of hypotheses or axioms.}{Albert Einstein (1950)}

\renewcommand{\kiviatPredicates}{2}
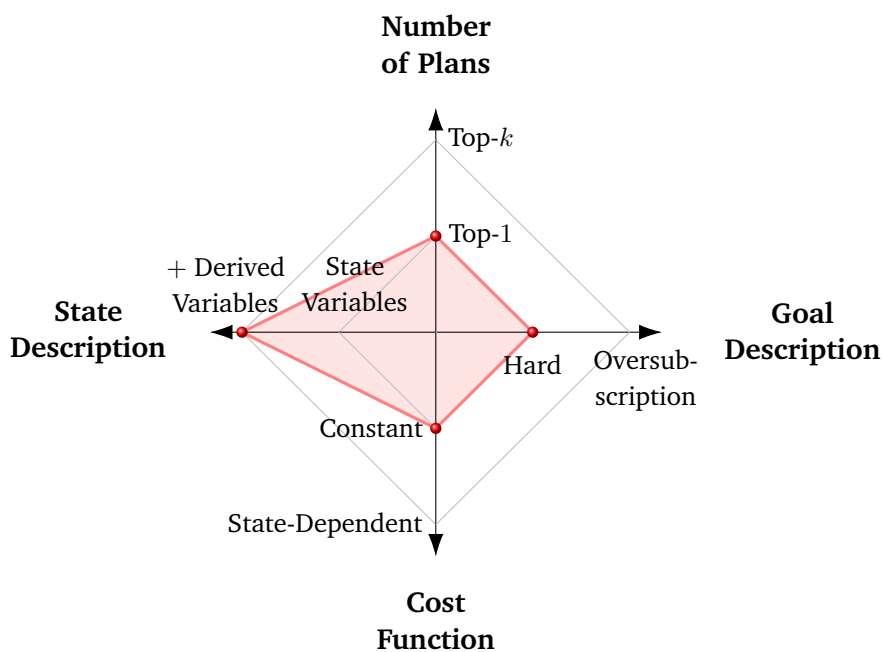
\begin{figure}[t]
    \begin{center}
        \begin{tikzpicture}
    \tkzKiviatDiagram[
    radial style/.style ={-{Latex[length=3mm, width=2mm]}},
    scale=0.85,
    label space=1.5,
    radial = 1,
    gap = 1.5,
    step = 1,
    lattice = 2]{
    \hspace{2.5cm}\mbox{\parbox{4cm}{{\begin{center}\textbf{Goal}\\\textbf{Description} \end{center}}}},
    {\textbf{Number of Plans}},
    \hspace{-1.5cm}\mbox{\parbox{3cm}{{\begin{center}\textbf{State}\\\textbf{Description}\end{center}}}},
    \textbf{Cost Function}}

    \ifbool{kiviatRed}{
        \tkzKiviatLine[very thick,color=red!50,
            fill=red!30,
            opacity=.35,
            mark=ball,
            mark size=2.5pt,
            ball color=red](\kiviatGoal{},\kiviatTopk{},\kiviatPredicates{},\kiviatCost{})
    }

    \ifbool{kiviatBlue}{
        \tkzKiviatLine[very thick,color=blue!50,
            fill=blue!30,
            opacity=.35,
            mark=ball,
            mark size=2.5pt,
            ball color=blue](1,1,1,1)
    }

    \ifbool{kiviatFull}{
    \tkzKiviatLine[very thick,color=darkgreen!50,
        fill=green!30,
        opacity=.35,
        mark=ball,
        mark size=2.5pt,
        ball color=darkgreen](1,2,2,2)
    \tkzKiviatLine[very thick,color=blue!50,
        fill=blue!30,
        opacity=.35,
        mark=ball,
        mark size=2.5pt,
        ball color=blue](1,1,1,1)
    \tkzKiviatLine[very thick,color=red!50,
        mark=ball,
        mark size=2.5pt,
        ball color=red](\kiviatGoal{},\kiviatTopk{},\kiviatPredicates{},\kiviatCost{})
    \LegendBox[shift={(-0cm,-2.25cm)}]{current bounding box.south west}%
    {
    red!100/{Forward symbolic search (contribution of this thesis)},
    asparagus!100/{Bidirectional symbolic search (contribution of this thesis)},
    blue!100/{Bidirectional symbolic search (previous state of the art)}}
    }
    {}

    \draw[] node[align=center,fill=none] at (1.5,-0.5) {\small Hard};
    \draw[] node[align=center,fill=none] at (3.25,-0.75) {\small Oversub-\\ \small scription};

    \draw[] node[align=center,fill=none] at (0.7,1.5) {\small Top-$1$};
    \draw[] node[align=center,fill=none] at (0.7,3.0) {\small Top-$k$};

    \draw[] node[align=center,fill=none] at (-1.25,0.75) {\small State\\ \small Variables};
    \draw[] node[align=center,fill=none] at (-3.25,0.75) {\small $+$ \small Derived\\ \small Variables};

    \draw[] node[align=center,fill=none] at (-1,-1.5) {\small Constant};
    \draw[] node[align=center,fill=none] at (-1.6,-3.0) {\small State-Dependent\phantom{0}};
\end{tikzpicture}
    \end{center}
    \caption[Overview of extension for classical planning (axioms).]{Overview of extensions for classical planning, where the red color denotes the planning formalism supported by the proposed symbolic search approach.}\label{fig:dervided_predicates:kiviat}
\end{figure}
\renewcommand{\kiviatPredicates}{1}

\section*{Core Publication of this Chapter}
\renewcommand{\citebf}[1]{\textbf{#1}}
\begin{itemize}
    \item \fullcite{speck-et-al-icaps2019}
\end{itemize}
\renewcommand{\citebf}[1]{#1}

In classical planning, Boolean or finite domain variables are used to describe the states of the world, the preconditions and the effects of actions, and the goal description \autocite{fikes-nilsson-aij1971,backstrom-nebel-compint1995}.
For many real-world problems, however, complex action preconditions or goals are desirable or even necessary for a compact problem description \autocite{thiebaux-et-al-aij2005}.
Axioms allow to model such complex preconditions and goals compactly by introducing a set of derived variables whose values are not directly influenced by the actions but are derived from the values of other variables using a set of logical axioms.

Although axioms are an essential feature \autocite{thiebaux-et-al-aij2005} of \pddl{} \autocite{mcdermott-et-al-tr1998,hoffmann-edelkamp-jair2005}, the common language for modeling planning tasks, modern planning techniques rarely support axioms, especially in cost-optimal planning.
Most admissible heuristics commonly used in \astar{} search, one of the most prominent approaches to cost-optimal planning, are not defined for their use with axioms.
The few heuristics that support axioms are based on naive relaxations that consider axioms as zero-cost actions, which may greatly reduce the informativeness of the heuristics.
One exception is the axiom-aware delete relaxation heuristic, obtained by applying a model for state constraints to planning with axioms \autocite{ivankovic-haslum-ijcai2015,haslum-et-al-jair2018}.
While these heuristics are often informative, they are also time-consuming to compute and therefore often do not pay off in terms of coverage or runtime. 

In this chapter, we define and motivate planning with axioms and summarize the known complexity and compilability results for this setting.
Then we describe three different sound and complete ways to extend symbolic search algorithms to support axioms natively introduced by \textcite{speck-et-al-icaps2019} (Figure \ref{fig:dervided_predicates:kiviat}).
The empirical study on different planning domains shows that the symbolic axiom encodings of \textcite{speck-et-al-icaps2019} yield an optimal planner that supports axioms and compares favorably with other state-of-the-art methods.

\section{Formalism}
Planning with axioms extends the formalism of classical planning (\Cref{def:planning-domain,def:planning-task}) with a set of derived variables and axioms as follows \autocite{thiebaux-et-al-aij2005,helmert-2008}.

\begin{definition}[Planning with Axioms]\label{def:planning_axioms}
    A \emph{planning task with axioms} is a tuple $\task = \langle \domain, \init, \goal, \limit \rangle$ with $\domain = \langle \vars, \operators, \constcostfun, \derivedvars, \axioms \rangle$ that extends an ordinary planning task and domain with a set of \emph{axioms} $\axioms$ that is used to evaluate a set of \emph{secondary} or \emph{derived} propositional {variables} $\derivedvars$. Partial variable assignments such as the preconditions $\pre_o$ of operators $o \in \operators$ and the goal condition are defined over primary and secondary variables $\vars \cup \derivedvars$.
    A variable assignment $s$ is called a \emph{state} if it is defined for all variables in $\vars$, and an \emph{extended state} if it is defined for all variables in $\vars \cup \derivedvars$.
    With $\states$ we refer to the set of all states and with $\extendedstates$ to the set of all extended states.
    $\axioms$ is a finite set of axioms over $\vars \cup \derivedvars$. An \emph{axiom} has the form $h \leftarrow b$ where the head $h$ is a variable in $\derivedvars$ and the \emph{body} $b$ is a finite conjunction of positive or negative literals over $\vars \cup \derivedvars$.

    The set of axioms is partitioning into a totally ordered set of axioms layers $\axioms_1 \prec \dots \prec \axioms_i \prec \dots \prec \axioms_k$ with $i=1,\dots,k$. The layer of an axiom is defined by the
    layer of its head, which is determined by a partition of the set of derived
    variables into subsets $\derivedvars_1 \prec \dots \prec \derivedvars_k$.
    We assume that this partition forms a \emph{stratification}, i.e., that for all $i=1, \dots, k$, and for each $d_i \in \derivedvars_i$, it holds that (1) if $d_j \in \derivedvars_j$ appears in the body of an axiom with head $d_i$, then $j \leq i$ and (2) if $d_j \in \derivedvars_j$ appears negated in the body of an axiom with head $d_i$, then $j < i$.
\end{definition}

The semantics of axioms is as follows: (1) to evaluate a derived variable, only axioms in the current or previous layers have to be considered, and (2) axioms have negation-as-fault semantics, i.e., if a fact cannot be derived as true, it is assumed to be false in subsequent layers \autocite{speck-et-al-icaps2019}.
Given a state $s \in \states$ (over $\vars$), the extended state $\axioms(s) \in \extendedstates$ (over $\vars \cup \derivedvars$) is uniquely defined by the standard stratified semantics \autocite{apt-et-al-1988,thiebaux-et-al-aij2005}.
In other words, axioms are evaluated in a layer-by-layer fashion using fixed point computations.
More precisely, given a state $s \in \states$, first all derived variables $d \in \derivedvars$ are set to their default value (false), i.e., $\lnot d$. 
Second, a fixed point computation is performed for each axiom layer in sequence to determine the final values of the derived variables \autocite{helmert-2008}.
The evaluated derived variables together with the state $s$ form the extended state $\axioms(s)$.

Intuitively, the axioms form a background theory that makes it possible to capture/derive some properties of states, i.e., secondary variables, from the primary state variables.
Consequently, an operator can only change the values of the primary variables ($\eff_o$ is a partial variable assignment over $\vars$), and based on the primary variables in the successor state, the values of the secondary variables can be derived using the axioms.
Moreover, an operator is applicable in a state $s \in \states$ if $\pre_o \subseteq \axioms(s)$ and a state $s \in \states$ is a goal state if $\goal \subseteq \axioms(s)$.
For a detailed description and discussion of the semantics of axioms, see \textcite{thiebaux-et-al-aij2005,helmert-2008,helmert-jair2006}.


\Cref{ex:axioms} illustrates the semantics and usefulness of axioms in the context of classical planning.

\begin{example}\label{ex:axioms}
    Consider the $\textit{navigate}$ operators of \Cref{ex:rover}, which navigates the rover from one cell to another with a cost of 0. This verbose modeling of the rover navigation leads to a larger state space and longer plan than necessary. Reachability can be expressed as a recursive property with axioms and derived properties.
    To model a $\textit{navigate}$ operator with axioms that moves the rover to a reachable location $loc$, we introduce a derived variable $\textit{reachable(loc)}$ as a precondition for the $\textit{navigate}$ operator. The values for the derived variables $\textit{reachable(loc)}$ are determined by a single layer of axioms $\axioms_1$.
    \begin{itemize}
        \item $\textit{reachable(loc)} \leftarrow \textit{rover-at}=loc$
        \item $\textit{reachable(to)} \leftarrow \textit{reachable(from)} \land \textit{free(to)} \land \textit{adjacent(from,to)}$
    \end{itemize}

    Intuitively, the first axiom means that the current location of the rover is reachable. The second axiom means that a free location adjacent to a reachable location is also reachable. By applying the axioms until a fixed point is reached, the current state $s$ is expanded to $\axioms(s)$, which then contains the information about the actual reachable locations ($reachable(loc)$) based on the current state $s$, which contains the rover location and all relevant information about the map.
    This encoding provides a natural and concise modeling of the transitive closure property of reachability. Moreover, it provides not only a smaller state space, but also a shorter plan $\pi^\rover{} = \langle$\navigate{7}{1}, \startDrone{7}{1}, \fly{7}{1}{6}{1}, \takeImg{6}{1}, \fly{6}{1}{10}{1}, \takeImg{10}{1}, \fly{10}{1}{7}{1}, \landDrone{7}{1}, \navigate{0}{5}$\rangle$, which avoids the irrelevant choice of the exact path the rover must take.\footnote{If the navigation operators of the rover would have non-zero costs, such a modeling would require state-dependent action costs, which we discuss in more detail in \Cref{ch:sdac}.}
\end{example}

In \Cref{ex:axioms}, the reachable cells for the rover are always the same, but this is not necessarily true in general. For example, it could be possible that actions by the rover make some cells impassable.  A similar scenario where reachability can change from state to state is Sokoban, where the locations of the boxes affect the reachability of the cells \autocite{ivankovic-haslum-ijcai2015,miura-fukunaga-aaai2017}.

\section{Complexity and Compilability}

\textcite{helmert-2008} showed that planning with axioms (\Cref{def:planning_axioms}) is \PSPACE-complete, since it generalizes classical planning (hardness) and a non-deterministic Turing machine can solve the bounded plan existence problem for planning with axioms with polynomial space (completeness).

\begin{theorem}[\cite{helmert-2008}]
    Bounded plan existence of planning with axioms is \PSPACE-complete. \qed
\end{theorem}

\textcite{thiebaux-et-al-aij2005} show that axioms are an essential feature of the planning language \pddl{} \autocite{mcdermott-et-al-tr1998,hoffmann-edelkamp-jair2005}, proving that it is in general not possible to compile away axioms without a super-polynomial growth of the plan length or description size.\footnote{Note that \textcite{thiebaux-et-al-aij2005} also show that the 1-step planning with axioms problem is \EXPTIME{}-complete, which results from the fact that they consider planning tasks described in \pddl{}, while we consider planning tasks and, in particular, axioms in a grounded and normalized form.}
Furthermore, \textcite{thiebaux-et-al-aij2005} argue that axioms are necessary to model real-world problems in a compact and elegant way, as they allow to
model complex action preconditions (e.g. transitive closure property) or goals.
From these two observations, we can conclude that it is desirable and necessary to have search algorithms that can handle axioms directly with native support to solve some real-world problems.

\section{Symbolic Search}

\textcite{speck-et-al-icaps2019} introduce three different variants to support and represent axioms in symbolic search. All three encodings are sound and complete, allowing for optimal and complete symbolic search. We briefly describe the idea of the three symbolic variants and present an empirical evaluation comparing the symbolic approaches \autocite{speck-et-al-icaps2019} with explicit heuristic search approaches for planning with axioms \autocite{ivankovic-haslum-ijcai2015}. 
For a more detailed explanation of the symbolic encodings, we refer to \textcite{speck-et-al-icaps2019}.

\begin{figure}
    \centering
    \begin{tikzpicture}[>=latex]

    \node[] at (1,0) (dots) {\dots};
    \node[draw] at (3,0) (states) {$\charf_{S}$};
    \node[draw, diamond, aspect=1,align=center] at (7.0,0) (fixpoint) {fixed\\point?
    };
    \node[draw,align=center] at (10.0,0) (fullstates) {$\charf_{\{\axioms{}(s) ~\mid~ s \in S\}}$};

    \draw[->, thick] (dots) -- (states) node[pos=0.1,above,align=center] {apply\\operations};
    \draw[->, thick] (states) -- (fixpoint) node[midway,above,align=center] {apply\\axioms $\operators_{\axioms_i}$};
    \draw[->, thick] (fixpoint) -- (fullstates)  node[midway,above,align=center,xshift=-2pt] {yes};
    \draw[->,thick] (fixpoint.north) |- (6, 1.3) -| (states);
    \node[] at (5.0, 1.5) {no};

    \draw [decorate,decoration={brace,amplitude=10pt},xshift=-4pt,yshift=0pt,rotate=90]
    (-1.25,-7.5) -- (-1.25,-2.75) node [black,midway,below,yshift=-10pt,align=center]
    {\footnotesize for each axiom layer \\ $\axioms_{1}, \dots, \axioms_{k}$};
    \draw[->,dotted,thick] (fullstates) -- (12.0,0);
\end{tikzpicture}
    \caption[\operators{}-based encoding of axioms.]{Visualization of symbolic \emph{forward} search with the
        \operators{}-based encoding of axioms \autocite{speck-et-al-icaps2019}.\label{fig:axiom_option1}}
\end{figure}


\paragraph{$\operators$-Based Encoding.}
The most straightforward way to integrate axioms into symbolic search is the so-called $\operators$-based encoding, where each axiom is interpreted as an operator, where the body is a (conditional) precondition and the head is the (conditional) effect.
Unlike the original operators, the axiom operators $\operators_{\axioms}$ affect a derived variable rather than a primary variable.
With this encoding, it is possible to expand a set of states, i.e., to evaluate the derived variables of each state in a set of states in a symbolic way with a fixed point computation that represents the operators $\operators_{\axioms}$ as a transition relation with BDDs.
\Cref{fig:axiom_option1} illustrates this procedure.
The general idea is to perform this fixed point computation after each application of the actual operators to expand all newly generated successor states.

\begin{figure}
    \centering
    \begin{tikzpicture}[>=latex]
    \node[] at (1,0) (dots) {\dots};
    \node[draw] at (3,0) (states) {$\charf_{S}$};
    \node[draw] at (10.0,0) (fullstates) {$\charf_{\{\axioms{}(s) ~\mid~ s \in S\}}$};

    \draw[->, thick] (dots) -- (states) node[pos=0.1,above,align=center] {apply\\operations};
    \draw[->, thick] (states) -- (fullstates) node[midway,above,align=center] {$\bigwedge_{d \in \derivedvars} (d \leftrightarrow \charf_{S_d})$ };
    \draw[->,dotted,thick] (fullstates) -- (12.0,0);
\end{tikzpicture}
    \caption[\vars{}-based encoding of axioms.]{Visualization of symbolic \emph{forward} search with the
        \vars{}-based encoding of axioms \autocite{speck-et-al-icaps2019}.\label{fig:axiom_option2}}
\end{figure}

\medskip
The following two axiom encodings are based on the idea of precomputing a symbolic representation over the primary variables $\vars$ for each derived variable $d \in \derivedvars$ if $d$ is true (\Cref{def:primary_representation}). 
This overcomes the problem of the $\operators$-based encoding, which requires an expensive fixed point computation to be performed in each planning step. 
\textcite{speck-et-al-icaps2019} present an efficient algorithm for computing the primary representation for each derived variable using BDDs.

\begin{definition}[Primary Representation]\label{def:primary_representation}
    Let $d \in \vars \cup \derivedvars$ be a (primary or derived) variable and $\axioms$ a set of axioms.
    The \emph{primary representation} of $d$ is the set of states $S_d$, which contains all states over $\vars$ where $d$ is evaluated to true, i.e., $S_d = \{ s \in \states ~|~ \axioms(s) \models d \}$.
\end{definition}

\paragraph{$\vars$-Based Encoding.}
\Cref{fig:axiom_option2} illustrates the $\vars$-based encoding, where each state of a given set of successor states $S$ is expanded.
Starting from the primary representation (\Cref{def:primary_representation}), the derived variable $d$ in a state $s \in S$ is true if $s \in S_d$. 
Therefore, $\charf_{\{\axioms{}(s) ~\mid~ s \in S\}} = \charf_{S} \land \bigwedge_{d\in \derivedvars} (d \leftrightarrow \charf_{S_d})$ represent the extended set of states.

\paragraph{Symbolic Translation.}
In contrast to the $\operators$-based and $\vars$-based encodings, the symbolic translation encoding completely omits derived variables by performing the search using a translation of the original planning task.
The idea of the translation is to replace all occurrences of derived variables in the planning task with their corresponding primary representation. In particular, derived variables in operator preconditions and the goal formula are replaced by their corresponding primary representation (\Cref{fig:axiom_option3}). Thus, no reasoning about derived variables during the actual search is required.
Note that this symbolic translation encoding is different from compilations that convert \pddl{} with axioms to \pddl{} without axioms \autocite{gazen-knoblock-ecp1997,thiebaux-et-al-aij2005}, since at no point is an explicit version of the compiled task (a new \pddl{} representation) created. 
While the primary representation, and thus the symbolic translation, can lead to exponential size growth in the worst case, in practice it is shown that the concise representation of BDDs alleviates this problem (see Empirical Evaluation).
\medskip

\begin{figure}[t]
    \centering
    \begin{tikzpicture}[>=latex]
	\node[align=center] at (0, 0.65) (tr_before) {
		transition relation\\$\charf_{T_o}$ over $\vars \cup \derivedvars$};
	\node[align=center] at (8, 0.65) (tr_after) {$\charf_{T'_o}$ with complex\\precondition over $\vars$};

	\node[align=center] at (0, -0.75) (goal_before) {goal condition\\$\charf_\goal$ over $\vars \cup \derivedvars$};
	\node[align=center] at (8, -0.75) (goal_after) {complex goal\\ condition $\charf_{\goal'}$ over $\vars$};

	\draw[->, thick] (tr_before) -- (tr_after) node[pos=0.5,above,align=center] {};
	\draw[->, thick] (goal_before) -- (goal_after) node[pos=0.5,above,align=center] {replace all $d \in \derivedvars$\\with $\charf_{S_{d}}$};



\end{tikzpicture}
    \caption[Symbolic translation encoding of axioms.]{
        Visualization of symbolic \emph{forward}, \emph{backward} and \emph{bidirectional} search with the symbolic translation encoding of axioms \autocite{speck-et-al-icaps2019}.\label{fig:axiom_option3}}
\end{figure}

\textcite{speck-et-al-icaps2019} define the $\operators$-based and the $\vars$-based encoding for symbolic forward search only. It is an open question how to perform backward search, since it is not clear how to efficiently reverse intermediate inferences about the values of derived variables.
In contrast, the symbolic compilation encoding performs all inferences over the derived variables as a preprocessing step. 
Thus, it allows to perform backward search and bidirectional search, which can significantly improve the planning performance.

\paragraph{Empirical Evaluation.}
\Cref{tab:coverage_axioms} shows the coverage (number of optimally solved instances) on different domains with axioms modeling verification problems \autocite{ghosh-et-al-jar2015,edelkamp-icaps2003wscompetition}, multi-agent planning with belief sets \autocite{kominis-geffner-icaps2015} or elevator control \autocite{koehler-schuster-aips2000}. 
Explicit \astar{} search \autocite{hart-et-al-ieeessc1968} is evaluated with the $\heu{blind}{}$ heuristic and the (naive) \heu{max}{} heuristic \autocite{ivankovic-haslum-ijcai2015}, which are the best known heuristics that are admissible and support axioms. 
We see that a BDD-based search is competitive with the $\operators$-based and $\vars$-based approaches compared to the explicit \astar{} search. 
This is mainly due to the good performance in the Miconic Axioms domain, which consists of many instances. 
However, symbolic search using the symbolic translation encoding proves to be the best performing and overall dominant approach, especially with bidirectional search.


\begin{table}
    \begin{center}
        \resizebox{1\textwidth}{!}{
            \begin{tabular}{lrrccrrr}
    \toprule
    \multirow{2}{*}{Algorithm}               &                                          &                                        & \multicolumn{5}{c}{\textbf{BDD Search}}               \\
    \cmidrule(lr){4-8}

                                            & \multicolumn{2}{c}{\astar{} (explicit)} & \multicolumn{1}{c}{\operators{}-based} &
    {\vars{}-based}                         & \multicolumn{3}{c}{Sym. Translation}                                                                                                      \\

    \cmidrule(lr){2-3} \cmidrule(lr){4-4} \cmidrule(lr){5-5}  \cmidrule(lr){6-8}
    Domain (\#Tasks)                        & \heu{blind}{}                            & \heu{max}{}                            & fwd                                     & fwd & fwd &
    bwd                                     & bid                                                                                                                                       \\
    \midrule
    & \multicolumn{7}{l}{\emph{International Planning Competition (IPC)}} \\
    {Blocks Axioms} (35)             & 18                                       & 18
                                            & 15                                       & 15
                                            & 21                                       & 18                                     &
    \textbf{30}                                                                                                                                                                         \\
    {Grid Axioms} (5)                & 1                                        & 2
                                            & 1                                        & 1
                                            & 1                                        & 0                                      & \textbf{3}                                            \\
    {Miconic Axioms} (150)           & 60                                       & 60
                                            & 127                                      & \textbf{150}
                                            & \textbf{150}                             & \textbf{150}
                                            & \textbf{150}                                                                                                                              \\
    {Optical Telegraphs} (48)        & 2                                        & 2
                                            & 2                                        & 2
                                            & \textbf{4}                               & 0
                                            & \textbf{4}                                                                                                                                \\
    {PSR Middle} (50)                & 35                                       & 35
                                            & 32                                       & 39
                                            & \textbf{50}                              & \textbf{50}
                                            & \textbf{50}                                                                                                                               \\
    {PSR Large} (50)                 & 14                                       & 14
                                            & 13                                       & 15
                                            & 24                                       & 23                                     &
    \textbf{25}                                                                                                                                                                         \\
    {Philosophers} (48)              & 5                                        & 5
                                            & 9                                        & 9
                                            & \textbf{12}                              & 4
                                            & \textbf{12}                                                                                                                               \\
    \midrule
    & \multicolumn{7}{l}{\emph{Complex Preconditions at IPC}} \\
    {Assembly} (30)                  & 0                                        & 0
                                            & 6                                        & 5
                                            & 9                                        & 8                                      &
    \textbf{11}                                                                                                                                                                         \\
    {Airport Adl} (50)               & 19                                       & \textbf{21}
                                            & 14                                       & 12
                                            & 20                                       & 11                                     & 19                                                    \\
    {Trucks} (30)                    & 6                                        & 8
                                            & \textbf{9}                               & \textbf{9}
                                            & \textbf{9}                               & 4
                                            & 8                                                                                                                                         \\
    \midrule
    & \multicolumn{7}{l}{\emph{\textcite{ivankovic-haslum-ijcai2015}}} \\
    {Blocker} (7)                    & \textbf{7}                               &
    \textbf{7}                              & 4                                        & 5
                                            & 5                                        & 5                                      & 5                                                     \\
    {Social Planning} (2)            & \textbf{2}                               &
    \textbf{2}                              & \textbf{2}                               & \textbf{2}
                                            & \textbf{2}                               & \textbf{2}                             & \textbf{2}                                            \\
    {Sokoban Axioms} (25)            & 19                                       & \textbf{20}
                                            & 7                                        & 7
                                            & 18                                       & \textbf{20}
                                            & \textbf{20}                                                                                                                               \\
    \midrule
    & \multicolumn{7}{l}{\emph{\textcite{ghosh-et-al-jar2015}}} \\
    Acc cc2 (7)                    & \textbf{7}                               & \textbf{7}
                                            & \textbf{7}                               &
    \textbf{7}
                                            & \textbf{7}                               & \textbf{7}
                                            & \textbf{7}                                                                                                                                \\
    Grid cc2 (13)                  & \textbf{8}                               & 7
                                            & 0                                        & 0
                                            & 0                                        & 0                                      & 0                                                     \\
    \midrule
    & \multicolumn{7}{l}{\emph{\textcite{kominis-geffner-icaps2015}}} \\
    Collab And Comm (1)            & \textbf{1}                               &
    \textbf{1}                              & 0                                        & 0
                                            & 0                                        & 0                                      & 0                                                     \\
    Muddy Children (1)             & \textbf{1}                               &
    \textbf{1}                              & \textbf{1}                               & \textbf{1}                             & \textbf{1}                              & 0
                                            & \textbf{1}                                                                                                                                \\
    Muddy Child (1)                & \textbf{1}                               &
    \textbf{1}                              & \textbf{1}                               &
    \textbf{1}                              & \textbf{1}                               &
    \textbf{1}                              & \textbf{1}                                                                                                                                \\
    Sum (1)                        & \textbf{1}                               &
    \textbf{1}                              & \textbf{1}                               &
    \textbf{1}                              & \textbf{1}                               & 0
                                            & \textbf{1}                                                                                                                                \\
    Word Rooms (2)                 & \textbf{2}                               &
    \textbf{2}                              & 0                                        & 0
                                            & 0                                        & 0                                      & 0                                                     \\
    \midrule
    w/o Miconic (406) & 149                                      & 154
                                            & 134                                      & 131
                                            & 185                                      & 153                                    &
    \textbf{199}                                                                                                                                                                        \\
    \textbf{Sum (556)}                      & 209                                      & 214
                                            & 251                                      & 281
                                            & 335                                      & 303                                    &
    \textbf{349}                                                                                                                                                                        \\
    \bottomrule
\end{tabular}

        }
    \end{center}
    \caption[Coverage for planning with axioms.]{Coverage (number of optimally solved instances) for explicit and symbolic search algorithms on domains with axioms \autocite{speck-et-al-icaps2019}.}
    \label{tab:coverage_axioms}
\end{table}

\chapter{State-Dependent Action Costs}\label{ch:sdac}
\chapterquote{The fashion of the world is to avoid cost, and you encounter it.}{William Shakespeare (1598)}

\renewcommand{\kiviatCost}{2}
\begin{figure}[t]
    \begin{center}
        \begin{tikzpicture}
    \tkzKiviatDiagram[
    radial style/.style ={-{Latex[length=3mm, width=2mm]}},
    scale=0.85,
    label space=1.5,
    radial = 1,
    gap = 1.5,
    step = 1,
    lattice = 2]{
    \hspace{2.5cm}\mbox{\parbox{4cm}{{\begin{center}\textbf{Goal}\\\textbf{Description} \end{center}}}},
    {\textbf{Number of Plans}},
    \hspace{-1.5cm}\mbox{\parbox{3cm}{{\begin{center}\textbf{State}\\\textbf{Description}\end{center}}}},
    \textbf{Cost Function}}

    \ifbool{kiviatRed}{
        \tkzKiviatLine[very thick,color=red!50,
            fill=red!30,
            opacity=.35,
            mark=ball,
            mark size=2.5pt,
            ball color=red](\kiviatGoal{},\kiviatTopk{},\kiviatPredicates{},\kiviatCost{})
    }

    \ifbool{kiviatBlue}{
        \tkzKiviatLine[very thick,color=blue!50,
            fill=blue!30,
            opacity=.35,
            mark=ball,
            mark size=2.5pt,
            ball color=blue](1,1,1,1)
    }

    \ifbool{kiviatFull}{
    \tkzKiviatLine[very thick,color=darkgreen!50,
        fill=green!30,
        opacity=.35,
        mark=ball,
        mark size=2.5pt,
        ball color=darkgreen](1,2,2,2)
    \tkzKiviatLine[very thick,color=blue!50,
        fill=blue!30,
        opacity=.35,
        mark=ball,
        mark size=2.5pt,
        ball color=blue](1,1,1,1)
    \tkzKiviatLine[very thick,color=red!50,
        mark=ball,
        mark size=2.5pt,
        ball color=red](\kiviatGoal{},\kiviatTopk{},\kiviatPredicates{},\kiviatCost{})
    \LegendBox[shift={(-0cm,-2.25cm)}]{current bounding box.south west}%
    {
    red!100/{Forward symbolic search (contribution of this thesis)},
    asparagus!100/{Bidirectional symbolic search (contribution of this thesis)},
    blue!100/{Bidirectional symbolic search (previous state of the art)}}
    }
    {}

    \draw[] node[align=center,fill=none] at (1.5,-0.5) {\small Hard};
    \draw[] node[align=center,fill=none] at (3.25,-0.75) {\small Oversub-\\ \small scription};

    \draw[] node[align=center,fill=none] at (0.7,1.5) {\small Top-$1$};
    \draw[] node[align=center,fill=none] at (0.7,3.0) {\small Top-$k$};

    \draw[] node[align=center,fill=none] at (-1.25,0.75) {\small State\\ \small Variables};
    \draw[] node[align=center,fill=none] at (-3.25,0.75) {\small $+$ \small Derived\\ \small Variables};

    \draw[] node[align=center,fill=none] at (-1,-1.5) {\small Constant};
    \draw[] node[align=center,fill=none] at (-1.6,-3.0) {\small State-Dependent\phantom{0}};
\end{tikzpicture}
    \end{center}
    \caption[Overview of extension for classical planning (cost function).]{Overview of extensions for classical planning, where the red color denotes the planning formalism supported by the proposed symbolic search approach.}\label{fig:cost_function:kiviat}
\end{figure}

\renewcommand{\kiviatCost}{1}

\section*{Core Publications of this Chapter}
\renewcommand{\citebf}[1]{\textbf{#1}}
\begin{itemize}
    \item \renewcommand{\citeaddendum}{\\\textbf{(Best Student Paper Runner-Up Award)}}\fullcite{speck-et-al-icaps2021}\renewcommand{\citeaddendum}{}
    \item \renewcommand{\citeaddendum}{\\\textbf{(Partly based on ideas from my master thesis)}}\fullcite{speck-et-al-icaps2018}\renewcommand{\citeaddendum}{}
\end{itemize}
\renewcommand{\citebf}[1]{#1}

In classical planning, it is common to assume constant action costs \autocite{fikes-nilsson-aij1971,backstrom-nebel-compint1995}.
This often results in increased effort for the modeler, because when a planning problem inherently involves state-dependent action costs (sometimes called conditional costs), the modeler of the problem must therefore distribute these costs over multiple copies of the original action.
In addition, the structure of the original cost function is hidden, which, however, could provide useful information and more compact representation possibilities for planning algorithms \autocite{geisser-phd2018}.
If we consider, e.g., probabilistic planning in form of factorized Markov decision processes \autocite{puterman-1994}, state-dependent action costs/rewards have been the standard for a long time \autocite{sanner-rddl2010,geisser-phd2018,geisser-et-al-socs2020} and are supported by many different approaches \autocite{keller-eyerich-icaps2012,geisser-speck-ippc2018,cui-khardon-ippc2018}.
Therefore, recently there has been increased interest in classical planning with state-dependent action costs \autocite{keller-geisser-aaai2015,keller-et-al-ijcai2016,geisser-phd2018,corraya-et-al-ki2019,mattmueller-et-al-aaai2018,ivankovic-et-al-icaps2019,haslum-et-al-jair2018,drexler-et-al-icaps2021,drexler-et-al-icaps2020wshsdip}.

In this chapter, we define and motivate planning with state-dependent action costs before summarizing the associated complexity and compilability investigated by \textcite{speck-et-al-icaps2021}.
Then, complete and optimal symbolic search algorithms based on BDDs and EVMDDs are presented and explained for planning with state-dependent action costs \autocite{speck-et-al-ipc2018,speck-et-al-icaps2018} (Figure \ref{fig:cost_function:kiviat}).
The empirical evaluation shows that the native support of state-dependent action costs within symbolic search can be beneficial overall compared to other explicit search approaches that are partially based on compilations.

\section{Formalism}
A planning domain and task with state-dependent action costs (sdac) is defined as follows \autocite{geisser-et-al-ijcai2015,geisser-phd2018}.

\begin{definition}[Planning with Sdac]\label{def:planning_sdac}
    An \emph{sdac planning task} $\langle \domain, \init, \goal, \limit \rangle$ with $\domain = \langle \vars, \operators, (\costfun_o)_{o \in \operators} \rangle$ is identical to a classical planning task (\Cref{def:planning-task}), except for having \emph{(local) state-dependent action cost functions} $\costfun_o : \states \to \mathbb{N}_0$ for each operator $o \in \operators$.
    The set of all operators $\operators$ induces a \emph{state-dependent action cost function} or \emph{cost function} $\costfun : \states \times \operators \to \mathbb{N}_0$  such that $\costfun(s, o) = \costfun_o(s)$ for all $s \in \states$.
    A plan $\pi = \langle o_0, \dots, o_{n-1} \rangle$ for an sdac planning task that generates a sequence of states $s_0, \dots, s_n$ generalizes a classical plan (\Cref{def:plan}) by considering for the cost computation the state in which each operator is applied, i.e., $\cost(\plan) = \sum_{i=0}^{n-1} \costfun_{o_i}(s_i)$.
\end{definition}

In general, such a state-dependent cost function can have an arbitrary form and even be uncomputable \autocite{geisser-phd2018}.
In practice, however, it is useful to restrict the form and expressiveness of the cost function.
Similar to \textcite{geisser-phd2018}, the focus of this work is mainly on cost functions that can be evaluated in polynomial time and are in concise form (\Cref{def:cost-function}).
However, in the theoretical complexity and compilability analysis, a comprehensive classification of compilability and non-compilability for planning with sdac is provided that also considers more complex classes of cost functions \autocite{speck-et-al-icaps2021}.

\begin{definition}[Operator Cost Function]\label{def:cost-function}
    We define a language $\mathcal{L}$ by the following Backus normal form:
    \[
        t ::= c \;\mybar\; v \;\mybar\; t + t \;\mybar\; t - t \;\mybar\; t \cdot t \;\mybar\; |t|,
    \]
    where $c \in \mathbb{Z}$, $v \in \vars$.
    The semantic of $\mathcal{L}$ is defined over states $s \in \states$ as follows:
    \begin{align*}
         & \costfun^c(s) = c                                                                                           \\
         & \costfun^v(s) = s(v)                                                                                        \\
         & \costfun^{t \circ t'}(s) = \costfun^t(s)  \circ \costfun^{t'}(s) \qquad \text{for } \circ \in \{+,-,\cdot\} \\
         & \costfun^{|t|}(s) = |\costfun^t(s)|
    \end{align*}

For a given term $t \in \mathcal{L}$, the interpretation $\costfun{}^t$ specifies the operator cost function, where we restrict the allowed operator cost function to a positive range, i.e., $\costfun{}^t : \states \to \mathbb{N}_0$. 
In the following, we often identify a cost function $\costfun{}^t$ with the term $t$ that defines it. 

\end{definition}

We emphasize that planning with sdac is a generalization of classical planning. Classical planning is an important special case in which we have a constant action cost, i.e., operator costs that are independent of the state in which the operator is applied.
Modeling operator costs as state-dependent allows for a more natural and concise representation of planning problems as \Cref{ex:sdac} illustrates.

\begin{example}\label{ex:sdac}
    Consider the drone flight operators of \Cref{ex:rover}.
    In classical planning (\Cref{def:planning-task}), we need one $\textit{fly}$ operator for each pair of cells $\textit{from}=(x,y)$ and $\textit{to}=(x',y')$ to model the Manhattan distance as constant costs.
    Note that such modeling is often used to represent motion operators with costs in the International Planning Competition (e.g., \name{transport} domain).
    With sdac, we are able to model such operators in a natural and concise way by specifying the Manhattan distance directly as the cost function of the $\textit{fly}$ operator.
    For each cell $\textit{to}=(x',y')$, we specify an operator $\textit{fly-to-}x'\textit{-}y'$ that is applicable when the drone is launched and has the effect of the drone being at $(x',y')$.
    The cost of the operator $\textit{fly-to-}x'\textit{-}y'$ is $|x_{\textit{drone}} - x'_{\textit{drone}}| + |y_{\textit{drone}} - y'_{\textit{drone}}|$.
\end{example}

\section{Complexity and Compilability}
The work of \textcite{speck-et-al-icaps2021} addresses the computational complexity and compilability of planning with sdac. They show that planning with sdac has the same complexity as ordinary planning, i.e., is \PSPACE-complete.
It is natural that a cost function should not be computationally more difficult than planning itself, i.e, more difficult than $\PSPACE$. Otherwise, the evaluation of the cost function would dominate the computation instead of solving the actual planning task, which should be the actual challenge.
By restricting the complexity of cost functions to \FPSPACE{}, we capture cost functions that can be evaluated in polynomial time (\Cref{def:cost-function}), up to more complex cost functions such as the cost of a robot motion operation \autocite{reif-focs1979,lavalle-2006}.

\begin{theorem}[\cite{speck-et-al-icaps2021}]\label{thm:plan_existence}
    Bounded plan existence of planning with sdac is $\PSPACE$-complete, if the cost function is in $\FPSPACE$. \qed
\end{theorem}

Recall that the same computational complexity of two planning formalisms does not imply the same expressive power \autocite{nebel-jair2000}.
While existing translations (EVMDD-based and combinatorial \autocite{geisser-et-al-ijcai2015,geisser-phd2018}) of sdac planning tasks into classical (constant cost) planning tasks work well in practice, both can lead to an exponential increase in task size, which means that strictly speaking they are not compilations by \Cref{def:compilation-scheme}.
\textcite{speck-et-al-icaps2021}, however, analyzed the compilability of sdac and obtained results of possibility and impossibility that depend on the desired plan length and the complexity of the cost function.
For this purpose, it is assumed that the cost function is provided in a concise form of a deterministic Turing machine (DTM).
It should be noted that, in general, other concise forms of cost functions, such as computer programs or the format defined in \Cref{def:cost-function}, can be simply converted to this form by specifying a DTM that computes/simulates the corresponding cost function.
Table \ref{tab:compilation_sdac} gives an overview of the compilation results, which we discuss below.

\begin{table}
    \begin{center}
    \resizebox{0.9\textwidth}{!}{
        \begin{tabular}{ll|p{2.5cm}|p{2.75cm}|p{2.75cm}|}
                                                                   &                                                      & \multicolumn{3}{c|}{Desired Plan Length Preservation}                                                                                                                                                                            \\[1mm]
                                                                   &                                                      & linearly                                                                                                                            & polynomially                                                                  & don't care \\
            \midrule
            \multirow{2}{*}{\rotatebox{90}{$\costfun$ Complexity}} & \raisebox{-3mm}{\rotatebox{90}{\vspace{-2cm} $\FP$}} & \multicolumn{1}{p{2.5cm}|}{\centering \emph{impossible}}                                                                            & \multicolumn{2}{p{5.5cm}|}{\centering \emph{possible} using DTM compilation}               \\
            \cmidrule{2-5}
                                                                   & \raisebox{-12mm}{\rotatebox{90}{$\FPSPACE$}}         & \multicolumn{2}{p{5.5cm}|}{\centering \emph{impossible} (unless polynomial hierarchy collapses at the third level)} & \multicolumn{1}{p{2.75cm}|}{\centering \emph{possible} using DTM compilation}              \\
            \midrule
        \end{tabular}
    }
\end{center}
\vspace{-4cm}
\begin{tikzpicture}
    \node at (0,0) {};
    \draw[line width=0.1cm,black] (5.325,1.1) |- (8.7,-0.345);
    \draw[line width=0.1cm,black] (5.325,-0.345) -| (8.7,-2.3);
\end{tikzpicture}
    \caption[Compilability results for planning with state-dependent action cost.]{Existence results for compilation schemes preserving task sizes polynomially depending on the computational complexity of the cost functions (rows) and the desired plan length preservation (columns) \autocite{speck-et-al-icaps2021}.}
    \label{tab:compilation_sdac}
\end{table}

\paragraph{Possibility Results.}  It can be shown that there exists a compilation scheme that preserves the plan length polynomially when the cost function is in $\FP$.
Moreover, there is a valid compilation scheme (unbounded plan length) if the cost function is in $\FPSPACE$.
The underlying idea of the compilation scheme that yields these two results is to simulate a DTM that computes the cost function $\costfun_o(s)$ of the operator $o \in \operators$ for a current state $s \in \states$ within the planning task.
For this purpose, each operator is compiled into an operator sequence that ensures that the task size and plan length are bounded by the space and time complexity of the corresponding DTM used to compute the cost function \autocite{speck-et-al-icaps2021}.
Note that this compilation makes it possible to encode various different cost functions that can be computed with a DTM, although the usefulness of the resulting planning task obviously depends heavily on the complexity of the DTM.

\paragraph{Impossibility Results.} Finally, \textcite{speck-et-al-icaps2021} proved that it is impossible to compile away, i.e., to find a compilation scheme that preserves plan length linearly when the cost function is in $\FP$. Moreover, it is impossible to preserve the plan length polynomially when the cost function is in $\FPSPACE$, unless the polynomial hierarchy collapses to the third level.

\medskip
Considering these results on complexity and compilability, one can see that it may be helpful not to compile sdac at all, but to keep it in the model and support it natively in algorithms, which avoids the overhead introduced by compilation \autocite{speck-et-al-icaps2021}.
In particular, the fact that cost functions in $\FP$, such as the fragment considered in practice (\Cref{def:cost-function}), require polynomial blowup in terms of domain description size and original plan length may make the use of the compilation technique infeasible.
In the following, we show how it is possible to natively support and represent sdac with decision diagrams to perform a symbolic search.

\section{Symbolic Search}
\textcite{speck-et-al-icaps2018} introduce a complete and optimal approach that performs symbolic search with EVMDDs.
In contrast to symbolic search with BDDs, where costs are represented by partitioning sets of states into subsets with identical cost values (see \Cref{ex:dds}), with EVMDDs it is possible to represent a set of reachable states together with reachability costs simultaneously.
The latter is possible by assigning to reachable states $S \subseteq \states$ the corresponding cost $g \in \mathbb{N}_0$ and to unreachable states the cost $\infty$.
\textcite{speck-et-al-icaps2018} introduced all the necessary operations such as the image and preimage operations for EVMDDs to enable symbolic search.
The main advantages of using EVBDDs/EVMDDs over ADDs and BDDs are that they can represent certain functions exponentially more compactly and that encoding reachability and cost in the open list, closed list, and transition relations simultaneously can also result in a more concise representation \autocite{speck-et-al-icaps2018}.

In addition to the EVMDD-based symbolic search for sdac planning tasks, we propose and study in this thesis a novel variant based on BDDs.
The underlying idea is to first create an ADD (instead of an EVMDD) that represents each operator as a transition relation that includes costs, before partitioning the ADD into multiple BDDs.
In other words, we create multiple transition relations for each operator with sdac, one for each possible cost value, encoding as a precondition the corresponding condition for the cost. This approach may result in multiple transition relations with different costs, but allows us to perform the classical symbolic search with cost bucketing, using the full power of sophisticated BDD operations and libraries.

Note that the representation of the cost function as decision diagram can lead to an exponential increase in the worst case. However, as usual with decision diagrams, it is assumed that the representation size of the cost functions is manageable.

\begin{figure}
    \subfloat[An EVMDD representing the transition relation of operator $o$ and the cost $\costfun_o(s)$ \autocite{speck-et-al-icaps2018}.\label{fig:tr_evmdd}]{
        \centering
        \makebox[0.425\textwidth][c]{
            \begin{tikzpicture}[%
        costnode/.style={pos=0.6,rectangle,thick,solid,
                inner sep=2pt,draw,fill=white,text=black,font=\small},%
        decisionnode/.style={circle,thick,minimum size=4mm,
                inner sep=2pt,draw,fill=white!80!black,text=black,font=\normalsize},%
        xscale=2,yscale=1.5,>=latex]
    \node[] (before) at (0,3.85) {}; 
    \node[decisionnode, minimum size=0.7cm] (x) at (0,3) {$x$};
    \node[decisionnode, minimum size=0.7cm] (xp) at (0,2.25) {$x'$};
    \node[decisionnode, minimum size=0.7cm] (y) at (0,1.5) {$y$};
    \node[decisionnode, minimum size=0.7cm] (yp0) at (-0.5,0.75) {$y'$};
    \node[decisionnode, minimum size=0.7cm] (yp1) at (0.5,0.75) {$y'$};
    \node[draw,thick,fill=white!80!black,rectangle, minimum size=0.7cm]
    (after0) at (0,0) {{$0$}};
    \draw[->, thick] (before) to[bend right=0,label distance=0mm,edge
    label={},swap,pos=0.3] node[costnode,pos=0.35] {$1$} (x);
    \draw[->, thick, dotted] (x) to[bend right=75,label distance=0mm,edge
    label={\small{$0$}},swap,pos=0.1] node[costnode,pos=0.5] {$0$} (xp);
    \draw[->, thick] (x) to[bend left=90,label distance=0mm,edge
    label={\small{$1$}},pos=0.15] node[costnode,pos=0.5] {$\infty$} (after0);
    \draw[->, thick, dotted] (xp) to[bend right=90,label distance=0mm,edge
    label={\small{$0$}},swap,pos=0.15] node[costnode,pos=0.5] {$\infty$} (after0);
    \draw[->, thick] (xp) to[bend left=75,label distance=0mm,edge
    label={\small{$1$}},pos=0.01] node[costnode,pos=0.5] {$0$} (y);
    \draw[->, thick, dotted] (y) to[bend right=30,label distance=0mm,edge
    label={\small{$0$}\phantom{.}},swap,pos=0.3] node[costnode,pos=0.35] {$0$} (yp0);
    \draw[->, thick] (y) to[bend left=30,label distance=0mm,edge
    label={\phantom{.}\small{$1$}},pos=0.3] node[costnode,pos=0.35] {$5$} (yp1);
    \draw[->, thick, dotted] (yp0) to[bend right=30,label distance=0mm,edge
    label={\small{$0$}},swap,pos=0.01] node[costnode,pos=0.35] {$0$} (after0);
    \draw[->, thick] (yp0) to[bend left=30,label distance=0mm,edge
    label={\small{$1$}},pos=0.01] node[costnode,pos=0.35] {$\infty$} (after0);
    \draw[->, thick, dotted] (yp1) to[bend right=30,label distance=0mm,edge
    label={\small{$0$}\phantom{.}},swap,pos=0.01] node[costnode,pos=0.35] {$\infty$} (after0);
    \draw[->, thick] (yp1) to[bend left=30,label distance=0mm,edge
    label={\phantom{.}\small{$1$}},pos=0.01] node[costnode,pos=0.35] {$0$} (after0);
\end{tikzpicture}
        }
    }\hfill
    \subfloat[Two BDDs representing the transition relation of operator $o$, where the left BDD encodes the cost of 1 and the right BDD encodes the cost of 6.\label{fig:tr_bdds}]{
        \centering
        \makebox[0.5\textwidth][c]{
            \begin{tikzpicture}[%
        costnode/.style={pos=0.6,rectangle,thick,solid,
                inner sep=2pt,draw,fill=white,text=black,font=\small},%
        decisionnode/.style={circle,thick,minimum size=4mm,
                inner sep=2pt,draw,fill=white!80!black,text=black,font=\normalsize},%
        xscale=2,yscale=1.5,>=latex]
    \begin{scope}
        \node[] (before) at (0,3.85) {}; 
        \node[decisionnode, minimum size=0.7cm] (x) at (0,3) {$x$};
        \node[decisionnode, minimum size=0.7cm] (xp) at (0.2,2.25) {$x'$};
        \node[decisionnode, minimum size=0.7cm] (y) at (0.4,1.5) {$y$};
        \node[decisionnode, minimum size=0.7cm] (yp) at (0.6,0.75) {$y'$};
        \node[draw,thick,fill=white!80!black,rectangle, minimum size=0.55cm]
        (after0) at (-0.25,0) {{$0$}};
        \node[draw,thick,fill=white!80!black,rectangle, minimum size=0.55cm]
        (after1) at (0.8,0) {{$1$}};
        \draw[->, thick] (before) to (x);
        \draw[->, thick, dotted] (x) to[bend right=0,label distance=0mm,edge
        label={\small{$0$}},pos=0.6] (xp);
        \draw[->, thick] (x) to[bend left=0,label distance=0mm,edge
        label={\small{$1$}},swap,pos=0.1] (after0);
        \draw[->, thick, dotted] (xp) to[bend right=0,label distance=0mm,edge
        label={\small{$0$}},swap,pos=0.1] (after0);
        \draw[->, thick] (xp) to[bend left=0,label distance=0mm,edge
        label={\small{$1$}},pos=0.6] (y);
        \draw[->, thick, dotted] (y) to[bend right=0,label distance=0mm,edge
        label={\small{$0$}},pos=0.6] (yp);
        \draw[->, thick] (y) to[bend left=0,label distance=0mm,edge
        label={\small{$1$}},swap,pos=0.1] (after0);
        \draw[->, thick, dotted] (yp) to[bend right=0,label distance=0mm,edge
        label={\small{$0$}},pos=0.6] (after1);
        \draw[->, thick] (yp) to[bend left=0,label distance=0mm,edge
        label={\small{$1$}},swap,pos=0.1] (after0);
    \end{scope}
    \begin{scope}[xshift=1.75cm]
        \node[] (before) at (0,3.85) {}; 
        \node[decisionnode, minimum size=0.7cm] (x) at (0,3) {$x$};
        \node[decisionnode, minimum size=0.7cm] (xp) at (0.2,2.25) {$x'$};
        \node[decisionnode, minimum size=0.7cm] (y) at (0.4,1.5) {$y$};
        \node[decisionnode, minimum size=0.7cm] (yp) at (0.6,0.75) {$y'$};
        \node[draw,thick,fill=white!80!black,rectangle, minimum size=0.55cm]
        (after0) at (-0.25,0) {{$0$}};
        \node[draw,thick,fill=white!80!black,rectangle, minimum size=0.55cm]
        (after1) at (0.8,0) {{$1$}};
        \draw[->, thick] (before) to (x);
        \draw[->, thick, dotted] (x) to[bend right=0,label distance=0mm,edge
        label={\small{$0$}},pos=0.6] (xp);
        \draw[->, thick] (x) to[bend left=0,label distance=0mm,edge
        label={\small{$1$}},swap,pos=0.1] (after0);
        \draw[->, thick, dotted] (xp) to[bend right=0,label distance=0mm,edge
        label={\small{$0$}},swap,pos=0.1] (after0);
        \draw[->, thick] (xp) to[bend left=0,label distance=0mm,edge
        label={\small{$1$}},pos=0.6] (y);
        \draw[->, thick, dotted] (y) to[bend right=0,label distance=0mm,edge
        label={\small{$0$}},swap,pos=0.1] (after0);
        \draw[->, thick] (y) to[bend left=0,label distance=0mm,edge
        label={\small{$1$}},pos=0.6] (yp);
        \draw[->, thick, dotted] (yp) to[bend right=0,label distance=0mm,edge
        label={\small{$0$}},swap,pos=0.1] (after0);
        \draw[->, thick] (yp) to[bend left=0,label distance=0mm,edge
        label={\small{$1$}},pos=0.6] (after1);
    \end{scope}
\end{tikzpicture}
        }
    }
    \caption[Visualization of symbolic representations of an operator with state-dependent action costs.]{Visualization of different variants for a symbolic representation of an operator $o = \langle \lnot x, x \rangle$ with state-dependent action costs $\costfun_o(s) = 5y +1$.}
    \label{fig:tr_sdac}
\end{figure}
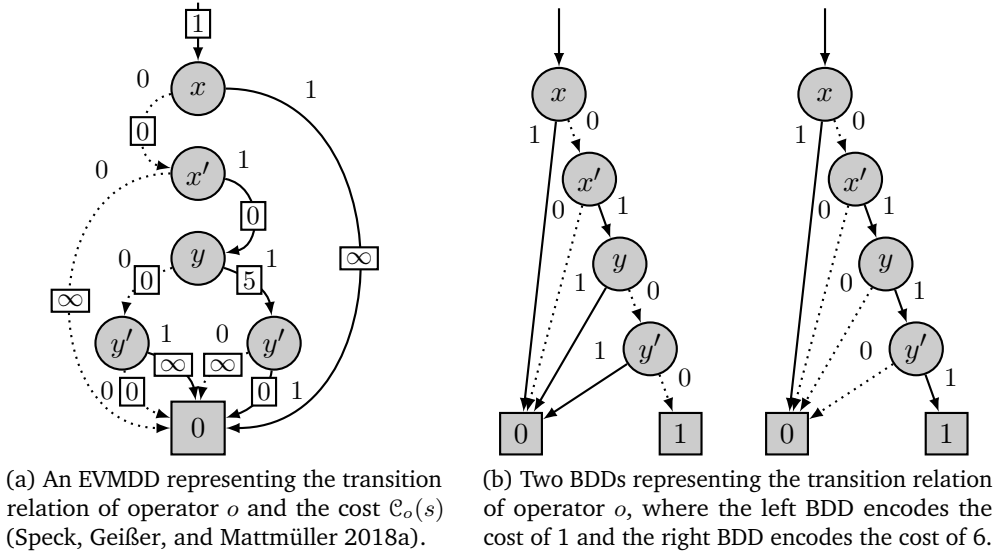

\begin{example}
    Consider an operator $o = \langle \lnot x, x \rangle$ with sdac $\costfun_o(s) = 5y +1$ for $s \in \states$. \Cref{fig:tr_sdac} visualizes the transition relation of $o$ represented with different decision diagrams, where the precondition (predecessor states) are encoded with unprimed variables and the effects (successor states) with primed variables. \Cref{fig:tr_evmdd} represents the EVMDD representing the transition relation of $o$, while encoding the cost $\costfun_o(s)$ for each state $s \in \states$ at the same time. To represent the operator $o$ with BDDs, multiple BDDs are necessary to decompose the cost function.
    \Cref{fig:tr_bdds} illustrates two BDDs representing the two possible cost values of $\costfun_o(s)$, specifically 1 and 6. The left BDD encodes the cost of 1 by encoding $ \lnot y$ in the precondition.  Note that by adding the new precondition $\lnot y$, $\lnot y$ also holds in all successor states, i.e., $\lnot y'$, since variable $y$ does not occur in the effect (closed world assumption). Similarly, the right BDD encodes $y$ as a precondition to obtain a cost of 6.
\end{example}

\paragraph{Empirical Evaluation.}  While \textcite{speck-et-al-icaps2018}  reported results comparing symbolic search with EVMDDs and translation-based approaches, here we present new empirical results incorporating recent advances, planners, and benchmarks for planning with sdac \autocite{geisser-phd2018,corraya-et-al-ki2019,speck-et-al-icaps2021}.

\Cref{tab:coverage_sdac} shows the coverage (number of optimally solved instances) of explicit \astar{} search \autocite{hart-et-al-ieeessc1968} with various translation-based heuristics \autocite{geisser-phd2018} and forward, backward and bidirectional symbolic search with all data structures represented as either EVMDDs \autocite{speck-et-al-icaps2018} or BDDs.
In explicit \astar{} search for planning with sdac, the cost function is represented as an EVMDD, which is used to evaluate the actual costs ($g$-values) \autocite{geisser-phd2018}.
We consider the \heu{blind}{} heuristic and the \heu{max}{} heuristic \autocite{bonet-geffner-ecp1999}, which is analyzed in two versions: \heu{max}{c} uses the combinatorial translation \autocite{geisser-phd2018} and \heu{max}{dd} uses the EVMDD translation \autocite{geisser-et-al-ijcai2015} to compute the heuristic values.

\begin{table}
    \begin{center}
        \resizebox{1\textwidth}{!}{
            \newcommand{\numtasks}[1]{(#1)}
\begin{tabular}{@{}lrrrrrrrrr@{}}
    \toprule
    Algorithm & \multicolumn{3}{c}{\astar{} (explicit)} & \multicolumn{3}{c}{\textbf{EVMDD Search}} & \multicolumn{3}{c}{\textbf{BDD Search}} \\
    \cmidrule(lr){2-4} \cmidrule(lr){5-7} \cmidrule(lr){8-10}
    Domain (\# Tasks) & \heu{blind}{} & \heu{max}{c} & \heu{max}{dd} & fwd & bwd & bid & fwd & bwd & bid \\
    \midrule
    Asterix \numtasks{30} & 10 & 5 & 10 & \textbf{30} & 29 & \textbf{30} & \textbf{30} & 28 & \textbf{30} \\
    
    Pegsol \numtasks{100} & \textbf{96} & 0 & \textbf{96} & 86 & 16 & 88 & \textbf{96} & 23 & \textbf{96} \\

    Gripper \numtasks{30} & 8 & 4 & 8 & 11 & 8 & 12 & \textbf{20} & 14 & 19 \\

    PSR Large \numtasks{100} & 33 & 12 & \textbf{39} & 36 & 20 & 36 & 28 & 23 & 28 \\
    PSR Middle \numtasks{100} & 79 &  34 & 84 & 88 & 55 & \textbf{90} & 74 & 58 & 74 \\

    Openstacks \numtasks{70} & 9 & 10 & 10 & 40 & 30 &  38 & \textbf{54} & 46 & 51 \\

    Transport \numtasks{30} & 24 & 19 & 24 & 24 & 23 & \textbf{29} & 24 & 23 & 28 \\
    TSP \numtasks{26} & \textbf{21} & 12 & \textbf{21} & 15 & 10 & 13 & 18 & 10 & 18 \\
    \midrule
    \textbf{Sum \numtasks{486}} & 280 & 96 & 292 & 330 & 191 & 336 & \textbf{344} & 225 & \textbf{344} \\
    \bottomrule
    \end{tabular}
    
        }
    \end{center}
    \caption[Coverage for planning with state-dependent action costs.]{Coverage (number of optimally solved instances) for explicit and symbolic search algorithms on versions of planning domains with state-dependent action costs.}
    \label{tab:coverage_sdac}
\end{table}

\bigskip
Overall, it can be seen that native support for sdac within symbolic search compares favorably with explicit search approaches.
Comparing symbolic search with BDDs and EVMDDs, there are domain-wise differences in coverage, with BDD-based symbolic search performing best overall.
However, the structural advantages of EVMDDs pay off especially in the PSR domain, where the cost function is complex and contains derived variables represented with the symbolic translation described in \Cref{ch:axioms}.\footnote{In \Cref{sec:combination}, the relationship and interplay of the different planning extensions is discussed in more detail.} Looking at the explicit approaches, we can see that in many domains it is not possible to perform a combinatorial translation \heu{max}{c}, which often requires exponentially many operators and thus solves the least number of tasks.
Explicit \astar{} search with the $\heu{blind}{}$ and the $\heu{max}{dd}$ heuristics performs much better.
These configurations use EVMDDs to concisely represent the cost functions and as a basis for the heuristic computation, which seems to pay off.
It turns out that the heuristic estimates of $\heu{max}{dd}$ can help the search, but the performance of $\heu{max}{dd}$ is still often not comparable to symbolic blind search.
However, there are domains where explicit heuristic search performs better than symbolic search. This shows that, similar to classical planning, a potential portfolio planner of these complementary search strategies could lead to a state-of-the-art planner that combines both strengths \autocite{sievers-et-al-aaai2019}.



\chapter{Oversubscription Planning}\label{ch:osp}
\chapterquote{A goal is not always meant to be reached, it often serves simply as something to aim at.}{Bruce Lee}

\renewcommand{\kiviatGoal}{2}
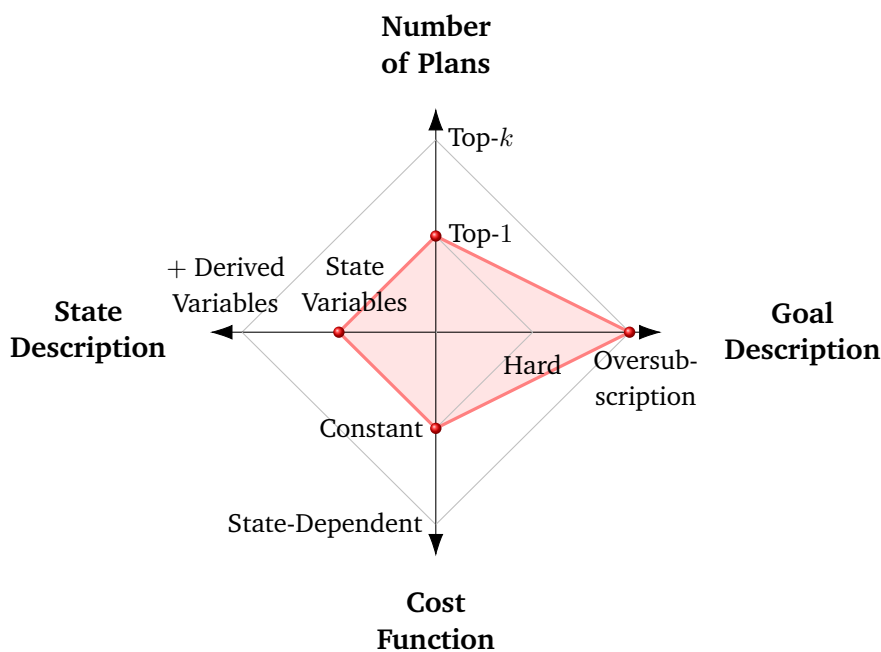
\begin{figure}[t]
    \begin{center}
        \begin{tikzpicture}
    \tkzKiviatDiagram[
    radial style/.style ={-{Latex[length=3mm, width=2mm]}},
    scale=0.85,
    label space=1.5,
    radial = 1,
    gap = 1.5,
    step = 1,
    lattice = 2]{
    \hspace{2.5cm}\mbox{\parbox{4cm}{{\begin{center}\textbf{Goal}\\\textbf{Description} \end{center}}}},
    {\textbf{Number of Plans}},
    \hspace{-1.5cm}\mbox{\parbox{3cm}{{\begin{center}\textbf{State}\\\textbf{Description}\end{center}}}},
    \textbf{Cost Function}}

    \ifbool{kiviatRed}{
        \tkzKiviatLine[very thick,color=red!50,
            fill=red!30,
            opacity=.35,
            mark=ball,
            mark size=2.5pt,
            ball color=red](\kiviatGoal{},\kiviatTopk{},\kiviatPredicates{},\kiviatCost{})
    }

    \ifbool{kiviatBlue}{
        \tkzKiviatLine[very thick,color=blue!50,
            fill=blue!30,
            opacity=.35,
            mark=ball,
            mark size=2.5pt,
            ball color=blue](1,1,1,1)
    }

    \ifbool{kiviatFull}{
    \tkzKiviatLine[very thick,color=darkgreen!50,
        fill=green!30,
        opacity=.35,
        mark=ball,
        mark size=2.5pt,
        ball color=darkgreen](1,2,2,2)
    \tkzKiviatLine[very thick,color=blue!50,
        fill=blue!30,
        opacity=.35,
        mark=ball,
        mark size=2.5pt,
        ball color=blue](1,1,1,1)
    \tkzKiviatLine[very thick,color=red!50,
        mark=ball,
        mark size=2.5pt,
        ball color=red](\kiviatGoal{},\kiviatTopk{},\kiviatPredicates{},\kiviatCost{})
    \LegendBox[shift={(-0cm,-2.25cm)}]{current bounding box.south west}%
    {
    red!100/{Forward symbolic search (contribution of this thesis)},
    asparagus!100/{Bidirectional symbolic search (contribution of this thesis)},
    blue!100/{Bidirectional symbolic search (previous state of the art)}}
    }
    {}

    \draw[] node[align=center,fill=none] at (1.5,-0.5) {\small Hard};
    \draw[] node[align=center,fill=none] at (3.25,-0.75) {\small Oversub-\\ \small scription};

    \draw[] node[align=center,fill=none] at (0.7,1.5) {\small Top-$1$};
    \draw[] node[align=center,fill=none] at (0.7,3.0) {\small Top-$k$};

    \draw[] node[align=center,fill=none] at (-1.25,0.75) {\small State\\ \small Variables};
    \draw[] node[align=center,fill=none] at (-3.25,0.75) {\small $+$ \small Derived\\ \small Variables};

    \draw[] node[align=center,fill=none] at (-1,-1.5) {\small Constant};
    \draw[] node[align=center,fill=none] at (-1.6,-3.0) {\small State-Dependent\phantom{0}};
\end{tikzpicture}
    \end{center}
    \caption[Overview of extension for classical planning (goal description).]{Overview of extensions for classical planning, where the red color denotes the planning formalism supported by the proposed symbolic search approach.}
    \label{fig:goal_description:kiviat}
\end{figure}
\renewcommand{\kiviatGoal}{1}

\section*{Core Publication of this Chapter}
\renewcommand{\citebf}[1]{\textbf{#1}}
\begin{itemize}
    \item \fullcite{speck-katz-aaai2021}
\end{itemize}
\renewcommand{\citebf}[1]{#1}

In conventional classical planning, goal states are defined by a goal condition, which is treated as a hard constraint that must be fully satisfied by any solution/plan.
In many real-world applications, however, the description of goals is oversubscribed, i.e., there are a large number of desirable, often competing goals of varying value, and a system (e.g., a Mars rover) may
not be able to achieve all of them with the available resources (e.g., battery power).
In such scenarios, it is natural to assign different utility values to states and to search for feasible and achievable states that maximize the overall utility value.
This planning formalism is called \emph{partial satisfaction planning}. Here, we distinguish between \emph{net-benefit planning} \autocite{vandenbriel-et-al-aaai2004}, where operator costs and state utility values, i.e., solution costs and solution utilities, are comparable, and \emph{oversubscription planning} \autocite{smith-icaps2004}, where those are not comparable. 

Symbolic search in the context of partial satisfaction planning has been studied to a limited extent.
For net-benefit planning, symbolic branch-and-bound search was introduced \autocite{edelkamp-kissmann-ijcai2009,kissmann-phd2012}, which considers so-called soft-goal planning tasks, i.e., planning tasks in which the utility function is a weighted sum over a subset of the state variables. 
Since \textcite{keyder-geffner-jair2009} provided a simple and concise compilation from net-benefit planning (soft-goal planning tasks) to classical planning, which is most commonly used in practice, we will focus on oversubscription planning, which is considered more challenging \autocite{domshlak-mirkis-jair2015}.
For oversubscription planning, there are a variety of explicit (heuristic) search approaches, mainly based on adapting existing heuristics from classical planning to oversubscription planning \autocite{mirkis-domshlak-icaps2013,domshlak-mirkis-jair2015,muller-karpas-icaps2018,katz-et-al-icaps2019,katz-et-al-icaps2019,garcia-olaya-et-al-aij2021}.
An important example of symbolic search in the context of oversubscription planning is the work of \textcite{eifler-et-al-ijcai2020}, in which plan explanations were investigated considering plan utilities.
The work of \textcite{speck-katz-aaai2021} is the first to explicitly define and study in detail symbolic search for solving oversubscription planning problems.

In this chapter, we define and motivate oversubscription planning that allows to extend the goal specification (\Cref{fig:goal_description:kiviat}), and discuss and summarize complexity and compilability results for this setting.
Then, an optimal and complete symbolic search approach for oversubscription planning introduced by \textcite{speck-katz-aaai2021} is presented.
The empirical analysis shows that the presented symbolic approach competes favorably with explicit heuristic state-space search.

\section{Formalism}

We consider oversubscription planning tasks \autocite{smith-icaps2004,katz-et-al-icaps2019}, where the solution cost and utility are \emph{not comparable}.

\begin{definition}[Oversubscription Planning]\label{def:planning_osp}
    An \emph{oversubscription planning (osp) task} is a tuple $\task = \langle \domain, \init, \goal, \limit \rangle$ with $\domain = \langle \vars, \operators, \constcostfun, \utility\rangle$ that extends an ordinary planning task and domain by a state utility function $\utility : \states \to \mathbb{N}_0$ that assigns utility to states.
\end{definition}

Similar to planning with sdac, where we focused mainly on cost functions that can be evaluated in polynomial time, here we focus mainly on utility functions that can be evaluated in polynomial time.
We assume that the utility function $\utility$ has the same form as an operator cost function (\Cref{def:cost-function}).

Since we consider utility in osp, it is natural to search for a plan that maximizes utility, i.e., leads to a state with high utility within the cost bound $\limit$. For this purpose, we define an osp plan as follows and introduce the concept of utility-optimal plans.

\begin{definition}[Osp Plan]\label{def:osp_plan}
    An \emph{(osp) plan} $\pi$ is defined as an ordinary classical plan (\Cref{def:plan}), i.e., an applicable sequence of operators that leads from an initial state $\init$ to a goal state $s_\star \in \goalstates$, where the plan cost (cumulative operator cost) is within the cost limit $\cost{}(\pi) \leq \limit$.

    The \emph{utility} $\utility(\pi)$ of the plan $\pi$ is defined by the utility of the final state $s_\star$ induced by $\pi$, i.e., $\utility(\pi) = \utility(s_\star)$.

    Such a plan $\pi$ is called \emph{utility-optimal}, or simply \emph{optimal} (if it is clear from the context), if there is no other plan $\pi'$ for which $\utility(\pi') > \utility(\pi)$.
    A plan $\pi$ is \emph{cheapest utility-optimal} if it is a cheapest among utility-optimal plans, i.e., there is no utility-optimal plan $\pi'$ for which $\cost(\pi') < \cost(\pi)$.
\end{definition}

The underlying idea of osp is to model many different possible competing goals with different associated utilities.
The overall goal is to achieve as much as possible given a cost $\limit$ that models a limited resource together with operator costs.
Therefore, in practice, the ``hard'' goals $\goal$ are often underspecified or even nonexistent, as the goals can be modeled with utility \autocite{katz-et-al-zenodo2019,katz-et-al-icaps2019}.
Since all the approaches and results presented in this chapter can straightforwardly support hard goals $\goal$, we consider them for completeness.
As a result of this modeling, there are often many valid plans that vary strongly in their utility. 
Thus, the main goal and challenge of oversubscription planning is to identify a plan with high utility.
The search for an utility-optimal plan, i.e., a plan with the highest utility, is called \emph{optimal oversubscription planning} and is the main topic of this chapter.
Originally, osp was introduced to address NASA missions such as planning science experiments for a Mars rover \autocite{smith-icaps2004}, where there are often limited resources and several different and competing goals.

\begin{example}\label{ex:osp}
    Recall \Cref{ex:rover}, where we specified the goal as taking images at (6,1) and (10,1) and traveling equipped with the drone to \say{Three Forks} at (0,5).
    It is reasonable to assume that the trip to (0,5) with the drone is the (hard) goal $\goal$, since the journey of the rover is expected to continue from there.
    Let us also assume that the drone has a limited battery capacity, as in reality.
    With the operator cost $\constcostfun$ we describe the battery consumption of the drone for the respective actions, and the final plan is constrained by the battery capacity of say $20$, which is represented as the cost limit $\limit = 20$.
    Recall that in our model, only the operators associated with the drone have non-zero costs.

    With osp, we can now express that we are more interested in the location (10,1) by assigning a utility of $10$ to the states in which an image was taken at (6,1) and a utility of $15$ to the states in which an image was taken at (10,1).
    If we consider the original plan $\pi^{\rover{}}$ with a cost of $\cost{}(\pi^\rover{}) = 24$ by taking images at (6,1) and (10,1), with a utility of $10+15=25$, we see that it is not feasible because it is too expensive at a cost of $\cost{}(\pi^\rover{}) = 24 \not\leq 20 = \limit$.

    Due to the cost bound, in a utility-optimal osp plan, only the objective of taking an image at (10,1) can be satisfied, ignoring location (6,1). 
    The final utility-optimal osp plan is as follows $\pi^\rover{} = \langle$\navigate{7}{2}, \navigate{7}{1}, \startDrone{7}{1}, \fly{7}{1}{10}{1}, \takeImg{10}{1}, \fly{10}{1}{7}{1}, \landDrone{7}{1}, \navigate{7}{2}, \dots, \navigate{0}{5}$\rangle$ with a cost of $\cost(\pi^\rover{}) = 18$ and a utility of $\utility(\pi^\rover{}) = 15$. If the cost limit were even lower, the plan would change to taking an image at (6,1) or even not taking any images at all.
\end{example}

\section{Complexity and Compilability}

In net-benefit planning, extensive complexity analyses have been proposed, showing, among other things, that it is \PSPACE{}-complete \autocite{aghighi-jonsson-aaai2014,aghighi-backstrom-ijcai2015}.
Moreover, for net-benefit planning, where it is common to define the state utility function as a weighted sum over a subset of the state variables, it was shown that these have no expressive power and can be easily compiled away \autocite{keyder-geffner-jair2009}.

However, in osp, where solution cost and solution utility are \emph{not comparable}, there is little work on complexity and compilability \autocite{katz-mirkis-ijcai2016}.
\textcite{katz-mirkis-ijcai2016} state that optimal oversubscription planning is \PSPACE{}-complete.
There is no explicit proof of this statement in the literature, so we provide such a proof in this thesis. 
For this purpose, we define the bounded utility plan existence problem as follows.

\begin{definition}[Bounded Utility Plan Existence]\label{def:osp_plan_existence}
    The \emph{bounded utility plan existence} problem is the decision problem of determining whether there exists a plan with utility greater than or equal to $u \in \mathbb{N}_0$ for a given osp task $\task$.
\end{definition}

Similar to planning with state-dependent action costs (\Cref{def:planning_sdac}), we assume a utility function that is computationally no more difficult than planning itself, i.e., lies in $\FPSPACE$.

\begin{theorem}\label{thm:osp_space}
    Bounded utility plan existence is \PSPACE{}-complete, if the utility function is in $\FPSPACE$.
\end{theorem}

\begin{proof}
    \PSPACE{}-hardness results from reducing the well-known bounded plan existence (with unit cost) problem (\Cref{thm:plan_existence}; \cite{bylander-aij1994}) to our problem, where the plan length is equal to the plan cost. \PSPACE{} membership can be proved by defining a nondeterministic Turing machine (NTM) that starts with the initial state and guesses an operator to apply at each step.
    In addition, the NTM evaluates the utility of the current state $s$ in each step. The NTM ends with ``Yes'' if $s \in \goalstates$ and $\utility{}(s) \geq u$ and with ``No'' if the selected operator is not applicable or the cost bound is exceeded.
    Since at any point in time only the current state, the summed costs, and the computation of the utility function (in \FPSPACE{}) need to be maintained, this NTM is in \NPSPACE{}, which is known to be equivalent to \PSPACE{} \autocite{savitch-jcss1970}.
\end{proof}

Note that this proof holds even if the osp formalism does not contain a goal specification $\goal$, as is often the case in the literature. The \PSPACE{}-hardness remains because we can simulate $\goal$ by assigning non-zero utility only to the goal states $\goalstates$, while requiring a plan with utility greater than $0$, i.e., $u>0$.

Similar to all other extensions, we can observe that the computational complexity is the same for classical and oversubscription planning.
The compilability of osp to classical planning is an open research question.
Whether and to what extent osp is more expressive than classical planning is not known yet.
However, there are some results and arguments in the literature that suggest that osp is (significantly) more expressive than classical planning.

\textcite{katz-mirkis-ijcai2016} investigated the computational complexity of osp given common classes of utility functions.
They showed that the complexity of several of these fragments is computationally more difficult for osp than for net-benefit planning, where operator costs and state utilities are comparable, indicating that osp is computationally more challenging.
\textcite{domshlak-mirkis-jair2015} argue that unlike classical planning or net-benefit planning, osp does not seem to be reducible to a shortest path problem with a single source and destination.
For this reason, solving osp tasks requires (1) exhaustive search algorithms and (2) sophisticated heuristics used in classical planning cannot be applied directly.
Based on these observations, \textcite{speck-katz-aaai2021} proposed to use symbolic search for osp because symbolic search is known to be an efficient search strategy that (1) can exhaustively explore the state space and (2) does not rely on heuristics, i.e., performs a blind search.

\begin{figure}
    \begin{center}
        \resizebox{1\textwidth}{!}{
            \begin{tikzpicture}[yscale=1]
    \begin{scope}
        \node[align=center,draw] (s0) at (0,0) {$\init = \lnot x \land \lnot y$\\$\utility(\init) = 0$};
        \node[align=center,draw] (s1) at (5,0.75) {$s_1 = \lnot x \land y$\\$\utility(s_1) = 0$};
        \node[align=center,draw] (s2) at (5,-0.75) {$s_2 = x \land \lnot y$\\$\utility(s_2) = 2$};
        \node[align=center,draw] (s3) at (10,0) {$s_3 = x \land y$\\$\utility(s_2) = 3$};

        \draw[dotted] (-1.5,1.5) rectangle (1.5,-2.2) node[pos=.5,align=center] {\\\\\\\\\\\\\\$g = 0$};
        \draw[dotted] (3.5,1.5) rectangle (6.5,-2.2) node[pos=.5,align=center] {\\\\\\\\\\\\\\$g = 1$};
        \draw[dotted] (8.5,1.5) rectangle (11.5,-2.2) node[pos=.5,align=center] {\\\\\\\\\\\\\\$g = 2$};

        \draw[dashed] (7.5,0.95) to (7.5,0.2);
        \draw[dashed] (7.5,-0.6) to (7.5,-2.0);
        \node[align=center] at (7.5,1.3) {$\limit = 1$};

        \draw[->, thick] (s0) to node [above, align=center] {$o_1$} (s1);
        \draw[->, thick] (s0) to node [above, align=center] {$o_2$} (s2);
        \draw[->, thick] (s2) to node [above, align=center] {$o_3$} (s3);
    \end{scope}
\end{tikzpicture}
        }
    \end{center}
    \caption[Visualization of an oversubscription planning task.]{Visualization of the transition system induced by the oversubscription planning tasks $\task$ in \Cref{ex:osp-algorithm} \autocite{speck-katz-aaai2021}.}
    \label{fig:osp-example}
\end{figure}

\section{Symbolic Search}
\textcite{speck-katz-aaai2021} define symbolic \emph{forward} search for osp in a straightforward way and showed that their approach is complete and optimal.
In fact, the presented symbolic search approach finds not only a utility-optimal plan but also a cheapest utility-optimal plan.

The underlying idea is to perform an exhaustive symbolic forward search, just as in classical planning, until the cost bound $\limit$ is exceeded.
The utility function $\utility$ is a numerical function that can be represented with one ADD or several BDDs, as explained in \Cref{ex:dds}.
In each expansion step, the utility of the expanded states is symbolically evaluated and the states with the highest utility so far are maintained.
In addition, representing the utility as ADDs (or BDDs) has the advantage that the highest possible value of the utility function is known directly, since it is simply the terminal node with the highest value.
Finally, we stop when either all reachable states have been expanded within the cost bound or a state with the overall maximum utility has been found, since there can be no better plan.
\Cref{ex:osp-algorithm} illustrates how the symbolic approach of \textcite{speck-katz-aaai2021} works.

\begin{example}[\cite{speck-katz-aaai2021}]\label{ex:osp-algorithm}
    Consider a unit-cost oversubscription planning task $\task = \langle \vars, \operators, \constcostfun, \utility, \init, \goal, \limit \rangle$ with an empty goal $\goal = \{\}$ such that all states are goal states, i.e., $\states = \goalstates$.
    Moreover, $\task$ has two binary variables $\vars = \{x,y\}$, where $\vardomain_x = \vardomain_y = \{0,1\}$, an initial state $\init(x) = \init(y) = 0$, a utility function $\utility = 2x + xy$, and a cost bound $\limit = 1$.
    There exist three operators $O = \{o_1,o_2,o_3\}$, where $o_1 = \langle \lnot x \land \lnot y, y \rangle$, $o_2 = \langle \lnot x \land \lnot y, x \rangle$ and $o_3 = \langle x \land \lnot y, y \rangle$.
    The induced transition system of $\task$ is depicted in Figure \ref{fig:osp-example}.

    Symbolic forward search for osp \autocite{speck-katz-aaai2021} starts with a BDD representing the set of states $open = \{\init\}= \lnot x \land \lnot y$ consisting of a single state, namely the initial state.
    \Cref{fig:add} visualizes the ADD and \Cref{fig:bdd_0,fig:bdd_2,fig:bdd_3} visualize the BDDs representing the utility function $\utility$ with a maximum utility of $3$.
    In the first step, we expand $\{\init\}$, which also forms the set of best states seen so far with utility $\utility(\init) = 0$.
    The expansion leads to two new states $open = \{s_1,s_2\} = (\lnot x \land y) \lor (x \land \lnot y)$, both of which can be achieved with cost $g = 1$.
    Since $s_2$ has utility $\utility(s_2) = 2$, the set of best states changes to $\{ s_2 \}$.
    Finally, the plan $\langle o_2 \rangle$ leading to one of the best states (here $s_2$) is reconstructed as the cost limit of $\limit = 1$ is exceeded.
    Recall that all these computations are performed with decision diagrams.
\end{example}

Finally, \textcite{speck-katz-aaai2021} argue that it is not straightforward to efficiently apply regression to osp, which would enable symbolic backward search and, in particular, symbolic bidirectional search.
Unlike classical planning, where the starting point of symbolic backward search is obvious, namely the goal states represented as a compact goal formula, the starting point of symbolic backward search for osp is not obvious.
The reason is that the objective is to find a final state with maximum utility, and it is not clear how this can be represented in backward search.
It might be possible to partition all goal states by their utility values and represent them as separate BDDs.
However, since the goal formula in osp tasks is often unspecified or even empty, this may require a regression of the entire
state space with multiple backward searches.
All in all, it is not clear how to efficiently apply symbolic regression to osp, which is an open research question.

\begin{table}[t]
    \begin{center}
        \resizebox{1\textwidth}{!}{
            \begin{tabular}{lrrrrrrrr}
    \toprule
    Algorithm                    & \multicolumn{2}{c}{\textbf{BDD Search}}   & \multicolumn{1}{c}{\textbf{\astar{uADD}}} & \multicolumn{3}{c}{\astar{mc}}                     & \multicolumn{2}{c}{BnB}                                                                                                                                                                                                              \\
    \cmidrule(lr){2-3} \cmidrule(lr){4-4} \cmidrule(lr){5-7} \cmidrule(lr){8-9}
    Benchmark (\# Tasks)     & \multicolumn{1}{c}{uBDD} & \multicolumn{1}{c}{uADD}                               & \multicolumn{1}{c}{{\heu{blind}{}}} & \multicolumn{1}{c}{\heu{blind}{}} & \multicolumn{1}{c}{\heu{max}{b}} & \multicolumn{1}{c}{\heu{m\&s}{b}} & \multicolumn{1}{c}{\heu{blind}{}} & \multicolumn{1}{c}{\heu{lmcut}{mc}} 
    \\
    \midrule
    ~~25\% Bound (1667) & 1271                                     & \textbf{1274}                                                          & 1165                                                & 1197                                   & 1190                                            & 1074                                             & 1183                                   & 1151                                             \\
    ~~50\% Bound (1667) & 990                                      & \textbf{993}                                                           & 860                                                 & 901                                    & 902                                             & 828                                              & 893                                    & 867                                              \\
    ~~75\% Bound (1667) & \textbf{866}                             & 862                                                                    & 718                                                 & 758                                    & 738                                             & 734                                              & 735                                    & 702                                              \\
    100\% Bound (1667)  & \textbf{802}                             & 793                                                                    & 629                                                 & 668                                    & 655                                             & 676                                              & 643                                    & 618                                              \\
    \midrule
    \textbf{Sum (6668)}      & \textbf{3929}                            & 3922                                                                   & 3372                                                & 3524                                   & 3485                                            & 3312                                             & 3454                                   & 3338                                             \\
    \bottomrule
\end{tabular}

        }
    \end{center}
    \caption[Coverage for oversubscription planning.]{Coverage (number of optimally solved instances) for explicit and symbolic search algorithms on oversubscription planning domains \autocite{speck-katz-aaai2021}.}
    \label{tab:coverage_oversubscription}
\end{table}

\paragraph{Empirical Evaluation.}
\Cref{tab:coverage_oversubscription} shows the coverage on four different osp benchmark sets \autocite{speck-katz-aaai2021}.
Each benchmark set consists of 57 domains and originates from the optimal track of the IPC 1998-2014, where the goal facts are no longer hard goal facts, but have assigned utilities \autocite{katz-et-al-zenodo2019,katz-et-al-icaps2019}.
The sets differ in the cost bounds, which for each planning instance are set at 25\%, 50\%, 75\%, or 100\% of the cost of the optimal or best known solution.
The two symbolic search approaches (BDD search) perform a forward search using BDDs, but differ in how the utility of a set of states is evaluated, either with multiple BDDs (uBDD) or one ADD (uADD).
As a result, symbolic search with a decomposed utility function in multiple BDDs (BDD search + uBDD) works best overall.
The main advantage of representing utility values with BDDs compared to ADDs is that the underlying decision diagram library \name{cudd} \autocite{somenzi-cudd2015} uses techniques such as complement edges to store BDDs more compactly \autocite{brace-et-al-acm1990}.

\begin{figure}
    \begin{center}
        \includegraphics[width=1\textwidth]{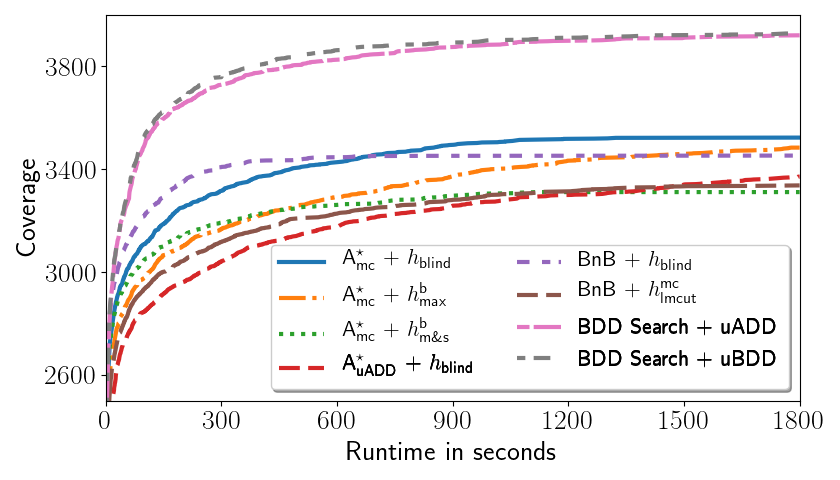}
    \end{center}
    \caption[Coverage over time for oversubscription planning.]{Coverage (number of optimally solved instances) over time for explicit and symbolic search algorithms on all oversubscription planning domains (25\%, 50\%, 75\%, or 100\% Bound) \autocite{speck-katz-aaai2021}.}
    \label{fig:coverage_time_oversubscription}
\end{figure}

Overall, the described symbolic search approach \autocite{speck-katz-aaai2021} performs better than explicit search approaches.
Comparing symbolic search with its explicit counterpart, $\astar{uADD}$ search with the blind heuristic \heu{blind}{}, we see a difference in performance due to the concise representation of state sets.
Considering other publicly available planning algorithms for osp, we find that branch-and-bound search BnB \autocite{katz-et-al-icaps2019} and $\astar{mc}$ search with multiple cost functions (mc) \autocite{katz-keyder-icaps2019wshsdip} perform worse than symbolic search in terms of overall coverage.
The performance difference between symbolic search and explicit heuristic search increases as the bound increases and thus the plan length increases.
The natural explanation for this is that heuristics do not pay off, since in osp it seems rather difficult to define a heuristic that is both informative and fast to compute. 
This observation is underlined by \Cref{fig:coverage_time_oversubscription}, which depicts the number of instances solved over time, where symbolic search dominates all other approaches after a short time.
Finally, as usual, there are domains where explicit heuristic search performs better than symbolic search.


\chapter{Top-k Planning}\label{ch:topk}
\chapterquote{If plan A fails, remember there are 25 more letters.}{Chris Guillebeau}

\renewcommand{\kiviatTopk}{2}
\begin{figure}[t]
    \begin{center}
        \begin{tikzpicture}
    \tkzKiviatDiagram[
    radial style/.style ={-{Latex[length=3mm, width=2mm]}},
    scale=0.85,
    label space=1.5,
    radial = 1,
    gap = 1.5,
    step = 1,
    lattice = 2]{
    \hspace{2.5cm}\mbox{\parbox{4cm}{{\begin{center}\textbf{Goal}\\\textbf{Description} \end{center}}}},
    {\textbf{Number of Plans}},
    \hspace{-1.5cm}\mbox{\parbox{3cm}{{\begin{center}\textbf{State}\\\textbf{Description}\end{center}}}},
    \textbf{Cost Function}}

    \ifbool{kiviatRed}{
        \tkzKiviatLine[very thick,color=red!50,
            fill=red!30,
            opacity=.35,
            mark=ball,
            mark size=2.5pt,
            ball color=red](\kiviatGoal{},\kiviatTopk{},\kiviatPredicates{},\kiviatCost{})
    }

    \ifbool{kiviatBlue}{
        \tkzKiviatLine[very thick,color=blue!50,
            fill=blue!30,
            opacity=.35,
            mark=ball,
            mark size=2.5pt,
            ball color=blue](1,1,1,1)
    }

    \ifbool{kiviatFull}{
    \tkzKiviatLine[very thick,color=darkgreen!50,
        fill=green!30,
        opacity=.35,
        mark=ball,
        mark size=2.5pt,
        ball color=darkgreen](1,2,2,2)
    \tkzKiviatLine[very thick,color=blue!50,
        fill=blue!30,
        opacity=.35,
        mark=ball,
        mark size=2.5pt,
        ball color=blue](1,1,1,1)
    \tkzKiviatLine[very thick,color=red!50,
        mark=ball,
        mark size=2.5pt,
        ball color=red](\kiviatGoal{},\kiviatTopk{},\kiviatPredicates{},\kiviatCost{})
    \LegendBox[shift={(-0cm,-2.25cm)}]{current bounding box.south west}%
    {
    red!100/{Forward symbolic search (contribution of this thesis)},
    asparagus!100/{Bidirectional symbolic search (contribution of this thesis)},
    blue!100/{Bidirectional symbolic search (previous state of the art)}}
    }
    {}

    \draw[] node[align=center,fill=none] at (1.5,-0.5) {\small Hard};
    \draw[] node[align=center,fill=none] at (3.25,-0.75) {\small Oversub-\\ \small scription};

    \draw[] node[align=center,fill=none] at (0.7,1.5) {\small Top-$1$};
    \draw[] node[align=center,fill=none] at (0.7,3.0) {\small Top-$k$};

    \draw[] node[align=center,fill=none] at (-1.25,0.75) {\small State\\ \small Variables};
    \draw[] node[align=center,fill=none] at (-3.25,0.75) {\small $+$ \small Derived\\ \small Variables};

    \draw[] node[align=center,fill=none] at (-1,-1.5) {\small Constant};
    \draw[] node[align=center,fill=none] at (-1.6,-3.0) {\small State-Dependent\phantom{0}};
\end{tikzpicture}
    \end{center}
    \caption[Overview of extension for classical planning (top-$k$).]{Overview of extensions for classical planning, where the red color denotes the planning formalism supported by the proposed symbolic search approach.}\label{fig:top_k:kiviat}
\end{figure}
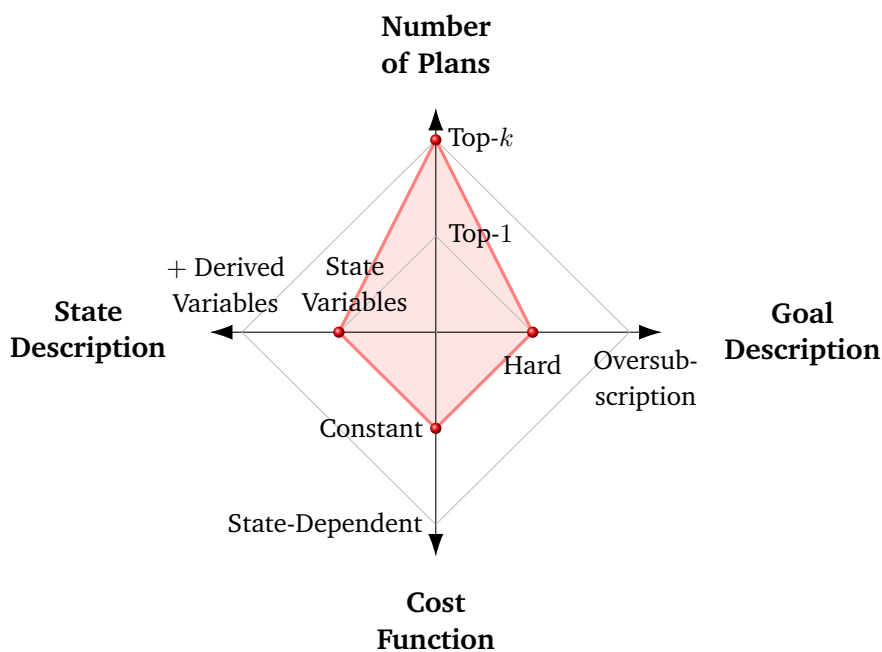
\renewcommand{\kiviatTopk}{1}

\section*{Core Publication of this Chapter}
\renewcommand{\citebf}[1]{\textbf{#1}}
\begin{itemize}
    \item \fullcite{speck-et-al-aaai2020}
\end{itemize}
\renewcommand{\citebf}[1]{#1}

In planning, it is common to assume that the model is fully specified and that the objective is to find a single plan.
While in some cases a single plan may be sufficient, in practice it is often better to have many good alternative plans \autocite{nguyen-et-al-aij2012}.
The possibility of obtaining multiple high-quality plans can be of great importance in various fields and applications, such as task and motion planning \autocite{lavalle-2006,ren-et-al-arxiv2021,lozano-et-al-iros2014}, diverse planning \autocite{katz-et-al-aaai2020}, scenario planning \autocite{sohrabi-et-al-aaai2018}, goal recognition \autocite{sohrabi-et-al-ijcai2016} or planning with user preferences \autocite{nguyen-et-al-aij2012,seimetz-et-al-ijcai2021}.

The problem of determining $k$ shortest paths for a given graph is a well-studied topic, going back to 1957 \autocite{bock-et-al-1957}.
However, the problem of determining $k$ cheapest plans, known as top-$k$ planning, has only been studied more recently \autocite{riabov-et-al-icaps2014wsspark}.
In addition to generating a set of possible solutions for all the above mentioned applications, the systematic enumeration of plans in an anytime behavior enables the realization of a \say{generate-and-test framework} by searching for plans that must satisfy certain complex requirements.
Thus, top-$k$ planning can also be used to search in a simplified version of a given problem, e.g., in an overapproximated, abstracted, or decomposed model representation \autocite{hoeller-icaps2021}.

In this chapter, we define and motivate top-$k$ planning (\Cref{fig:top_k:kiviat}) before presenting the work of \textcite{speck-et-al-aaai2020} on this topic.
On the theoretical side, we discuss the question of whether top-$k$ planning is computationally more difficult than ordinary planning.
On the practical side, we describe a sound and complete symbolic search approach to top-$k$ planning.
The empirical evaluation shows that symbolic search performs better than other state-of-the-art search methods for both a small and a large number of desired plans $k$.

\section{Formalism}

The objective of top-$k$ planning is to determine a set of $k$ different plans with lowest cost for a given classical planning task (\Cref{def:planning-task}).
We formally define top-$k$ planning as follows \autocite{riabov-et-al-icaps2014wsspark,katz-et-al-icaps2018,speck-et-al-aaai2020}.

\begin{definition}[Top-k Planning]\label{def:top-k-planning}
    Given a planning task $\task$ and a natural number $k \in \mathbb{N} \cup \{ \infty \}$, top-$k$ planning is the problem of determining a set of plans $P \subseteq P_\task$ such that:
    \begin{enumerate}
        \item there exists no plan $\plan' \in P_\task$ with $\plan' \not\in P$ that is cheaper than some plan $\plan \in P$, and
        \item $|P| = k$ if $|P_\task| \geq k$, and $|P| = |P_\task{}|$, otherwise.
    \end{enumerate}
\end{definition}

We emphasize that top-$k$ planning is a generalization of classical planning, where $k=1$ corresponds to cost-optimal classical planning.

There can be infinitely many plans, since according to \Cref{def:top-k-planning} plans with cycles, i.e., plans which visit the same state multiple times, are allowed.
Moreover, there can be infinitely many plans with optimal costs if cycles with $0$ costs exist.
In general, top-$k$ planners not only provide a way to generate the $k$ best plans, but also to enumerate all possible plans if no $0$-actions or, more precisely, no $0$-cost cycles exist.
Depending on the application, plans with cycles can be important to provide all possibilities, e.g., when user preferences are not fully known or the model specification is not complete.
In addition, the sequential generation of all plans allows the realization of complete search strategies that search in a simplified version of a problem at hand \autocite{hoeller-icaps2021}.
The symbolic approach presented in this chapter can be modified to generate only simple plans, i.e., plans without cycles \autocite{vontschammer-et-al-icaps2022}.
However, we will focus on \say{ordinary} top-$k$ planning, which includes all possible plans.
\Cref{ex:top_k} illustrates the idea of top-$k$ planning when user preferences are not fully known and the model is incomplete.

\begin{example}\label{ex:top_k}
    Consider \Cref{ex:rover}, in which we identified an optimal plan where the rover first navigates to (7,1) before the drone flies and takes an image at (6,1) and then at (10,1).
    Finally, the rover equipped with the drone navigates to (0,5).
    A closer look reveals that the order in which the drone takes the images at the two desired locations is irrelevant to the overall cost of the plan.
    With top-$k$ planning we can determine both plans, one where the drone flies first to (6,1) and then to (10,1) and vice versa.
    Based on these alternatives, it can be decided at which of the two locations the images should be taken first, depending on what the user prefers.
    In addition, there are various plans with different navigation routes of the rover, which allow to choose a different path during the execution of the plan without replanning and thus to react to unforeseen events that were not considered in the model.
\end{example}

\section{Complexity and Compilability}
Clearly, top-$k$ planning is as hard as classical planning, since it is a straightforward generalization.
However, the question arises to what extend it is harder to determine multiple plans instead of a single one.
Unlike the previously discussed extension to classical planning, top-$k$ planning is not a direct model extension.
The concept of compilability cannot (readily) be applied here, since the output changes from a single plan to a set of plans, which raises a conceptually different decision problem.
Thus, \textcite{speck-et-al-aaai2020} examined the computational complexity of top-$k$ planning and defined the \emph{bounded top-$k$ plans existence} problem (\Cref{def:bounded-top-k-plan-existence}) to answer the question of whether top-$k$ planning is as difficult as classical planning.

\begin{definition}[Bounded Top-k Plans Existence]\label{def:bounded-top-k-plan-existence}
    \emph{Bounded top-$k$ plans existence} is the decision problem: Given a planning task $\Pi$ and two natural numbers $\ell$ and $k$, are there at least $k$ different plans of length at most $\ell$?
\end{definition}

First of all, \textcite{speck-et-al-aaai2020} showed that the bounded top-$k$ plans existence problem in general is \PSPACE-complete (\Cref{thm:bounded-top-k-plan-existence}), which is surprising because the ordinary bounded plan existence problem, i.e., answering the question whether only one such plan exists, is also \PSPACE-complete (\Cref{thm:bounded-plan-existence}).

\begin{theorem}[\cite{speck-et-al-aaai2020}]\label{thm:bounded-top-k-plan-existence}
    Bounded top-$k$ plans existence is \PSPACE-complete. \qed
\end{theorem}

Considering polynomially bounded plans, it turns out that the bounded \mbox{top-$k$} plan existence decision problem is \PP-hard \autocite{gill-siam-1977,speck-et-al-aaai2020}, while the ordinary bounded plan existence problem is \NP-complete \autocite{bylander-aij1994}.

\begin{theorem}[\cite{speck-et-al-aaai2020}]\label{thm:bounded-top-k-plan-existence-short}
    Bounded top-$k$ plans existence is \PP-hard if the plan length $\ell$ is polynomially bounded by the task size, i.e., $\ell \leq p(\size{\task})$ for some polynomial $p$. \qed
\end{theorem}

\Cref{thm:bounded-top-k-plan-existence-short} indicates that top-$k$ planning is much harder in practice than ordinary classical planning.
The assumption that both problems have the same complexity when the plan length is polynomially bounded would imply a collapse of the polynomial hierarchy at $\mathsf{P}^{\text{\NP}}$ \autocite{toda-siam1991}, which is considered very unlikely.

\section{Symbolic Search}
\textcite{speck-et-al-aaai2020} propose a symbolic search approach to top-$k$ planning that is sound and complete.
Recall that symbolic search for classical planning is divided into two phases: 1) a reachability phase, in which states with increasing costs are generated until a goal state is found, and 2) a plan reconstruction phase, in which the search is regressed to reconstruct a corresponding goal path and plan.
Since symbolic search expands entire sets of states at once, it is not unusual to find multiple goal states at once.
If multiple goal states are found, then of course multiple plans are found as well.
But even if only one goal state is found, it is possible that multiple plans are found that lead to that particular goal state.
Usually, symbolic search reconstructs and reports only one such plan, ignoring all others.
Based on this observation, \textcite{speck-et-al-aaai2020} show that three modifications to ordinary symbolic search lead to a sound and complete symbolic search algorithm for top-$k$ planning.
First, they propose that after a goal state is expanded, all plans leading to that goal state are reconstructed.
Second, during the search, states are not closed, i.e., all newly generated states with their corresponding reachability costs are insert in the open list without filtering them by already expanded states, otherwise suboptimal plans may be lost.
Third and finally, the termination criterion is adjusted so that the algorithm terminates if either $k$ plans are found or the open list contains only states that have already been expanded at least once and are not part of a goal path induced by a plan that has already been found.

\Cref{ex:gripper} illustrates the functioning of this modified symbolic search for top-$k$ planning.
For simplicity, this example describes symbolic forward search for a task with unit cost.

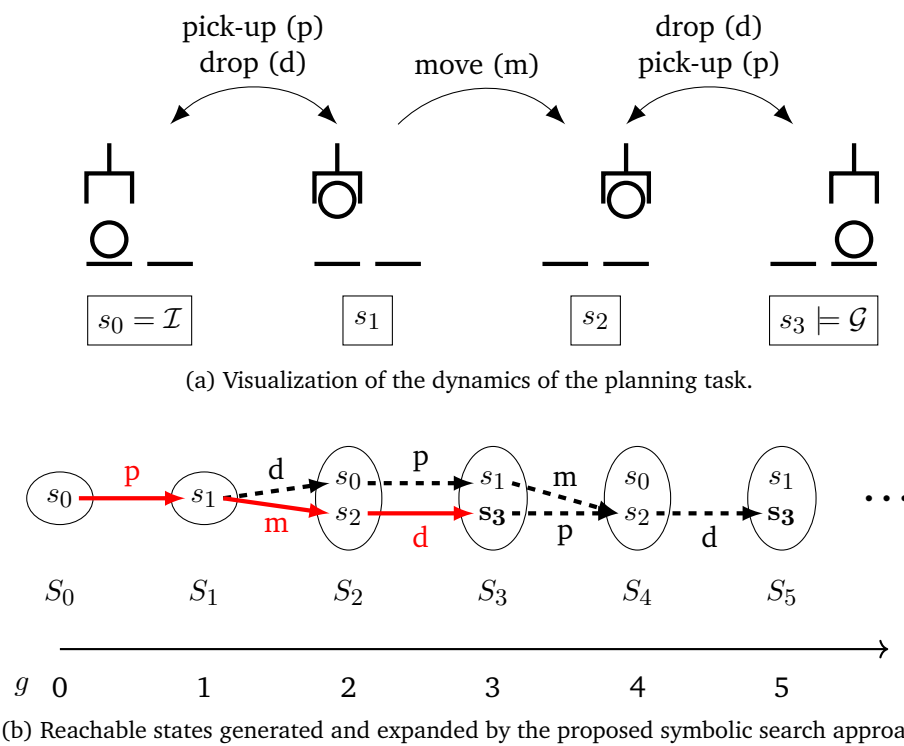
\begin{figure}
    \centering
    \subfloat[Visualization of the dynamics of the planning task.\label{fig:gripper_example_a}]{
        \centering
        \makebox[1\textwidth][c]{
            \newcommand{\gripperFloor}[3]
{
    \draw [ultra thick] (#1,#2) -- (#1 + 0.6, #2);
    \draw [ultra thick] (#1 + 0.8, #2) -- (#1 + 1.4,#2);
    \node [draw,minimum size=0.65cm] at (#1 + 0.7, #2 - 0.75) {$s_#3$};

}

\newcommand{\gripperHand}[2]
{
    \draw [ultra thick] (#1 + 0.3,#2) -- (#1 + 0.3, #2 + 0.4);
    \draw [ultra thick] (#1, #2 - 0.4) --(#1,#2) -- (#1 + 0.6,#2) -- (#1 + 0.6, #2 - 0.4);
}

\newcommand{\gripperBall}[2]
{
    \draw [ultra thick] (#1,#2) circle [radius=0.225];
}

\begin{tikzpicture}[]
    \gripperFloor{0}{0}{0 = \init}
    \gripperBall{0.3}{0.325}
    \gripperHand{0.0}{1.2}

    \gripperFloor{3}{0}{1}
    \gripperBall{3.3}{0.85}
    \gripperHand{3}{1.2}

    \gripperFloor{6}{0}{2}
    \gripperBall{7.1}{0.85}
    \gripperHand{6.8}{1.2}

    \gripperFloor{9}{0}{3 \models \goal}
    \gripperBall{10.1}{0.325}
    \gripperHand{9.8}{1.2}

    \draw[>=latex, <->] (1.1,1.85) to[ultra thick, bend left=45] node [above, align=center] {pick-up (p) \\ drop (d)} (3.3,1.85);
    \draw[>=latex, ->] (4.1,1.85) to[ultra thick, bend left=45] node [above, align=center] {move (m)} (6.3,1.85);
    \draw[>=latex, <->] (7.1,1.85) to[ultra thick, bend left=45] node [above, align=center] {drop (d) \\ pick-up (p)} (9.3,1.85);

    \node at (0,-1.1) {};

\end{tikzpicture}
        }
    }
    \\ \medskip
    \subfloat[Reachable states generated and expanded by the proposed symbolic search approach.\label{fig:gripper_example_b}]{
        \centering
        \makebox[1\textwidth][c]{
            \begin{tikzpicture}[]
    \newcommand*{\xd}{1.9}
    \node[draw,ellipse] at (0*\xd, -2) {$s_0$};
    \node[draw,ellipse] at (1*\xd, -2) {$s_1$};
    \node[draw,ellipse, align=center] at (2*\xd, -2) {$s_0$ \\ $s_2$};
    \node[draw,ellipse, align=center] at (3*\xd, -2) {$s_1$ \\ $\mathbf{s_3}$};
    \node[draw,ellipse, align=center] at (4*\xd, -2) {$s_0$ \\ $s_2$};
    \node[draw,ellipse, align=center] at (5*\xd, -2) {$s_1$ \\ $\mathbf{s_3}$};
    \node[align=center] at (5.75*\xd, -2) {\Large \textbf{\dots}};

    \node[] at (0*\xd, -3.25) {$S_0$};
    \node[] at (1*\xd, -3.25) {$S_1$};
    \node[] at (2*\xd, -3.25) {$S_2$};
    \node[] at (3*\xd, -3.25) {$S_3$};
    \node[] at (4*\xd, -3.25) {$S_4$};
    \node[] at (5*\xd, -3.25) {$S_5$};

    \path [draw, ->, thick] (0.0, -4) to (-0.5+6*\xd, -4);
    \node at (-0.5, -4.5) {$g$};
    \foreach \x in {0,...,5}
    \node  at (\xd*\x,-4.5) {\x};

    \draw[>=latex, ->,ultra thick, red] (0.25+0*\xd,-2) to[ultra thick] node [above] {p} (-0.25+1*\xd,-2);
    \draw[>=latex, ->,ultra thick, dashed] (0.25+1*\xd,-2) to[ultra thick] node [above] {d} (-0.25+2*\xd,-1.8);
    \draw[>=latex, ->,ultra thick, red] (0.25+1*\xd,-2) to[ultra thick] node [below] {m} (-0.25+2*\xd,-2.2);
    \draw[>=latex, ->,ultra thick, red] (0.25+2*\xd,-2.2) to[ultra thick] node [below] {d} (-0.25+3*\xd,-2.2);
    \draw[>=latex, ->,ultra thick, dashed] (0.25+2*\xd,-1.8) to[ultra thick] node [above] {p} (-0.25+3*\xd,-1.8);
    \draw[>=latex, ->,ultra thick, dashed] (0.25+3*\xd,-1.8) to[ultra thick] node [above] {m} (-0.25+4*\xd,-2.2);
    \draw[>=latex, ->,ultra thick, dashed] (0.25+3*\xd,-2.2) to[ultra thick] node [below] {p} (-0.25+4*\xd,-2.2);
    \draw[>=latex, ->,ultra thick, dashed] (0.25+4*\xd,-2.2) to[ultra thick] node [below] {d} (-0.25+5*\xd,-2.2);
\end{tikzpicture}
        }
    }
    \caption[Visualization of symbolic search for top-$k$ planning.]{Visualization of the \name{gripper} planning task and the functioning of the proposed symbolic search approach for top-$k$ planning \autocite{speck-et-al-aaai2020}.}
    \label{fig:gripper_example}
\end{figure}

\begin{example}[\cite{speck-et-al-aaai2020}]\label{ex:gripper}
    Consider a unit cost planning task with two rooms and a robot with a gripper, as shown in \Cref{fig:gripper_example_a}.
    The robot can move from room A to B if it carries the ball, but it cannot return.
    There is a possibility to pick-up and drop the ball in each room.
    The goal is to get the ball from room A to room B.

    Assuming that the desired number of plans is $k = 3$, \Cref{fig:gripper_example_b} illustrates the functioning of the proposed symbolic search approach.
    First, all states $S_0$ reachable with cost $0$, i.e., only the initial state, are expanded resulting in the set of states $S_1 = \{s_1 \}$ reachable with cost $1$.
    The subsequent expansion of $S_1$ yields the state set $S_2 = \{ s_0, s_2 \}$, whose cost is $2$.
    Note that $s_0$ is a part of $S_2$, although it has been previously expanded and therefore would no longer be considered in ordinary symbolic search.
    The next expansion yields $S_3 = \{ s_1, s_3 \}$ with cost $3$, leading to the expansion of $S_3$, which contains the goal state $s_3$.
    Therefore, the plan reconstruction procedure is executed that yields exactly one plan $\plan_1 = \langle \textit{pick-up}, \textit{move}, \textit{drop} \rangle$ visualized in red in \Cref{fig:gripper_example_b}.
    Since a total of $3$ plans is desired, the search continues from here until the next goal states are expanded, which occurs with the expansion of $S_5$.
    Plan reconstruction yields two plans with a cost of $5$ each, namely $\plan_2 = \langle \textit{pick-up}, \textit{drop}, \textit{pick-up}, \textit{move}, \textit{drop} \rangle$ and $\plan_3 = \langle \textit{pick-up}, \textit{move}, \textit{drop}, \textit{pick-up},  \textit{drop} \rangle$.
    As the desired number of plans is achieved, the algorithm terminates with $\{\plan_1, \plan_2, \plan_3\}$.
    However, if more plans were desired, the search would continue.
\end{example}

Looking more closely at the reconstruction phase, we can see that we perform an exhaustive greedy backward search using the provided perfect heuristic of the reachability phase.
Although this exhaustive search may seem expensive, it is goal-driven due to the perfect heuristic, and each time the initial state is reached, a new plan is created.
Finally, \textcite{speck-et-al-aaai2020} also define and describe the support of general operator costs, including zero costs, and defined symbolic backward search and symbolic bidirectional search, which can significantly improve planning performance.

\begin{figure}
    \begin{center}
        \includegraphics[width=\textwidth]{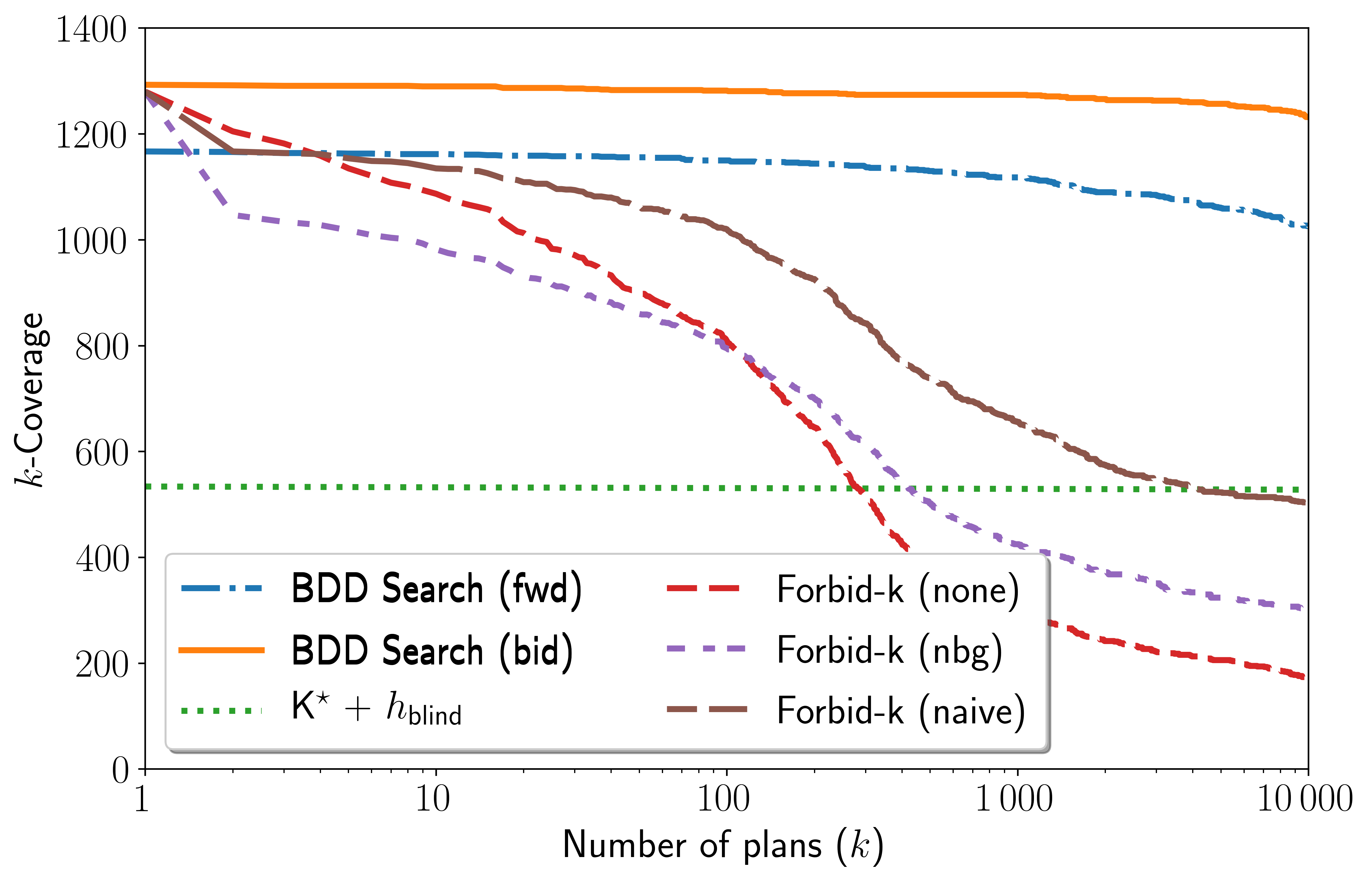}
    \end{center}
    \caption[The $k$-coverage for top-$k$ planning.]{The $k$-coverage (number of instances for which $k$ best plans were found) for explicit and symbolic search algorithms on classical planning domains \autocite{speck-et-al-aaai2020}.}
    \label{fig:coverage_time_topk}
\end{figure}

\paragraph{Empirical Evaluation.}
To compare different approaches to top-$k$ planning, we consider $k$-coverage, i.e., the number of instances for which a planner reports a set of $k$ best plans or reports only $k' < k$ plans but proves that only $k'$ plans exist (see \Cref{def:top-k-planning}).
\Cref{fig:coverage_time_topk} compares the $k$-coverage of the discussed symbolic search approach based on BDDs from \textcite{speck-et-al-aaai2020} with other state-of-the-art approaches from the literature on domains of the optimal track of the International Planning Competitions 1998-2018.
The $x$-axis indicates the number of desired plans $k$, while the $y$-axis represents the number of instances in which the corresponding $k$-coverage has been reached.
\kstar{} \autocite{aljazzar-leue-aij2011,katz-et-al-icaps2018} generates and processes parts of the implicit search tree as needed and can be enhanced by a heuristic.
However, the evaluation only includes \kstar{} with the blind heuristic because \kstar{} is very memory intensive, so the performance differences reported by \textcite{speck-et-al-aaai2020} between \kstar{} with and without sophisticated heuristics are negligible.
\forbidk{} \autocite{katz-et-al-icaps2018} is an iterative approach that searches for additional plans through a replanning loop that forbids already found plans and preserves all other plans.
The underlying search is an orbit space search with structural symmetries \autocite{alkhazraji-et-al-ipc2014,domshlak-et-al-tr2015} and the \name{LM}-cut heuristic \autocite{helmert-domshlak-icaps2009}.
The three versions of \forbidk{} considered differ in a plan reordering step, where different reordering strategies (none, naive, and neighbor) can be used to generate multiple plans from a single plan.

Overall, we can see that symbolic bidirectional search performs as well as \forbidk{} when only one plan ($k=1$) is requested.
But already for $k=2$ symbolic bidirectional search performs best and for $k=4$ symbolic forward search surpasses Forbid-k, while symbolic bidirectional search already shows a large performance advantage.
Most importantly, symbolic search scales much better to larger $k$ than all other approaches.
\kstar{} solves only six more instances for $k = 1$ than for $k = 10\,000$, due to the high memory consumption that is the bottleneck of the approach.
The replanning of \forbidk{} turns out to be too expensive when a larger number of plans is desired, resulting in a large performance drop.
Finally, \textcite{speck-et-al-aaai2020} expect the dominant performance of symbolic search to hold even for very large $k$ until the high number of plan reports is the limiting factor, i.e., the limiting factor is writing the plans to disk.

\chapter{Discussion}\label{ch:discussion}
\chapterquote{Be happy, but never satisfied.}{Bruce Lee}

\setboolean{kiviatFull}{true}
\renewcommand{\kiviatTopk}{2}
\renewcommand{\kiviatGoal}{2}
\renewcommand{\kiviatCost}{2}
\renewcommand{\kiviatPredicates}{2}
\begin{figure}[t]
    \begin{center}
        \begin{tikzpicture}
    \tkzKiviatDiagram[
    radial style/.style ={-{Latex[length=3mm, width=2mm]}},
    scale=0.85,
    label space=1.5,
    radial = 1,
    gap = 1.5,
    step = 1,
    lattice = 2]{
    \hspace{2.5cm}\mbox{\parbox{4cm}{{\begin{center}\textbf{Goal}\\\textbf{Description} \end{center}}}},
    {\textbf{Number of Plans}},
    \hspace{-1.5cm}\mbox{\parbox{3cm}{{\begin{center}\textbf{State}\\\textbf{Description}\end{center}}}},
    \textbf{Cost Function}}

    \ifbool{kiviatRed}{
        \tkzKiviatLine[very thick,color=red!50,
            fill=red!30,
            opacity=.35,
            mark=ball,
            mark size=2.5pt,
            ball color=red](\kiviatGoal{},\kiviatTopk{},\kiviatPredicates{},\kiviatCost{})
    }

    \ifbool{kiviatBlue}{
        \tkzKiviatLine[very thick,color=blue!50,
            fill=blue!30,
            opacity=.35,
            mark=ball,
            mark size=2.5pt,
            ball color=blue](1,1,1,1)
    }

    \ifbool{kiviatFull}{
    \tkzKiviatLine[very thick,color=darkgreen!50,
        fill=green!30,
        opacity=.35,
        mark=ball,
        mark size=2.5pt,
        ball color=darkgreen](1,2,2,2)
    \tkzKiviatLine[very thick,color=blue!50,
        fill=blue!30,
        opacity=.35,
        mark=ball,
        mark size=2.5pt,
        ball color=blue](1,1,1,1)
    \tkzKiviatLine[very thick,color=red!50,
        mark=ball,
        mark size=2.5pt,
        ball color=red](\kiviatGoal{},\kiviatTopk{},\kiviatPredicates{},\kiviatCost{})
    \LegendBox[shift={(-0cm,-2.25cm)}]{current bounding box.south west}%
    {
    red!100/{Forward symbolic search (contribution of this thesis)},
    asparagus!100/{Bidirectional symbolic search (contribution of this thesis)},
    blue!100/{Bidirectional symbolic search (previous state of the art)}}
    }
    {}

    \draw[] node[align=center,fill=none] at (1.5,-0.5) {\small Hard};
    \draw[] node[align=center,fill=none] at (3.25,-0.75) {\small Oversub-\\ \small scription};

    \draw[] node[align=center,fill=none] at (0.7,1.5) {\small Top-$1$};
    \draw[] node[align=center,fill=none] at (0.7,3.0) {\small Top-$k$};

    \draw[] node[align=center,fill=none] at (-1.25,0.75) {\small State\\ \small Variables};
    \draw[] node[align=center,fill=none] at (-3.25,0.75) {\small $+$ \small Derived\\ \small Variables};

    \draw[] node[align=center,fill=none] at (-1,-1.5) {\small Constant};
    \draw[] node[align=center,fill=none] at (-1.6,-3.0) {\small State-Dependent\phantom{0}};
\end{tikzpicture}
    \end{center}
    \caption[Overview of the contribution of this thesis.]{
        Overview of the contribution of this thesis in terms of extending symbolic search to support different expressive extensions of classical planning.
    }\label{fig:dicussion:kiviat}
\end{figure}
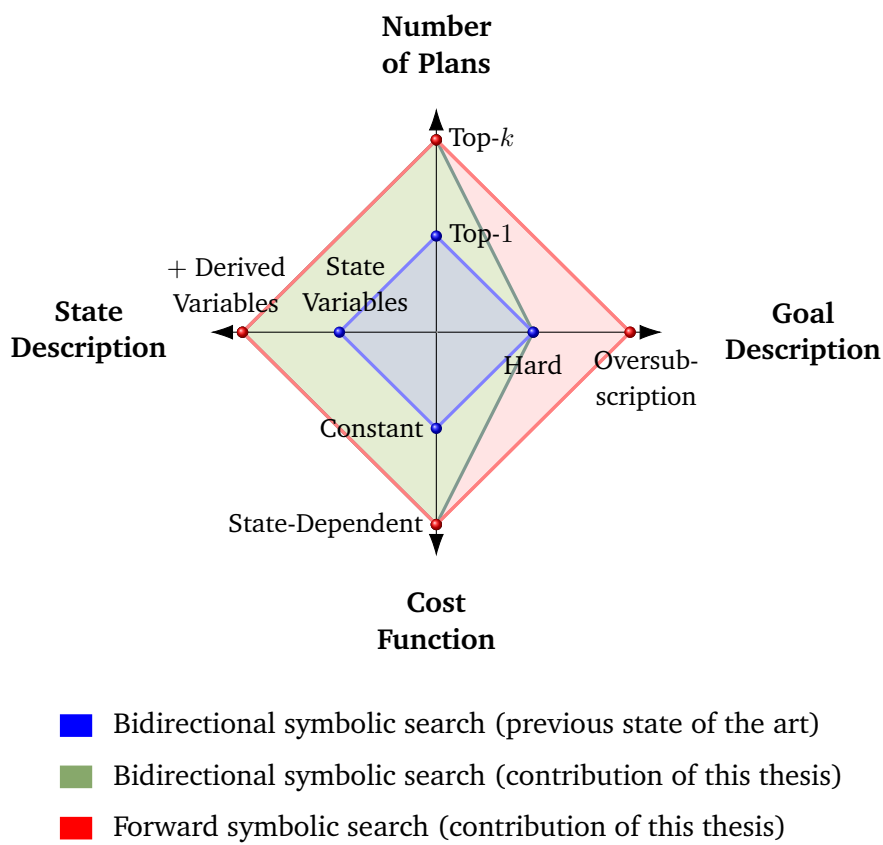

\setboolean{kiviatFull}{false}

As we discussed in \Cref{sec:symbolic_search_discussion}, symbolic blind search, i.e., without any heuristics, is the dominant search strategy for symbolic search, on par with explicit heuristic search in optimal planning.
For this reason, symbolic search provides a good basis for supporting several extensions of classical planning.
In this chapter, we discuss the combination of the previously considered extensions and how symbolic search with the introduced modifications can support this setting.
Finally, possible future work related to this thesis is discussed.

\section{Combination of Extensions}\label{sec:combination}
As we have seen in the previous chapters, it is possible to extend symbolic search to support various extensions of classical planning.
A natural question that arises is, of course, about combining these extensions, which makes intuitive sense since they all capture and extend different aspects of classical planning.
The key question is whether the different enhancements to symbolic search can be combined to support several of these planning extensions simultaneously.
The short answer to this question is: Yes, with some minor restrictions on the search directions.
This is remarkable because the support of only one of these features is quite rare, and the support of more than one feature is almost inexistent in the literature on cost-optimal planning.
The reason is the same why very few planning techniques and planners support a single feature, namely that almost all optimal planners are based on heuristic search and it is very difficult to design admissible, informative, and fast to evaluate heuristics that take into account the different extensions to classical planning.
We will now briefly discuss the combination of the discussed extensions for classical planning and how and why the proposed modifications to symbolic search can support all of these extensions simultaneously.

\paragraph{Model Extensions.}
Let us first focus on the three discussed model extensions that result in \emph{oversubscription planning with axioms and state-dependent action costs}.
This planning formalism is simply defined by combining all the features of the extension and allowing the state-dependent cost function and the utility function to take into account derived variables \derivedvars such that the cost function of an operator $o$ is given by $\costfun_o : \extendedstates \to \mathbb{N}_0$ and the utility function $\utility : \extendedstates \to \mathbb{N}_0$.
This makes it possible to concisely encode complex state-dependent action cost functions and state utility functions.
In the empirical evaluation of \Cref{ch:sdac}, we have already mentioned that derived variables can be used in a state-dependent action cost function in the Power Supply Restoration (PSR) domain \autocite{thiebaux-cordier-ecp2001,thiebaux-et-al-ijcai2013,hoffmann-et-al-jair2006}, to encode whether certain parts of the power system are energized.
A closer look, however, reveals an interesting connection between these model extensions.
While derived variables in the cost or utility function can lead to a simple and concise representation, it might be possible to simulate the evaluation of the values of the derived variables by the cost or utility function, which could make the consideration of derived variables in the cost or utility function obsolete.
Moreover, it might be possible to simplify a complex cost or utility function by introducing new derived variables and thus shifting the complexity into the computation of the derived variables.
This shows that these extensions are closely related and a more detailed analysis of the connection is an exciting question for future work (\Cref{sec:future_work}).

All proposed symbolic encodings of axioms and derived variables in the context of symbolic search (\Cref{ch:axioms}) allow derived variables to exist in the domain of operator cost functions and utility functions.
Thus, the proposed enhancements to symbolic search can be directly used for oversubscription planning with axioms and state-dependent action costs.
This is especially true for the dominant axiom encoding strategy, symbolic translation, in which each derived variable is replaced by its primary representation (\Cref{def:primary_representation}).
However, as \Cref{fig:dicussion:kiviat} illustrates, in oversubscription planning tasks, i.e., in the presence of utilities, only forward symbolic search can be performed, since it is not clear how to perform backward or bidirectional symbolic search efficiently.

\paragraph{Top-k Planning with Model Extensions.}
The search for the best $k$ plans is independent of the support of axioms and state-dependent action costs and can be easily combined.
However, if one combines oversubscription planning and top-$k$ planning, the question arises which are the best $k$ plans.
Although this question has never been answered in the literature, it is natural to rank plans first by utility and then by cost, i.e., a plan $\pi$ is better than a plan $\pi'$ iff 1) $\utility(\pi) > \utility(\pi')$ or 2) $\utility(\pi) = \utility(\pi')$ and $\cost(\pi) < \cost(\pi')$.
And indeed, this can be supported with minor modifications to the proposed symbolic search approach by maintaining not only the states with the best utility found so far, but all relevant ones and executing exhaustive plan extraction as done in symbolic search for top-$k$ planning (\Cref{fig:dicussion:kiviat}).
\medskip

Overall, the proposed extension of symbolic search yields an optimal, sound, and complete approach to \emph{top-$k$ oversubscription planning with axioms and state-dependent action costs}.
As \Cref{fig:dicussion:kiviat} illustrates, this not only enhances the applicability and state of the art of symbolic search for classical planning with expressive extensions, but also provides the first planner that can support several of these extensions at once.\footnote{Available online: \url{https://github.com/speckdavid/symk}}
Finally, the proposed planning algorithm that supports the interaction of all these features does not suffer from its generality, because if certain features are not present, the approach becomes standard symbolic search.

\section{Future Work}\label{sec:future_work}
While we have highlighted some future challenges and research in the individual chapters, here we focus on future research directions related to symbolic (heuristic) search and planning with expressive extensions in general.

We have seen in this thesis different types of decision diagrams such as EVMDDs, ADDs and BDDs used in the context of symbolic search for classical planning.
A more detailed comparison in theory and practice could provide more insight into when to use which decision diagrams.
Other types of decision diagrams such as Functional Decision Diagrams \autocite{kebschull-et-al-dac1992} or Kronecker Functional Decision Diagrams \autocite{drechsler-becker-tcad1998} could also lead to an even more concise and efficient representation.

Throughout this thesis, we have assumed a fixed static variable order.
Since the order of the variables plays an important role in the efficiency of symbolic search, as the size of the decision diagrams strongly depends on the chosen order of the variables, it is important to find good orders.
While it is known that the computation of an optimal order for decision diagrams is \coNP{}-complete \autocite{bryant-ieeecomp1986}, it is still a challenging question whether and how to find good orders in practice \autocite{kissmann-hoffmann-icaps2013,kissmann-hoffmann-jair2014}.
Moreover, dynamic reordering techniques \autocite{rudell-iccad1993} are used in other fields, such as model checking \autocite{yang-et-al-fmcad1998} or logic synthesis \autocite{scholl-et-al-tcad1999}, which have been explored up to now only sparsely in planning \autocite{kissmann-hoffmann-jair2014,kissmann-et-al-ipc2014}.

The efficient integration of heuristics into symbolic search is still an open research task.
We believe that the main goal should be to keep the decision diagrams involved as small as possible.
\textcite{fiser-et-al-icaps2021wshsdip} introduced operator potentials based on potential heuristics \autocite{pommerening-et-al-aaai2015}, which seem to provide concise representation in symbolic forward search, while efficiently pruning states based on their heuristic values.
More work certainly needs to be done here to better understand the search behavior of symbolic heuristic search.
In particular, it is interesting to see if it is possible to find heuristics that give any size guarantees for the decision diagrams involved, and if and how the results can be generalized to symbolic bidirectional heuristic search.

Considering the discussed extensions of classical planning, an in-depth theoretical and empirical evaluation of the combination of several such extensions and the symbolic search approach presented for this purpose is certainly a promising research direction.
In particular, it is interesting to examine the relationship between different extensions.
It might be possible that one extension can be compiled away (with small overhead) in the presence of another extension.
Moreover, several of these extensions are not yet covered by more expressive planning formalisms, such as
fully observable nondeterministic planning \autocite{cimatti-et-al-aij2003}, partially observable nondeterministic planning \autocite{bertoli-et-al-aij2006,speck-et-al-ki2015}, or hierarchical task network planning \autocite{erol-et-al-amai1996,geier-bercher-ijcai2011}.
An analysis of how and to what extent all these features can enhance these planning formalisms is an important future line of research, and whether the proposed ideas can also be used for symbolic search approaches for these settings \autocite{kissmann-edelkamp-ki2009,behnke-speck-aaai2021,bertoli-et-al-aij2006}.




\chapter{Conclusion}\label{ch:conclusion}
\chapterquote{Perhaps the best test of [an agent's] intelligence is [its] capacity for making a summary.}{Lytton Strachey}

We theoretically and empirically evaluated symbolic search for cost-optimal planning with expressive extensions. 
First, we analyzed the search behavior of symbolic heuristic search in form of \bddastar{} and revealed an unknown or at least often overlooked fundamental problem: We showed that using a heuristic does not always improve the search performance of \bddastar{}.
In general, even the perfect heuristic can deteriorate search performance exponentially.
Second, we analyzed classical planning with different expressive extensions, such as planning with axioms and derived variables, planning with state-dependent action costs, oversubscription planning, and top-$k$ planning.
The discussed and presented results on the computational complexity and compilability of these planning formalisms show that native support for the relevant planning extensions is needed to solve many real-world problems.
Based on this observation, we proposed symbolic blind search to solve such planning formalisms with expressive extensions. 
More specifically, we have proposed symbolic search approaches that provide optimal, sound, and complete algorithms for these planning formalisms. 
Our empirical evaluations show that the presented symbolic search approaches perform favorably in all these planning settings compared with other state-of-the-art approaches. 
Finally, we analyzed the combination of all these classical planning extensions, i.e., top-$k$ oversubscription planning with axioms and state-dependent action costs, and show how symbolic search with the introduced modifications can support this planning formalism.
\backmatter

\cleardoublepage
\listoffigures
\cleardoublepage
\listoftables

\setlength\bibitemsep{1.5\itemsep} 
\printbibliography[heading=bibintoc,notcategory=ignore]

\ifbool{attachPDFs}{
    \cleardoublepage
\renewcommand{\citebf}[1]{\textbf{#1}}

\chapter{Core Publications}\label{chapter:core_publications}
The following publications contain the core results of this thesis and have been published at leading conferences on AI and automated planning. 
The publications are arranged according to their appearance in the thesis and are all provided in their publication form.

\cleartoleftpage
\vspace*{\fill}
\begin{itemize}
  \item \fullcite{speck-et-al-icaps2020}
\end{itemize}
\vspace*{\fill}
\clearpage
\includepdf[pages=-]{papers/speck-etal-icaps2020.pdf}

\cleartoleftpage
\vspace*{\fill}
\begin{itemize}
  \item \fullcite{speck-et-al-icaps2019}
\end{itemize}
\vspace*{\fill}
\clearpage
\includepdf[pages=-]{papers/speck-etal-icaps2019.pdf}

\cleartoleftpage
\vspace*{\fill}
\renewcommand{\citeaddendum}{\\\textbf{(Best Student Paper Runner-Up Award)}}
\begin{itemize}
  \item \fullcite{speck-et-al-icaps2021}
\end{itemize}
\renewcommand{\citeaddendum}{}
\vspace*{\fill}
\clearpage
\includepdf[pages=-]{papers/speck-etal-icaps2021a.pdf}

\cleartoleftpage
\vspace*{\fill}
\renewcommand{\citeaddendum}{\\\textbf{(Partly based on ideas from my master thesis)}}
\begin{itemize}
  \item \fullcite{speck-et-al-icaps2018}
\end{itemize}
\renewcommand{\citeaddendum}{}
\vspace*{\fill}
\clearpage
\includepdf[pages=-]{papers/speck-etal-icaps2018.pdf}

\cleartoleftpage
\vspace*{\fill}
\begin{itemize}
  \item \fullcite{speck-katz-aaai2021}
\end{itemize}
\vspace*{\fill}
\clearpage
\includepdf[pages=-]{papers/speck-katz-aaai2021.pdf}

\cleartoleftpage
\vspace*{\fill}
\begin{itemize}
  \item \fullcite{speck-et-al-aaai2020}
\end{itemize}
\vspace*{\fill}
\clearpage
\includepdf[pages=-]{papers/speck-etal-aaai2020.pdf}

\renewcommand{\citebf}[1]{#1}
}

\end{document}